%% file: arxiv.tex
\documentclass[11pt]{article}

\usepackage[margin=1.5in]{geometry}

\input{headers}

\usepackage{amsmath}
\usepackage{graphicx,psfrag,epsf}
\usepackage{enumerate}
\usepackage[numbers]{natbib}
\usepackage{url} 

\begin{document}
\title{Estimating Mixed Memberships with Sharp Eigenvector Deviations
}
\author{Xueyu Mao\thanks{Department of Computer Science. Email: \href{mailto:xmao@cs.utexas.edu}{xmao@cs.utexas.edu}}, Purnamrita Sarkar\thanks{Department of Statistics and Data Sciences. Email: \href{mailto:purna.sarkar@austin.utexas.edu}{purna.sarkar@austin.utexas.edu}}, Deepayan Chakrabarti\thanks{Department of Information, Risk, and Operations Management. Email: \href{mailto:deepay@utexas.edu}{deepay@utexas.edu}}\\ The University of Texas at Austin}
\date{}

\maketitle
\input{abstract}

\input{intro}
\input{proposed}

\input{analysis}

\input{exp}
\input{expreal}

\input{analysis-detailed}
\input{entry_eigen_1}

\input{conc}

\small{
\bibliographystyle{plainnat}
\bibliography{references}
}
\newpage
\section*{Appendix}
\appendix

\section{Identifiability}\label{sec:identifiability}
Our proof links the MMSB parameters $\bTheta$ and $\bB$ to the eigen-decomposition of the probability matrix $\bP$, and then exploits its geometric structure. Specifically, we show that the eigenvector row corresponding to any node lies inside a polytope whose vertices correspond to pure nodes. When $\bB$ is full rank, the polytope has $K$ linearly independent vertices, and the community memberships $\btheta_i$ of each node $i$ are fixed by the position of its eigenvector row with respect to these vertices. This proves part~(a) of Theorem~\ref{mmsb_iden_nece_not}. When $\bB$ is rank-deficient, the points corresponding to the pure nodes are linearly dependent. However, under the conditions of part~(b), the constraints on {$\bTheta$ and $\bB$} are shown to make the model identifiable. In other cases, we construct a new $\bTheta'$ that still yields the same probability matrix $\bP$. This proves part~(c).

\input{identifiable_1}

\section{Some Auxiliary Results, Proof of Lemmas~\ref{lem:v-row-norm}}\label{sec:auxiliary}
\input{concentration2}

\input{supp-mmsb}
\section{Proofs for Section~\ref{sec:analysis}}\label{sec:proof_sec_5}
\subsection{Proofs of Lemma~\ref{lem:gammaP_bound}} \label{sec:proof_gammaP}
\begin{lem}
	\label{lem:e-discr}
	Consider the intervals defined in Definition~\ref{def:int}. We have for positive eigenvalues: $\lambda_{s_k}\leq \sum_{i=1}^k n_i g_k$.
	\begin{proof}
		We prove this by induction. First, note that the smallest positive eigenvalue is larger than $\lambda^*(\bP)$ by definition. 
		For $k=1$, $\lambda_{s_1}-\lambda_{e_1}\leq (n_1-1)\lambda_{e_1}$, and hence $\lambda_{s_1}\leq n_1\lambda_{e_1}=n_1 g_1$. Now assume that $\lambda_{s_k}\leq \sum_{i=1}^k n_i g_k$. Hence, 
		\bas{
			\lambda_{s_{k+1}}\leq (n_{k+1}-1)g_{k+1}+(\lambda_{e_{k+1}}-\lambda_{s_k})+\lambda_{s_k}=n_{k+1}g_{k+1}+\sum_{i=1}^kn_{i}g_{k}\leq \sum_{i=1}^{k+1}n_ig_{k+1}.
		}
		The last step holds since $g_k<g_{k+1}$.
	\end{proof}
\end{lem}

\begin{proof}[Proofs of Lemma~\ref{lem:gammaP_bound}]
	First consider positive eigenvalues. 
	By Lemma~\ref{lem:e-discr},	$\lambda_{s_k}\leq g_k\sum_{i=1}^kn_i\leq g_k\sum_{i=1}^{I^{+}}n_i\leq g_k K$ and by definition $\lambda_{s_k}/g_k\leq\kappa(\bP)$, so $\lambda_{s_k}/g_k\leq \min\{K,\kappa(\bP)\}$. 
	From construction of the eigenvalue intervals, we have:
	$\lambda_{s_k}-\lambda_{s_{k-1}}\leq n_kg_k$. 
	Also note that $\sum_{k=1}^{I^{+}}{\bbb{\lambda_{s_k}-\lambda_{s_{k-1}}}/{g_k}}={\sum_{k=1}^{I^{+}}\bbb{\lambda_{s_k}-\lambda_{s_{k-1}}}/{\lambda^*(\bP)}}={{\lambda_{s_{I^{+}}}}/{\lambda^*(\bP)}}\leq \sigma_1(\bP)/\lambda^*(\bP)=\kappa(\bP)$, where $\sigma_1(\bP)$ is the largest singular value of $\bP$, we have $\sum_{k=1}^{I^{+}}\bbb{\bbb{\lambda_{s_k}-\lambda_{s_{k-1}}}/{g_k}}\leq \min\{K,\kappa(\bP)\}$. A similar argument can be made for negative eigenvalues. So $\gammaP\leq 2\min\{K,\kappa(\bP)\}^2$.
	
	If eigenvalues of $\bP$ can be divided by a constant number of bins where in each bin the eigenvalues are of the same order, for each bin, there will be at most a constant of intervals defined in Definition~\ref{def:int}, or the eigenvalues can not be of the same order. In that case, $\lambda_{s_k}$, $\lambda_{e_k}$ and $g_k$ are of the same order and $I^{+}+I^{-}$ is a constant, so $\gammaP=O(1)$.
\end{proof}
\subsection{Proof of Lemma~\ref{lem:buildup}}\label{sec:pf_buildup}
\begin{proof}
	Since $\res_\bA(z)-\res_\bP(z)=\res_\bP(z)(\bP-\bA)\res_\bA(z)$,
	\bas{
		\res_\bA(z)-\res_\bP(z)
		&=\left(\bM_z-\frac{\bI}{z}\right)(\bP-\bA)(\res_\bA(z)-\res_\bP(z))+\res_\bP(z)(\bP-\bA)\res_\bP(z)\nonumber
	}
	Bringing $\frac{\bA-\bP}{z}(\res_{\bA}(z)-\res_{\bP}(z))$ to the LHS, and using the definition of the resolvent of $\bA-\bP$ we get: 
	\ba{\label{eq:grecursive}
		\res_\bA(z)-\res_\bP(z)&=z\res_{\bA-\bP}(z)\left(\underbrace{\bM_z(\bA-\bP)(\res_\bA(z)-\res_\bP(z))}_{R_0}+\underbrace{\res_\bP(z)(\bA-\bP)\res_\bP(z)}_{R}\right)
	}
	
	As it turns out, each of the rows of $zR$ either have small Frobenius norm or they disappear when combined with $z\res_{\bA-\bP}(z)$ post integration. We will show this step by step.  Note that, using Eq~\eqref{eq:resp-decomp}, $R$ in the above equation can be decomposed as:
	\bas{
		&\res_\bP(z)(\bA-\bP)\res_\bP(z)
		=\underbrace{\frac{\bA-\bP}{z^2}}_{R_1}+\underbrace{\bM_z(\bA-\bP)\res_{\bP}(z)}_{R_2}-\underbrace{\frac{\bA-\bP}{z}\bM_z}_{R_3}
	}

	Next, we show that  $z\res_{\bA-\bP}(z)R_1$ disappears upon integration. 
	Since by construction $\forall z\in \C_k, \forall k$, $|z|\geq a_k >\|\bA-\bP\|$, none of the contours contain zero, Eq~\eqref{eq:res-series} immediately gives:
	\ba{\label{eq:r1}
		\frac{1}{2\pi\sqrt{-1}}\oint_{\C_k}z\res_{\bA-\bP}(z)R_1dz
		&=-\sum_{t\geq 1}\oint_{\C_k}{ \frac{1}{z}}\bb{\frac{\bA-\bP}{z}}^tdz=0
	}
	
	Thus Eqs~\eqref{eq:res-eigen-row},~\eqref{eq:grecursive} and~\eqref{eq:r1} give us,
	\bas{
		\be_x^T(\bV_k\bV_k^T-\vh_k\vh_k^T)&=-\frac{1}{2\pi\sqrt{-1}} \oint_{\C_k}\be_x^T\res_{\bA-\bP}(z)z(R_0+R_2-R_3)dz\nonumber\\
		\|\be_x^T(\bV_k\bV_k^T-\vh_k\vh_k^T)\|&\leq
		\frac{b_k-a_k+2\gamma_k}{\pi}\max_{z\in \C_k}\|\be_x^T\res_{\bA-\bP}(z)z(R_0+R_2-R_3)\|
	}
	Now we bound each part individually.
	\bas{
		\|\be_x^T\res_{\bA-\bP}(z)z(R_0+R_2)\|&=\|\be_x^T\res_{\bA-\bP}(z)z\bM_z(\bA-\bP)\res_{\bA}(z)\|\\
		&\leq \|\be_x^T\res_{\bA-\bP}(z)z\bM_z\|_F\|\bA-\bP\|\|\res_{\bA}(z)\|\\
		&\stackrel{(i)}{\leq}|z|\|\res_{\bA}(z)\|\|\bA-\bP\|\|\bE_z\|\|\be_x^T\res_{\bA-\bP}(z)\bV\|=:P_{1}(z) 
	}
	Step (i) uses Eq~\eqref{eq:mzdef}.
	Finally we also have: 
	\bas{
		\|\be_x^Tz\res_{\bA-\bP}(z)R_3\|&\leq \|\be_x^T\res_{\bA-\bP}(z)(\bA-\bP)\bV\| \|\bE_z\|=:P_{2}(z)
	}
	
	Thus, the statement of the lemma follows.
\end{proof}
\subsection{Proof of Lemma~\ref{lem:azuma-better}}\label{sec:azuma-better}
\input{concentration1}

\input{erdos-proof}

\input{lemma_for_abbe}

\input{lemma_for_cape}

\input{low_rank_P}

\section{Consistency of estimated quantities}\label{sec:inferred_proof}
\input{err_bound_row}

\input{error_rest_row}
\input{lemma_for_jin}
\input{lemma_for_tensor}

\input{pruning}
\input{supp-simu}

\end{document}

%% file: headers.tex
\usepackage{color}
\usepackage[utf8]{inputenc} 
\usepackage[T1]{fontenc}    
\usepackage{hyperref}       
\usepackage{url}            
\usepackage{booktabs}       
\usepackage{amsfonts}       
\usepackage{nicefrac}       
\usepackage{microtype}      

\usepackage{algorithm}
\usepackage{algorithmic}
\usepackage{wrapfig}

\usepackage{amsthm}
\usepackage{amsmath, amssymb,amsfonts, xspace, verbatim, bm}
\usepackage{alltt}
\usepackage{lmodern}
\usepackage{xr}

\usepackage{tikz,pgfplots}
\usetikzlibrary{backgrounds}

\usepackage{extarrows}
\usepackage{changepage}

\usepackage{pdfpages}
\usepackage{pgfplots}
\pgfplotsset{compat=1.7}

\DeclareMathAlphabet\mathbfcal{OMS}{cmsy}{b}{n}

\theoremstyle{plain}
\newtheorem{thm}{Theorem}[section]
\newtheorem{lem}[thm]{Lemma}
\newtheorem{cor}[thm]{Corollary}

\theoremstyle{definition}
\newtheorem{thmDef}{Theorem}[section]
\newtheorem{defn}[thmDef]{Definition}
\newtheorem{assumption}[thmDef]{Assumption}

\theoremstyle{plain}
\newtheorem{thmrem}{Theorem}[section]
\newtheorem{rem}[thmrem]{Remark}

\newcommand{\bb}[1]{\left(#1\right)}
\newcommand{\bas}[1]{\begin{align*}#1\end{align*}}
\newcommand{\ba}[1]{\begin{align}#1\end{align}}

\newcommand{\txtred}[1]{}

\newcommand{\var}{\mbox{var}}

\newcommand{\bzero}{{\bf 0}}
\newcommand{\bone}{{\bf 1}}

\newcommand{\bA}{{\bf A}}
\newcommand{\bB}{{\bf B}}

\newcommand{\bC}{{\bf C}}

\newcommand{\be}{{\bf e}}
\newcommand{\bE}{{\bf E}}
\newcommand{\bF}{{\bf F}}
\newcommand{\bG}{{\bf G}}
\newcommand{\bH}{{\bf H}}
\newcommand{\bI}{{\bf I}}

\newcommand{\bM}{{\bf M}}
\newcommand{\bO}{{\bf O}}
\newcommand{\bP}{{\bf P}}

\newcommand{\bR}{{\bf R}}

\newcommand{\bu}{{\bf u}}
\newcommand{\bU}{{\bf U}}
\newcommand{\bv}{{\bf v}}
\newcommand{\bV}{{\bf V}}
\newcommand{\bW}{{\bf W}}

\newcommand{\bS}{{\bf S}}
\newcommand{\bx}{{\bf x}}
\newcommand{\bX}{{\bf X}}

\newcommand{\bY}{{\bf Y}}

\newcommand{\balpha}{{\boldsymbol \alpha}}
\newcommand{\bbeta}{{\boldsymbol \beta}}
\newcommand{\btheta}{{\boldsymbol \theta}}
\newcommand{\bTheta}{{\boldsymbol \Theta}}

\newcommand{\diag}{\text{diag}}

\newcommand{\bbb}[1]{\left(#1\right)}
\newcommand{\ccc}[1]{\left[#1\right]}
\newcommand{\ddd}[1]{\left\{#1\right\}}

\newcommand{\A}{\mathcal{A}}

\newcommand{\C}{\mathcal{C}}

\newcommand{\cS}{\mathcal{S}}

\newcommand{\R}{\mathbb{R}}

\newcommand{\rank}{\mathrm{rank}}

\newcommand{\cov}{\mathrm{Cov}}
\newcommand{\conv}{\mathrm{Conv}}

\newcommand{\uE}{\mathrm{E}}

\newcommand{\uP}{\mathrm{P}}

\newcommand{\hv}{\hat{\bf V}\xspace}

\newcommand{\hV}{\hat{\bf V}\xspace}

\newcommand{\hE}{\hat{\bf E}\xspace}

\newcommand{\hx}{\boldsymbol X\xspace}

\newcommand{\hxp}{\hx_p}
\newcommand{\hvp}{\hv_p}

\newcommand{\bN}{{\bf N}}

\newcommand{\cSp}{\cS_p\xspace}

\newcommand{\bal}{\nu\xspace}
\newcommand{\bpi}{{\bf \Pi}}

\newcommand{\E}{\mathcal{E}}
\newcommand{\evz}{\E'}
\newcommand{\evt}{\E_t}
\newcommand{\evot}{\E'}
\newcommand{\evze}{\stackrel{\evz}{=}}

\newcommand{\evote}{\stackrel{\evot}{=}}

\newcommand{\OurAlgo}{{\textsf SPACL}\xspace}

\newcommand{\vh}{\hat{\bV}}
\newcommand{\oh}{\hat{\bO}}

\newcommand{\lh}{\hat{\lambda}}
\newcommand{\res}{\bG}

\newcommand{\RelativeErrorB}{\tilde{O}\bbb{\frac{\gammaP \kappa({\bTheta^T\bTheta})Kn^{1.5} }{\sqrt{\rho}\lambda^*(\bB)(\lambda_K(\bTheta^T\bTheta))^{2}}}}

\newcommand{\ErrorBDirichletBalanced}{ \tilde{O}\bbb{\frac{\min\{K,\kappa(\bB)\}^2 K^{3}}{\sqrt{\rho n}\lambda^*(\bB)}} }

\newcommand{\ThetaError}{ \tilde{O}\bbb{\frac{\gammaP  (\kappa({\bTheta^T\bTheta}))^{1.5} \sqrt{Kn}  }{\sqrt{\rho}\lambda^*(\bB)\lambda_K(\bTheta^T\bTheta)}} }

\newcommand{\ThetaErrorDirichletBalanced}{ \tilde{O}\bbb{\frac{\min\{K,\kappa(\bB)\}^2K^{1.5}}{\sqrt{\rho n}\lambda^*(\bB)}} }
\newcommand{\eigenspacerowwise}{ \tilde{O}\bbb{\frac{\gammaP \sqrt{Kn} }{\sqrt{\rho}\lambda^*(\bB)(\lambda_K(\bTheta^T\bTheta))^{1.5}}}}

\newcommand{\additionalone}{\ensuremath \lambda^*(\bB)=\tilde{\Omega}\bb{\dfrac{\gammaP ({\kappa(\bTheta^T\bTheta)})^{1.5}K\sqrt{n}}{\sqrt{\rho}\lambda_K(\bTheta^T\bTheta)}}\xspace}

\newcommand{\gammaP }{\psi(\bP)}

%% file: abstract.tex
\begin{abstract}
We consider the problem of estimating community memberships of nodes in a network, where every node is associated with a vector determining its degree of membership in each community.
Existing provably consistent algorithms often require strong assumptions about the population, are computationally expensive, and only provide an overall error bound for the whole community membership matrix.
This paper provides uniform rates of convergence for the inferred community membership vector of {\em each} node in a network generated from the Mixed Membership Stochastic Blockmodel (MMSB); to our knowledge, this is the first work to establish per-node rates for overlapping community detection in networks.
We achieve this by establishing sharp row-wise eigenvector deviation bounds for MMSB. 
Based on the simplex structure inherent in the eigen-decomposition of the population matrix, we build on established corner-finding algorithms from the optimization community to infer the community membership vectors.
Our results hold over a broad parameter regime where the average degree only grows poly-logarithmically with the number of nodes. 
Using experiments with simulated and real datasets, we show that our method achieves better error with lower variability over competing methods, and processes real world networks of  up to 100,000 nodes within tens of seconds.
\end{abstract}

%% file: intro.tex
\section{Introduction}
\label{sec:intro}
In most real-world networks, a node belongs to multiple communities. In an university, professors have joint appointments to multiple departments;  a movie like ``Dirty Harry'' in the Netflix recommendation network belongs to action, thriller, and the drama genre according to Google; in a book recommendation network like goodreads.com, ``To Kill a Mockingbird'' can be classified as a classic, historical fiction, young-adult fiction, etc.  
The goal of community detection is to consistently infer each node's community memberships from just the network structure.

A well-studied variant of this problem assumes that each node belongs to a single community.
For instance, under the Stochastic Blockmodel (SBM)~\citep{holland_stochastic_1983}, the probability of a link between two nodes depends only on their respective communities. 
Thus, provably consistent inference  under the Stochastic Blockmodel involves finding the unknown cluster membership of each node (see~\citep{lei2015consistency,rohe2011spectral,mcsherry2001spectral}) and these are not immediately applicable for the general problem where a node may belong to multiple communities to different degrees.

In this paper, we work with the popular Mixed Membership Stochastic Blockmodel (MMSB)~\citep{airoldi2008mixed}. This generalizes the Stochastic Blockmodel by letting each node $i$ have different degrees of membership in all communities. In particular, each node $i$ is associated with a community membership vector $\btheta_i\in \R^K$ ($\btheta_i \geq 0, \|\btheta_i\|_1 = 1$), drawn from a Dirichlet prior. The model for generating the symmetric adjacency matrix is as follows:
~\footnote{{Note that self-loops are allowed here for simplicity of analysis. Without them, the analysis gets cumbersome, leading to a negligible error term added to all our bounds, and we skip it for ease of exposition.}}
\begin{equation}\label{eq:mmsb}
\begin{aligned}
\btheta_i&\sim \text{Dirichlet} (\balpha) & \balpha\in \mathbb{R}_+^K,\quad i\in [n]\\
\bP&:=\rho \bTheta \bB \bTheta^T & \bA_{ij}=\bA_{ji}\sim \text{Bernoulli}(\bP_{ij})\quad i,j\in[n]
\end{aligned}
\end{equation}
The matrix $\bTheta$ has $\btheta_i^T$ as its $i^{\mathrm{th}}$ row.
For identifiability we assume $\max_{ij} \bB_{ij}=1$.
When $\bB$ has higher values on its diagonal as compared to the off-diagonal, edges are likely between nodes that have a high membership in the same community.  These are called assortative communities.
In contrast, in {disassortative} settings, off-diagonal elements are larger than diagonal elements.
Bipartite graphs are an extreme case of this.
The smallest singular value of $\bB$, denoted by $\lambda^*(\bB)$, is a measure of the separation between communities. 
A larger $\lambda^*(\bB)$ corresponds to more well-separated communities.
The parameter $\rho$ controls the the expected average degree of nodes $O(n\rho)$.
We allow both $\rho$ and  $\lambda^*(\bB)$ to go to zero with increasing number of nodes $n$.
The quantity $\alpha_0 = \sum_{a=1}^{K} \alpha_a$ controls the level of overlap between members of different communities.
As $\alpha_0 \rightarrow 0$, the MMSB model degenerates to the Stochastic Blockmodel.
The goal of community detection under the MMSB model is to recover $\bTheta$ and $\bB$ from the observed adjacency matrix $\bA$.

Prior work on this problem include MCMC~\citep{airoldi2008mixed} and computationally efficient variational approximation methods~\citep{gopalan2013efficient} (SVI) which do not have any guarantees of consistency. Other interesting network models for overlapping communities and non-negative matrix factorization style inference methods which do not have theoretical guarantees include ~\citep{Ball2011Overlapping,wang2011community, wang2016supervised,BNMF2011}. A notable family of algorithms that has been shown to be theoretically consistent uses tensor-based methods~\citep{MMSBAnandkumar2014,hopkins2017bayesian}. However, these are typically hard to implement, and provide overall error bounds for the columns of the estimated $\bTheta$ matrix.

Recently~\citet{mao2017} have proposed a provably consistent geometric algorithm (GeoNMF) for MMSB with diagonal $\bB$ and $\balpha=\alpha_0\bone_K/K$. However the guarantees only work in the dense regime where average degree grows faster than $\sqrt{n}$. In contrast we consider the general model where the only condition on $\bB$ is full rank. We propose a different algorithm which works when the  degree grows faster than poly-logarithm of $n$.        

\citet{zhang2014detecting} propose a provably consistent spectral algorithm (OCCAM) for a related but different model with degree correction. 
Similar to non-negative matrix factorization methods~\citep{arora12computing,mao2017}, the authors assume that 
each community has some ``pure'' nodes (which only belong to that community).
The authors also assume that $\bB$ is positive semidefinite and full rank with equal diagonal entries. Other assumptions ensure that the $k$-medians loss function on $\btheta_i$ attains its minimum at the locations of the pure nodes and there is a curvature around this minimum. This condition is typically hard to check.

Concurrent work~\cite{jin2017estimating} studies the degree corrected MMSB model, which extends the MMSB model by allowing degree heterogeneity. 
The authors  show an interesting fact that the top eigenvectors, normalized appropriately, still form a simplex. 
However, their proposed algorithm requires a combinatorial search step (SVS)\footnote{In the latest version of~\citep{jin2017estimating}, the authors have added other methods, and proved node-wise error bounds. But they note that among these methods, SVS performs the best. We compare against the newer bounds later in our paper.}, and has a complexity $O(n^{KL}+K^3L^{K+1})$ for some tuning parameter $L\geq K$.
This can be prohibitive for large $K$.
SVS is analyzed under three separate settings, a) $\btheta_i$ are sampled from a distribution on the simplex such that every cluster has $\Theta(n)$ pure nodes, and the non pure nodes are sufficiently separated from the pure ones; b) the $\btheta_i$'s are fixed, but form a few clusters, or c) the $\btheta_i$'s are fixed, and most nodes are pure nodes.

 Other notable examples of related but different models include ~\citep{ray2014overlap,kaufmann2016spectral}. In~\citep{ray2014overlap}, the authors show consistency when the overlap between clusters is small, whereas in ~\citep{kaufmann2016spectral},  a combinatorial algorithm (SAAC) is proposed for detecting overlapping communities for a related model.

In this paper, our contributions are as follows.
 
\smallskip\noindent
{\bf Identifiability:} We present both necessary and sufficient conditions for identifiability of the MMSB model in Sec~\ref{sec:proposed}.  
To our knowledge, we are the first to report both necessary and sufficient conditions for identifiability under the MMSB model.

\smallskip\noindent
{\bf Recovery algorithm:} As shown by many authors~\citep{mao2017,jin2017estimating,panov2017consistent},  the population eigenvectors (i.e., eigenvectors of the matrix $\bP$) form a  rotated and scaled simplex.  
We present an algorithm called \OurAlgo, which re-purposes an existing algorithm~\citep{gillis2014fast} for detecting corners in a rotated and scaled simplex to find pure nodes, and then uses these to infer $\bTheta$ and $\bB$.  It also includes a novel preprocessing step that improves performance in sparse settings. The main compute-intensive parts of the algorithm are a)  top-$K$ eigen-decomposition of $\bA$, b) calculating $k$-nearest neighbors of a point for preprocessing. There are highly optimized algorithms and data structures for both of these steps~\citep{Bentley:1975:kd,press92numerical,Beygelzimer:2006:covertree}.

\input{notation_table}
\smallskip\noindent
{\bf Node-wise error bound:} 
Some of the existing works on MMSB type models show consistency in terms of the deviation or correlation of $\hat{\bTheta}$ as a whole with respect to the truth~\cite{zhang2014detecting,kaufmann2016spectral,hopkins2017bayesian,jin2017estimating}. Others establish consistency of the deviation of columns of $\hat{\bTheta}$~\cite{MMSBAnandkumar2014} (soft memberships of all nodes to a particular community) from their population counterpart. 
In contrast, we obtain a uniform rate of convergence of \textit{each} cluster membership vector $\hat{\btheta}_i, i\in[n]$  to $\btheta_i$. 
To our knowledge this is the first work to establish uniform node-wise error bounds for an estimation algorithm for overlapping network models. 

\smallskip\noindent
{\bf Empirical validation:} In Sec~\ref{sec:exp}, we compare \OurAlgo with OCCAM, variational methods, SAAC and existing non-negative matrix factorization algorithms (GeoNMF, BSNMF) on both simulated and large real world networks with up-to 100,000 nodes.  

%% file: notation_table.tex
\begin{table}[!t]
	\caption{\label{table:notation}Table of notations. $K$ leading eigenvectors of a matrix correspond to $K$ largest eigenvalues in magnitude.}
	\centering
	\scalebox{0.76}{
		\begin{tabular}{ | l | l ||l|l|}
			\hline
			\rule{0pt}{3ex}    
			$n$& Number of nodes & $K$&Number of communities\\[.5em]
			$\rho\bB\in [0,1]^{K\times K}$&  Community link probabilities ($\bB=\bB^T$)&
			$\balpha\in \R_+^{K\times 1}$& Dirichlet prior parameters
			\\[.5em]
			$\bTheta\in \R_+^{n\times K}$& Fractional community memberships
			& $\alpha_0$&$\sum_i\alpha_i$\\[.5em]
			$\alpha_{\min}$ ($\alpha_{\max}$)& $\min_{i\in [K]} \alpha_i$ ($\max_{i\in [K]} \alpha_i$)&$\bal$&$\alpha_0/\alpha_{\min}$\\[.5em]
			$\bA$ & Adjacency matrix & \bP& $\rho\bTheta\bB\bTheta^T$\\[.5em]
			$\rho$& Upper bound on $\bP_{ij}$ &$\bI_m$& $m\times m$ identity matrix\\[.5em]
			$\bE$& Diagonal matrix of $K$ largest  & $\hE$ & Diagonal matrix of $K$ largest  \\[.5em]
			& eigenvalues in magnitude of $\bP$ & & eigenvalues in magnitude of $\bA$\\[.5em]
			$\bV\in \R^{n\times K}$& $K$ leading eigenvectors of $\bP$  & $\hV\in \R^{n\times K}$ & $K$ leading eigenvectors of $\bA$\\[.5em]
			
			$\bV_P\in \R^{K\times K}$ & True $K$ pure node index rows of $\bV$ & $\lambda_K(\bM)$ & $K^{th}$ largest eigenvalue of $\bM$\\[.5em]
			$\bV_p\in \R^{K\times K}$ & Estimated $K$ pure node index rows of $\bV$ &
			$\lambda^*(\bM)$ & $K^{th}$ largest singular value of $\bM$\\[.5em]
			{$\kappa(\bM)$} & Condition number of matrix $\bM$ & $\lambda_i$ & $i^{th}$ largest eigenvalue of $\bP$\\[.5em]
			$\bpi\in \{0,1\}^{K\times K}$&Permutation matrix & $\hat{\lambda}_i$ & $i^{th}$ largest eigenvalue of $\bA$\\[.5em]
			$\bone_m$&All ones vector of length $m$  & $\be_i$& $\be_i(j)=1(i=j)$ \\[.5em]
			\hline
		\end{tabular}
	}
\end{table}

%% file: proposed.tex
\section{Notations, Identifiability and Algorithms}
\label{sec:proposed}
Before presenting our results on identifiability we introduce some notations and assumptions. 
Let $[n]:=\{1,2,\cdots,n\}$. 
For any matrix $\bM$, {we use $\bM(i,:)$/$\bM(:,i)$, $\bM(S,:)$/$\bM(:,S)$ to denote the $i^{th}$ row/column of matrix $\bM$ and the submatrix formed by rows/columns in set $S$ of matrix $\bM$ respectively, and $S=i:j$ denotes the set of indices from $i$ to $j$.}
We use $\|\bM\|$ and $\|\bM\|_F$ to respectively denote the operator and Frobenius norms of a matrix $\bM$, and $\|\bv\|$ to denote the Euclidean norm of a vector $\bv$. 
We denote $[\bX|\bY]$ as the concatenation of columns of matrices $\bX$ and $\bY$.
We use $\tilde{O}$ and $\tilde{\Omega}$ to denote upper and lower bounds up to poly-logarithmic factors.
Finally we present a consolidated list of notations in Table~\ref{table:notation}.

We shall now provide necessary and sufficient conditions for the identifiability of the MMSB model with respect to  $\bTheta$ and $\bB$.
\subsection{Identifiability}
In this section, we obtain necessary and sufficient conditions for the identifiability of MMSB. In contrast, prior work~\citep{zhang2014detecting,mao2017,kaufmann2016spectral,panov2017consistent} typically establishes sufficient conditions. 
We defer the proofs of the theorems in this section to the Appendix (Sec~\ref{sec:identifiability}). 

Define a pure node as a node which belongs to exactly one community. All nodes in a Stochastic Blockmodel are pure nodes, since every node belongs to exactly one community.
Define a ``completely mixed'' node as a node $m$ such that $\theta_{mj}>0$ for all $j\in [K]$.

\begin{thm}\label{mmsb_iden_nece_not}
	Suppose there are $K$ communities, with at least one pure node for each community.
  Then,
	\begin{enumerate}
		\vspace{-0.3em}
		\item [(a)] If $\rank(\bB) = K$, then the MMSB model is identifiable up to a permutation.
		\vspace{-0.6em}
		\item [(b)] If $\rank(\bB) = K-1$, and no row of $\bB$ is an affine combination of the other rows of $\bB$, then the MMSB model is identifiable up to a permutation.
		\vspace{-0.6em}
		\item [(c)] In any other case, if there exists a completely mixed node, then the model is not identifiable.
	\end{enumerate}
\end{thm}

\begin{thm}\label{mmsb_iden_nece}
	Suppose that $\rho \bB_{ij} \in (0,1)$ for all $i,j\in[K]$. MMSB is identifiable up to a permutation only if there is at least one pure node for each of the $K$ communities.
\end{thm}
The above theorems show that the existence of pure nodes is necessary in most practical scenarios.

\subsection{Algorithm}
We do inference for the MMSB model under the following assumption, which is sufficient for identifiability.
\begin{assumption}\label{as:ident}
$\bB\in\R^{K\times K}$ is full rank, and there is at least one pure node for each of the $K$ communities. 
\end{assumption}
{Since the Dirichlet distribution does not give rise to pure nodes, we assume that the set $\{\btheta_i,i\in[n]\}$ includes one pure nodes from each cluster in addition to $n-K$ vectors drawn from a Dirichlet. 
The addition of one pure node per cluster to the standard Dirichlet draws does not affect the analysis and we ignore this for ease of exposition. }

We will now discuss our inference algorithm, whose consistency results are presented in  Sec~\ref{sec:main}.
Let $\bP=\bV \bE \bV^T$ be the top-$K$ eigendecomposition of $\bP$. 
We proceed from a simple observation that the population eigenvectors lie on a rotated and scaled simplex, as shown next. The following lemma is the starting point of most existing analysis for Stochastic Blockmodels, and different variants of this have been observed independently by a number of other researchers~\citep{panov2017consistent,jin2017estimating,mao2017}.
 
 \begin{lem} \label{lem:v-theta}
 	Let $\bV$ be the top $K$ eigenvectors of $\bP$. Then, under Assumption~\ref{as:ident}, $\bV = \bTheta\bV_P$, where $\bV_P = \bV(\mathcal{I},:)$ {is full rank} and $\mathcal{I}$ is the indices of rows corresponding to $K$ pure nodes, one from each community.
 \end{lem}

\begin{proof}
	W.L.O.G., reorder the nodes so that $\bTheta(\mathcal{I},:) = \bI$.
	Then, $\bV_P \bE \bV_P^T = \bP(\mathcal{I}, \mathcal{I}) = \rho\bB$, so $\bV_P\in\R^{K\times K}$ is full rank.
	Now, observe that
	$\bV_P \bE \bV^T = \bP(\mathcal{I},:) = \rho \bTheta(\mathcal{I},:~) \bB \bTheta^T= \rho \bB \bTheta^T.$
	Hence,
	$\bV = \bP \bV \bE^{-1} = \rho \bTheta \bB \bTheta^T \bV \bE^{-1} = \bTheta \bV_P \bE \bV^T\bV \bE^{-1}=\bTheta \bV_P.$
\end{proof}
\input{algo_fig}

Lemma~\ref{lem:v-theta} establishes that the corners of the simplex have the highest norm.
This allows us to find the pure nodes using existing corner-finding methods such as the successive projection algorithm (SPA)~\citep{gillis2014fast}.

Our algorithm, called ``Sequential Projection After CLeaning'' (\OurAlgo, Algorithm~\ref{alg:nmf-mmsb-pure-res}) applies SPA after a preprocessing step that prunes away noisy high-norm points.
SPA first finds the node with the maximum row norm of empirical eigenvector matrix $\hat{\bV}$. This node is added to the set of pure nodes. Then, all remaining rows of $\hat{\bV}$ are projected on to the subspace that is orthogonal to the span of the pure nodes. The process is repeated for $K$ iterations, and yields a set of $K$ pure nodes, one from each community. 
With the pure nodes in hand, \OurAlgo estimates $\bTheta$ and $\bB$ using Lemma~\ref{lem:v-theta}.
We will show that these estimates are consistent up to a permutation (Theorem~\ref{thm:theta-B}).

If we had access to the population eigenvectors $\bV$, SPA would return the true pure nodes. 
However, in reality we only observe the empirical eigenvectors, which are noisy versions of the population eigenvectors. 
So there can be spurious nodes with row norm larger than those of the ``pure'' nodes. 
As the graph gets sparser, the empirical points deviate more from the population simplex.
This motivates the pruning step of \OurAlgo.
The main idea of pruning is to identify and remove the nodes which are far away from the population simplex. 
Algorithm~\ref{alg:prune} finds these by first finding contenders of pure nodes, i.e., nodes $i$ whose eigenvector rows $\hat{\bV}_i:=\be_i^T\hat{\bV}$ have large norm. 
Among these, it prunes nodes which do not have too many nearest neighbors, or in other words, have larger average distance to their nearest neighbors in comparison to others.
The removal of these nodes improves the performance of SPA on sparse networks.

\begin{figure}[!t]
	\centering
	\begin{tabular}{@{\hspace{0em}}c@{\hspace{0em}}c@{\hspace{0em}}}
		\includegraphics[width=0.4\textwidth]{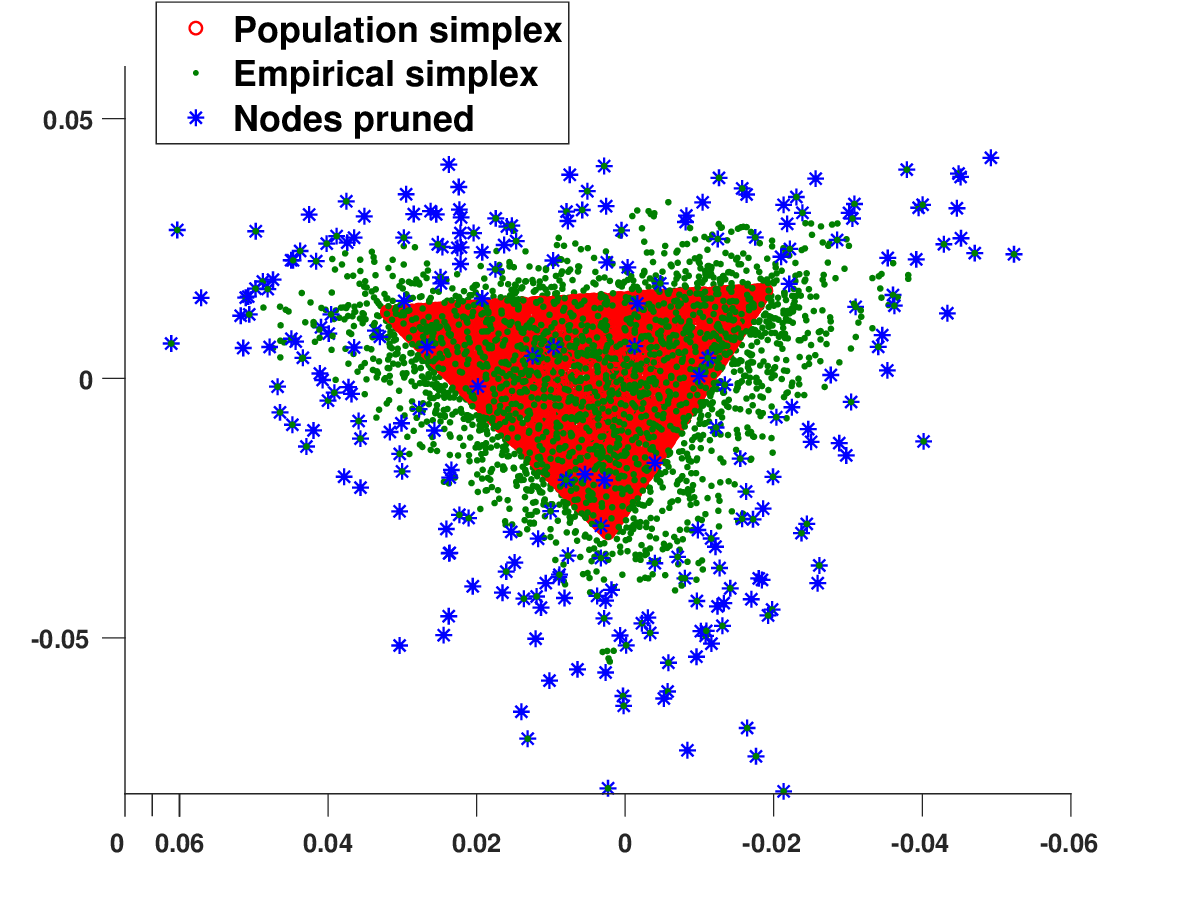}	\quad&
		\includegraphics[width=0.4\textwidth]{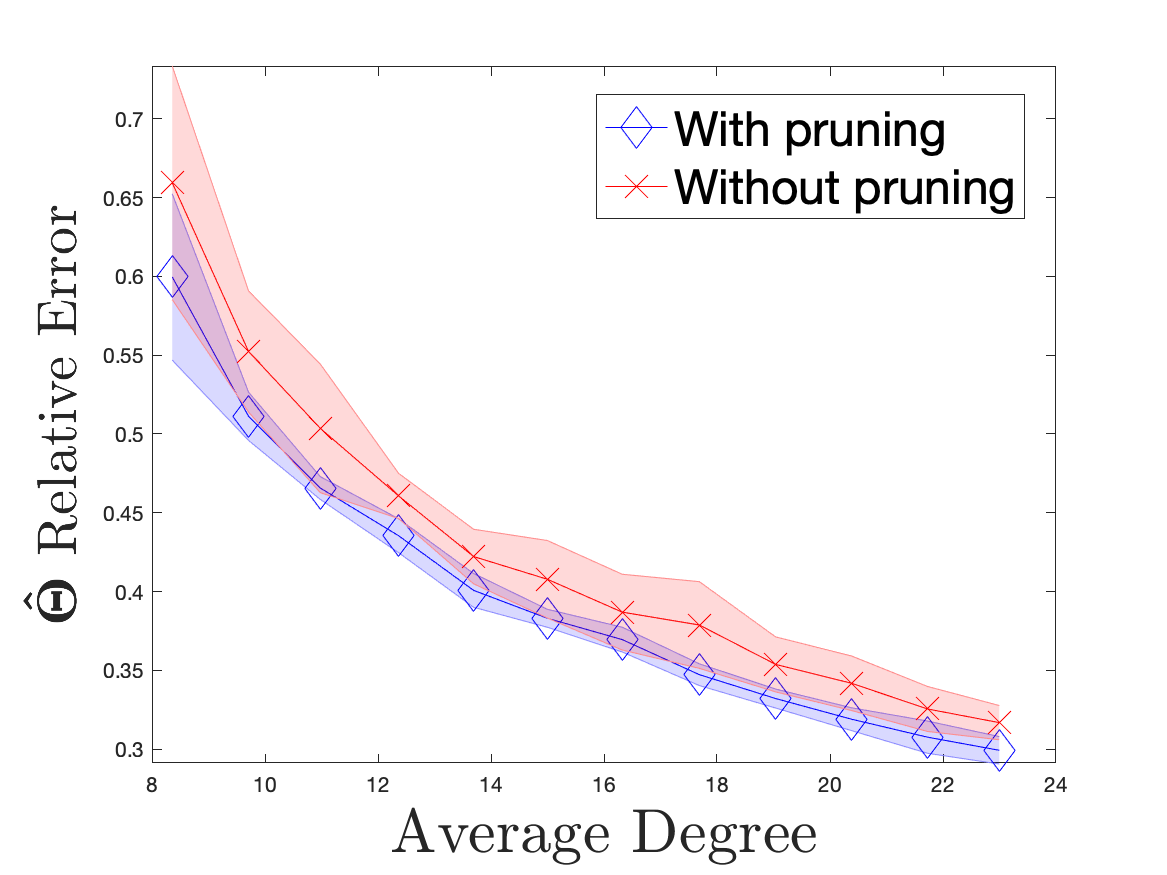}
		\\
		{\small(A)}&{\small(B)}
	\end{tabular}
	\caption{\label{fig:pruning}
		MMSB model with $n=5000$, $\balpha=(0.4,0.4,0.4)$, $\bB=(1-q)\bI_3+q\bone_3\bone_3^T$ with $q=0.001$. (A) Nodes picked out by Pruning with $\rho=0.007$. (B) Effect of pruning on estimating $\hat{\bTheta}$ {(relative error defined in Sec~\ref{exp:sim})}.  }
\end{figure}

{Fig~\ref{fig:pruning} (A) shows the benefits of pruning on a simulated network.
After pruning, the remaining nodes are closer to the population simplex.
This leads to better estimation.  
Fig~\ref{fig:pruning} (B) varies $\rho$ from 0.0050 to 0.0138 leading to average degrees increasing from 8 to 23, and shows the effect of pruning (blue $\Diamond$) over not pruning (red $\times$) on the relative estimation error of $\bTheta$. 
A more detailed discussion on pruning can be found in the Appendix (Sec~\ref{sec:prune_works}). 
}

%% file: algo_fig.tex
\begin{figure}[!t]
	\begin{minipage}[t]{0.47\linewidth}
		\vspace{-0.8em}
		\begin{algorithm}[H]
			\caption{\OurAlgo}
			\label{alg:nmf-mmsb-pure-res}
			\begin{algorithmic}[1]
				\REQUIRE Adjacency matrix $\bA$, number of clusters $K$
				\ENSURE  $\hat{\bTheta}$, $\hat{\bB}$, $\hat{\rho}$.
				\STATE Get the top-$K$ eigen-decomposition of
				$\bA$ as $\hat{\bV}\hat{\bE}\hat{\bV}^T$. 
				\STATE $S=\mbox{Prune}(\hv,10,.75,.95)$
				\STATE $\hx=\hv([n]\setminus S,:)$
				\STATE $\cSp=\mathrm{SPA}(\hx^T)$
				\STATE $\hxp = \hx(\cSp,:)$
				\STATE $\hat{\bTheta}=\hat{\bV}\hxp^{-1}$.
				\STATE $\hat{\bTheta} = \diag(\hat{\bTheta}_+ \bone_K)^{-1}\hat{\bTheta}_+$
				\STATE $\hat{\bB} = \hxp\hat{\bE}\hxp^T$
				\STATE $\hat{\rho} = \max_{i,j} \hat{\bB}_{ij}$. $\hat{\bB} = \hat{\bB} / \hat{\rho}$
			\end{algorithmic}
		\end{algorithm}
	\end{minipage}\hspace{1em}
\begin{minipage}[t]{0.51\linewidth}
	\vspace{-0.8em}
	\begin{algorithm}[H]
		\caption{Prune}
		\label{alg:prune}
		\begin{algorithmic}[1]
			\REQUIRE Empirical eigenvectors $\hat{\bV}\in \R^{n\times K}$, an integer $r$, and two numbers ${q,\varepsilon\in(0,1)}$. 
			\ENSURE Set $S$ of nodes to be pruned.
			\vspace{0.02in}
			\FOR{{$i\in n$}}
			\STATE $v_i=\|\be_i^T\hv\|$
			\ENDFOR
			\STATE $S_0=\{i:\|\be_i^T\hat{\bV}\|\geq \mathrm{quantile}(\bv,q)\}$
			\FOR{{$i\in S_0$}}
			\STATE $d_i:=$\{\mbox{Dist. to $r$  nearest neighbors}\}
			\STATE $x_i=\sum_j d_{ij}/r$
			\ENDFOR
			\STATE $S=\{i:x_i\geq \mathrm{quantile}(x,1-\varepsilon)\}$
			\vspace{0.005in} 
		\end{algorithmic}
	\end{algorithm} 
\end{minipage}
\end{figure}

%% file: analysis.tex
\section{Main results}
\label{sec:main}
We want to prove that the sample-based estimates $\hat{\bTheta}$, $\hat{\bB}$ and $\hat{\rho}$ concentrate around their population counterparts, respectively, $\bTheta$, $\bB$, and $\rho$. 
By Lemma~\ref{lem:v-theta}, this requires concentration of the rows of the empirical eigenvector matrix $\hat{\bV}$ to the population counterpart $\bV$.
Existing techniques like the Davis-Kahan Theorem~\citep{yu2015useful} only provide convergence in the Frobenius norm $\|\bV-\hat{\bV}\bO\|_F$ (for some rotation matrix $\bO$) or the operator norm $\|\bV\bV^T-\hat{\bV}\hat{\bV}^T\|$.
These lead to loose bounds on the rows of $\hat{\bV}$. 
Other existing techniques~\citep{BalaSpec2011, Athreya2016,mao2017} can be applied to show that rows of $\hat{\bV}$ have $\tilde{O}_P(1/\sqrt{n\rho^2})$ relative error, but this is only meaningful when the degree grows faster than square root of $n$, i.e. the dense degree regime. 
We show that, under a broad parameter regime, the suitably defined relative deviation of any row of $\hat{\bV}$ from its population counterpart, converges to zero when average degree only grows faster than the poly-logarithm of $n$.

In Section~\ref{sec:main:evec}, we show the row-wise eigenspace error bound in terms of eigenvalues of $\bTheta^T\bTheta$.
In Section~\ref{sec:main:modelparams}, we translate the eigenspace bounds into error bounds on estimated $\hat{\bTheta}$ and $\hat{\bB}$ matrices.
Then, in Section~\ref{sec:dirichlet}, we provide detailed results when the rows of $\bTheta$ are drawn i.i.d from a Dirichlet distribution.
Throughout, we compare our bounds to other bounds in concurrent results.
We also discuss the implications of our results for specific models like the Stochastic Blockmodel.

\subsection{Row-wise eigenvector error bounds}
\label{sec:main:evec}
\begin{assumption}\label{as:theta_P}
	Assume $\rho n =\Omega(\log n)$, $\lambda_K(\bTheta^T\bTheta)\geq1/\rho$, and $\lambda^*(\bP)\geq 4\sqrt{n\rho}(\log n)^{\xi}$ for some constant $\xi>1$.
\end{assumption}

\begin{thm}[Row-wise eigenspace error] 
	\label{thm:entrywise}
	If Assumptions~\ref{as:ident} and~\ref{as:theta_P} are satisfied, 
	then 
	with probability at least $1-O(Kn^{-2})$, 
	\ba{\label{eq:eigen_row_thm}
		\max_{i\in[n]} \left\| \be_i^T( \hat{\bV}\hat{\bV}^T-\bV  \bV^T) \right\|&=\eigenspacerowwise,
	}
	where $\gammaP$  measures how well the eigenvalues of $\bP$ can be packed into bins. 
	The precise definition is deferred to Eq~\eqref{eq:gammap}, Sec~\ref{sec:analysis} for ease of exposition.
\end{thm}
Later, we will show that $\gammaP\leq 2\min\{K,\kappa(\bP)\}^2$ in the worst case.
But $\gammaP=O(1)$ if the eigenvalues of $\bP$ can be divided into a constant number of bins where each bin has eigenvalues of the same order.

\begin{rem}[Generalizing to low rank population matrices]
In the Appendix (Sec~\ref{sec:generalP}), we also establish a similar eigenvector deviation result for networks generated from general low rank population matrices. 
\end{rem}
\begin{rem}[Row-wise eigenvector error]
	Note that the above row-wise error immediately gives us an error bound on rows of $\hV$,
	\bas{
\|\be_i^T(\hV-\bV(\bV^T\hV))\|= \|\be_i^T(\hV\hv^T-\bV\bV^T)\hV\|\leq \|\be_i^T(\hV\hv^T-\bV\bV^T)\|.
}
The $K\times K$ matrix $\bV^T\hV$ takes out the projection of $\bV$ on $\hV$ from $\hV$. {Note that while we use $\bV^T\hV$ to align $\bV$ and $\hV$, most existing literature uses its matrix sign function~\citep{abbe2017entrywise,cape2019two}. An detailed example can be found in Lemma~\ref{lem:Abbe} in the Appendix.}
\end{rem}

The proof of Theorem~\ref{thm:entrywise} can be found in Sec~\ref{sec:analysis}. A key element in the proof is the delocalization of population eigenvectors.
\begin{lem}[Delocalization of population eigenvectors] \label{lem:v-row-norm} 
	We have that, $\max_i \| \be_i^T\bV \|^2\leq 1/\lambda_K(\bTheta
	^T\bTheta)$ and $\min_i\| \be_i^T\bV \|^2
	\geq 1/(K\lambda_1(\bTheta
	^T\bTheta))$.
\end{lem}

We defer the proof to Appendix (Sec~\ref{sec:auxiliary}).
Using this, we can prove the following.
\begin{cor}[Row-wise relative convergence] 
	\label{cor:entrywise}
	If Assumption~\ref{as:theta_P} is satisfied, and furthermore,
	$\lambda_K(\bTheta^T\bTheta)=\Omega(n/K)$, $K=\Theta(1)$ and $\lambda^*(\bB)=\Omega(1)$, then 
	\bas{
		\max_{i\in[n]} \frac{\left\| \be_i^T( \hat{\bV}\hat{\bV}^T-\bV  \bV^T) \right\|}{\|\be_i^T\bV\bV^T\|}&=\tilde{O}\bbb{\frac{1}{\sqrt{n\rho}}}
	}
	with probability at least $1-O(Kn^{-2})$.
\end{cor}

In concurrent work on MMSB models~\citep{panov2017consistent}, analysis of empirical eigenvectors  yields a suboptimal $\tilde{O}_P(1/\sqrt{n\rho^2})$ rate on the Frobenius norm of the overall deviation of the whole community membership matrix from its population counterpart, thereby proving consistency only in the regime where average degree grows faster than square root of $n$, not poly-logarithm of $n$. While concurrent developments on entry-wise eigenvector analysis~\citep{abbe2017entrywise,eldridge2017unperturbed,cape2018signal} obtain the better $\tilde{O}_P(1/\sqrt{n\rho})$ rate, they either have a relatively worse dependence on $\lambda^*(\bB)$ or implicitly assume that the population eigenvalues are of the same order. In~\citep{cape2018signal}, the authors assume that $K$ grows slower than poly-log of $n$.
We show that the row-wise eigenvector bounds in~\citep{abbe2017entrywise} yield a worse dependence of $\lambda^*(\bB)$ than ours in the Appendix  (Sec~\ref{sec:supp-abbe}).  
We achieve this better dependence on $\lambda^*(\bB)$ by a new construction in which we consider groups of population eigenvalues lying within specially constructed intervals, such that the ratio of the largest and smallest eigenvalues within any interval is controlled. 
Note that, if the population eigenvalues are of the same order, average expected degree in~\citep{abbe2017entrywise} can be a constant times $\log n$, whereas we require it to grow faster than $\log^2 n$. 

{We can show that our bound is tighter by an order of $1/\sqrt{n\rho}$ than a direct application of the concentration bounds for general singular subspaces established in~\citep{cape2019two} to the MMSB model. While it is possible to improve this bound by using our theoretical results and careful analysis similar to that of the $\rho$-correlated SBM graphs in~\citep{cape2019two}, even then, our row-wise eigenspace error bound is tighter by a factor of $\sqrt{\rho}$ under a broad parameter regime, a detailed discussion of which is deferred to Sec~\ref{sec:supp-cape} of the Appendix along with derivations.}

So far we have talked about row-wise bounds on empirical eigenspaces. But it seems cumbersome to apply our algorithms on the $n\times n$ $\vh\vh^T$
matrix. The following simple result shows that our algorithms return the same set of pure nodes using $\hv$ and $\hv\hv^T$ (proof in Sec~\ref{sec:inferred_proof} of the Appendix). Thus, for the algorithm we simply use $\hv$.
\begin{lem}\label{lem:vvt-v-equiv}
	The pruning algorithm (Algorithm~\ref{alg:prune}) and the SPA algorithm will return the same node indices on both $\hv$ ($\hv^T$ for SPA) and $\hv\hv^T$.
\end{lem}

\subsection{Consistency of estimated quantities}
\label{sec:main:modelparams}
We now use our row-wise eigenspace error bounds to analyze Algorithm~\ref{alg:nmf-mmsb-pure-res}.
We do not analyze the pruning algorithm (Algorithm~\ref{alg:prune}), since that requires distributional assumptions on the row-wise eigenvector errors.
We need the following assumption.
\begin{assumption}\label{as:spa}
Assume
	 $\additionalone$. 
\end{assumption}

\begin{thm}
	\label{thm:theta-B}
	Let $\hat{\bTheta}$ be obtained from Step 6 of Algorithm~\ref{alg:nmf-mmsb-pure-res}. 
  We denote the row-wise eigenspace error from Theorem~\ref{thm:entrywise} as follows:
  \bas{
  	\epsilon=\eigenspacerowwise.
	}
  If Assumptions~\ref{as:ident},~\ref{as:theta_P}, and~\ref{as:spa} hold, there exists a permutation matrix $\bpi$ such that with probability at least $1-O(K/n^2)$, 
	\ba{
		\max_{i\in[n]}{\left\|\be_i^T\bbb{\hat{\bTheta}-\bTheta\bpi}\right\| }
    &= O\bbb{\sqrt{\lambda_1({\bTheta^T \bTheta})}\kappa({\bTheta^T\bTheta})\epsilon},
		\label{eq:theta-err}\\
		\frac{1}{\rho}{\|\hat{\rho}\hat{\bB}-\rho\bpi^T\bB\bpi\|_F}
    &= O\bbb{\frac{ \kappa({\bTheta^T\bTheta})\sqrt{K} n}{\sqrt{\lambda_K({\bTheta^T\bTheta})}}\epsilon}.
    \label{eq:b-err}
	}
  The proof can be found in the Appendix (Sec~\ref{sec:inferred_proof}).
\end{thm}

Under the conditions of Corollary~\ref{cor:entrywise}, our row-wise eigenvector bound leads to $\tilde{O}(1/\sqrt{n\rho})$ rates of convergence of $\hat{\btheta}_i$ to $\btheta_i$. 
To our knowledge, this is the first such result for detecting mixed memberships in networks.

\begin{rem}[Application to Stochastic Blockmodels]
Theorem~\ref{thm:entrywise} can be used to establish strong consistency for Spectral Clustering for Stochastic Blockmodels. 
{Here $\bTheta$ is a binary membership matrix with exactly one ``1'' on each row representing the cluster that node belongs to. 
So, $\bTheta^T\bTheta$ is a diagonal matrix whose diagonal elements (and eigenvalues) represent the sizes of the clusters. }
Consider the standard settings of $K=2$ equal-sized clusters: $\rho\bB=(p_n-q_n)\bI_2+q_n\bone_2\bone_2^T$,  {and $\lambda_{1}(\bTheta^T\bTheta)=\lambda_{K}(\bTheta^T\bTheta)=n/2$}. 
By definition, $\max_{ij}\bB_{ij}=1$, so $\rho=p_n$ and $\lambda^*(\bB)=(p_n-q_n)/p_n$.
Our results imply exact recovery with probability greater than $1-O(K/n^2)$, as long as  $(p_n-q_n)/\sqrt{p_n}=\tilde{\Omega}(1/\sqrt{n})$.
This matches the separation condition in existing literature \citep{mcsherry2001spectral,chen2014improved} up-to logarithmic factors. 
Note that, existing work on sharp threshold for exact recovery~\citep{abbe2015exact} assumes $p_n = a{\log n}/{n}$ and $q_n = b{\log n}/{n}$, where $a,b$ are some constants. 
This implies $\lambda^*(\bB)={(a-b)}/{a}$. 
But we also allow $\lambda^*(\bB)\ll 1$ in the regime that the average expected degree grows as poly-log of $n$.
\end{rem}

\begin{rem}[Comparison to~\citep{jin2017estimating}]
	{ In the latest version of~\citep{jin2017estimating} (updated Sep. 4th, 2019), the authors have added row-wise concentration results for eigenspaces. Their assumptions translate to $\kappa(\bTheta^T\bTheta)=\Theta(1)$. Furthermore, their assumption on the eigenvalues of $\bP$ translates to $\gammaP=O(1)$ in our terminology. Thus, in this regime, {our error bound on estimating $\btheta_i$ (converted to $\ell_1$ norm by multiplying $\sqrt{K}$) is $\sqrt{K}$ worse than theirs up-to logarithmic factors.} A detailed discussion can be found in Sec~\ref{sec:lem_for_jin} of the Appendix. 
	}
\end{rem}

\subsection{Application to Dirichlet Prior}
\label{sec:dirichlet}
Now we consider the case where the $\{\btheta_i\}$ vectors are drawn from a Dirichlet distribution.
We cannot directly use the bound in Theorem~\ref{thm:theta-B} since that bound depends on $\bTheta$.
However, we can probabilistically bound the relevant functions of $\bTheta$.

\begin{lem} \label{lem:theta_property} 
	If $\btheta_i\sim\mathrm{Dirichlet}(\balpha)$ with $\alpha_{\max}=\max_a \alpha_a$,  $\alpha_{\min}=\min_a\alpha_a$ and $\bal:=\alpha_0/\alpha_{\min}$, 
	\bas{
		\uP\left(\lambda_1({\bTheta^T\bTheta})\leq \frac{3n\bbb{\alpha_{\max}+\|\balpha\|^2}}{2\alpha_0(1+\alpha_0)}\right)&\geq 1-K\exp\bbb{-\frac{n}{36\nu^2(1+\alpha_0)^2}}\\
		\uP\left(\lambda_K({\bTheta^T\bTheta})\geq \frac{n}{2\nu(1+\alpha_0)}\right)&\geq 1-K\exp\bbb{-\frac{n}{36\nu^2(1+\alpha_0)^2}}\\
		\uP\left(\kappa({\bTheta^T\bTheta})  \leq {3\frac{\alpha_{\max}+\|\balpha\|^2}{\alpha_{\min}}}\right)&\geq 1-2K\exp\left(-\frac{n}{36\nu^2(1+\alpha_0)^2}\right)
	}
	where $\kappa(.)$ is the condition number of a matrix.
\end{lem}
\begin{assumption}[Parameters of Dirichlet]\label{as:sep}
	Assume for some constant $\xi>1$, we have,
	\bas{
		\bal:= \dfrac{\alpha_0}{\alpha_{\min}}\leq   \dfrac{\min( \sqrt{\frac{n}{27\log n}}, {n\rho})}{2(1+\alpha_0)}, \dfrac{\lambda^*(\bB)}{\nu}\geq \dfrac{8(1+\alpha_0)(\log n)^{\xi}}{\sqrt{n\rho}}.
	}
	\end{assumption}
	One can easily check that under Assumption~\ref{as:sep}, by Lemma~\ref{lem:theta_property}, 
	Assumption~\ref{as:theta_P} is satisfied with probability at least $1-O(Kn^{-3})$.
	When $\alpha_0$ is a constant, the condition on $\lambda^*(\bB)/\nu$ immediately implies $\rho n= \Omega((\log n)^{2\xi})$, since $\lambda^*(\bB)\leq \|\bB\|\leq K\leq \nu$. 
	{Since the expected average degree is $O(n\rho)$, these conditions mean that the average degree must grow faster than poly-log of $n$.
		This is the most common regime where most consistency results on network clustering are shown~\citep{MMSBAnandkumar2014,lei2015consistency,mcsherry2001spectral}. The magnitude of $\alpha_0$  limits the amount of overlap between communities. 
		As noted also by~\citep{MMSBAnandkumar2014}, in many real world applications nodes belong a few communities -- so a constant or slowly growing $\alpha_0$ is a reasonable assumption. For example, the conditions imposed by~\citep{jin2017estimating} on $\bTheta$ can be translated to $\alpha_0=O(1)$ in the context of MMSB models.
		{Note that, our results can handle large $\alpha_0$, but at the cost of a worse error bound.}
		
		Our conditions also allow $K$ to grow with $n$.
    If $\rho=O(1)$, $\alpha_0=O(1)$, and $\lambda^*(\bB)=\Theta(1)$, then $K$ can grow with $\sqrt{n}$, up to poly-log terms (using the fact that $\nu\geq K$).
		Now, consider the common case of a simple MMSB model with $K$ communities: $\rho\bB=(p_n-q_n)\bI_K+q_n\bone_K\bone_K^T$ and $\balpha=\alpha_0\bone_K/K$.
		Since  the largest element of $\bB$ is one by definition, we have $\rho=p_n$.
		This yields $\lambda^*(\bB)=(p_n-q_n)/p_n$.
		We also have $\nu=K$. 
		Hence the second condition can be interpreted as a lower bound on cluster separation: $({p_n-q_n})/{\sqrt{p_n}}=\tilde{\Omega} \bb{{K}/{\sqrt{n}}}$.
		This matches the separation condition in existing literature \citep{MMSBAnandkumar2014}. 
	}
\medskip

We now show error bounds on $\hat{\bTheta}$ and $\hat{\bB}$, when $\btheta_i$ is drawn from a Dirichlet distribution.
For ease of exposition, we focus on the case with similar $\alpha_i$ and $\alpha_0=O(1)$.
This corresponds to roughly-balanced communities with limited overlap.
\begin{cor}\label{cor:theta-B_dirichlet_balance}
	Let $\btheta_i\sim \mathrm{Dirichlet}(\alpha)$ with $\max_a \alpha_a\leq C \min_a \alpha_a$ for some constant $C\geq1$, $\alpha_0=O(1)$.
  If Assumptions~\ref{as:ident} and~\ref{as:sep} 
	hold, and $\lambda^*(\bB)=\tilde{\Omega}(\frac{\min\{K,\kappa(\bB)\}^2K^2}{\sqrt{n\rho}})$, there exists a permutation matrix $\bpi$ such that with probability at least $1-O(K/n^2)$, 
	\ba{
		\max_{i\in[n]}{\left\|\be_i^T\bbb{\hat{\bTheta}-\bTheta\bpi}\right\| }
		&= \ThetaErrorDirichletBalanced,
		\label{eq:theta-err_dirichlet_balance}\\
		\frac{1}{\rho}{\|\hat{\rho}\hat{\bB}-\rho\bpi^T\bB\bpi\|_F}
		&= \ErrorBDirichletBalanced.
		\label{eq:b-err_dirichlet_balance}
	}
\end{cor}

\begin{rem}[Error bound on $\hat{\bTheta}$ as a whole] 
	Note that we can get the Frobenius norm of the error for the whole matrix by directly accumulating the row-wise error bounds. 
	With all other hyperparameters and parameters like $\alpha_0$, $\nu$, $K$ and $\lambda^*(\bB)$ held constant, our Frobenius-norm bound on $\hat{\bTheta}$ is tighter by a factor of $\sqrt{\rho}$ than that in~\citep{mao2017,panov2017consistent}, which allows the analysis to work on networks with average degree $\tilde{\Omega}(\log n)$ rather than $\tilde{\Omega}(\sqrt{n})$.
  
  \citet{MMSBAnandkumar2014} have the same degree regime as ours, but their algorithm assumes prior knowledge of $\alpha_0$.
  Our bound has a worse dependence on $K$, $\alpha_0$ and $\nu$ compared to them. 
  To be concrete, when $\kappa(\bB)=\Theta(1)$ and the clusters are balanced with mild overlap, i.e. $\max_a \alpha_a/\min_a \alpha_a=\Theta(1)$ and $\alpha_{0}=O(1)$, we have an additional $\sqrt{K}$ factor (after converting our Frobenius norm bound to $\ell_1$ norm by multiplying $\sqrt{Kn}$ and theirs by $K$ to get the error of the whole $\hat{\bTheta}$ matrix).
  In the worst case, our bound has an additional $K^2\sqrt{\nu}(1+\alpha_{0})$ factor. 
  We provide more details in the Appendixary (Sec~\ref{sec:tensor}).
\end{rem}

%% file: exp.tex
\section{Experimental results}
\label{sec:exp}
We present both simulation results and real data experiments to compare \OurAlgo with existing algorithms for overlapping network models. We compare with the Stochastic Variational Inference algorithm (SVI)~\citep{gopalan2013efficient}, a geometric algorithm for non-negative matrix factorization for MMSB models with equal Dirichlet parameters (GeoNMF)~\citep{mao2017}, Bayesian SNMF (BSNMF)~\citep{BNMF2011},
the OCCAM algorithm~\citep{zhang2014detecting} for recovering mixed memberships, and the SAAC algorithm~\citep{kaufmann2016spectral}.\interfootnotelinepenalty=10000
\smallskip\noindent
\footnote{We were unable to run the GPU implementation of \citep{MMSBAnandkumar2014} since a required library CULA is no longer open source. We could not get good results  with the CPU implementation with default settings.}
For real data experiments we use two large datasets (with up to 100,000 nodes) from the DBLP corpus.
One of these is assortative ($\bB$ has positive eigenvalues) and one which is {disassortative} ($\bB$ has negative eigenvalues). We show that for the {disassortative} setting,  \OurAlgo significantly outperforms other methods.
\subsection{Simulations}\label{exp:sim}
In this section, we investigate the sensitivity of \OurAlgo and competing algorithms to the Dirichlet parameter $\balpha$, the number of communities $K$, the sparsity control parameter $\rho$, and to the eigenvalues of $\bB$.  
Our simulated graphs have $n=5000$, unless specified otherwise.
We show the relative error ($\|\hat{\bTheta} - \bTheta\|_F / \|\bTheta\|_F$) of different methods, averaged over $10$ random runs in a range of parameter settings.
The largest row-wise relative error has similar trends.  
Further results for varying $\bB$ are presented in the Appendix (Sec~\ref{sec:suppsimu}).

Some algorithms have an underlying model that is slightly different from MMSB.
We handle these as follows.
For OCCAM, we normalize each row of $\bTheta$ by its $\ell_2$ norm, thereby absorbing the $\ell_2$ norm in the degree parameter.
For SAAC, we threshold elements of $\bTheta$ by $1/K$ to get a binary matrix.
For BSNMF, no adjustment is necessary. 
However, note that BSNMF assumes $\bB$ is identity. 

\input{sim_figure}

\smallskip\noindent
{\bf Changing $\balpha$:} In Fig~\ref{fig:expsimone} (A) we use $\balpha = (0.5 - \epsilon_\alpha,0.5,0.5 + \epsilon_\alpha)$ and plot the relative error against $\epsilon_\alpha$. We set  $K=3$, $\rho=0.15$,   $\bB_{ii}=1$, $i\in[K]$, $\bB_{ij}=0.5$ for $i\neq j$.  Recall that for skewed $\balpha$ we get unbalanced cluster sizes. 
\OurAlgo is better than SAAC, SVI, BSNMF and GeoNMF, and also more stable (small variance). For imbalanced clusters (large $\epsilon_\alpha$), \OurAlgo also outperforms OCCAM.

\smallskip\noindent
{\bf Changing $K$:} In Fig~\ref{fig:expsimone} (B) we plot relative error against increasing $K$. We use $\rho=0.1$, $\balpha_i=3/K=1$, $\bB_{ii}=1$, $i\in[K]$, $\bB_{ij}=0.2$ for $i\neq j$. 
We can see that \OurAlgo outperforms SAAC, and is more stable than BSNMF and GeoNMF. When $K$ is very large (>7), everyone performs poorly. When $K$ is small (<5), \OurAlgo works much better than OCCAM and SVI. However, when $K$ is moderately large, OCCAM is slightly better than \OurAlgo. 
This is because in those cases, the eigenspaces do not concentrate very well, and estimating $\hat{\bTheta}$ with cluster centroids (as in OCCAM) seems to reduce the noise. 

\smallskip\noindent
{\bf Changing sparsity}:  We set $\balpha=(0.4,0.4,0.4)$, $\bB_{ii}=1$, $i\in[K]$, $\bB_{ij}=0.05$ for $i\neq j$. We increase $\rho$ from 0.005 to 0.013,  Fig~\ref{fig:expsimone} (C) shows the result. We see that, the error of \OurAlgo is smaller than or similar to that of the best performing algorithm among the others.
In addition, it also has smaller variance.

\smallskip\noindent
{\bf Changing $\lambda_K(\bB)$:} We conclude the simulations with experiments on $\bB$ with negative eigenvalues. We generate $\bB$ so that the smallest eigenvalue $\lambda_K(\bB)$ of $\bB$ is negative. 
We set $\bB=
\begin{bmatrix}
1 & 0.2 & 0.1 \\
0.2 & 0.5 & 0.075\cdot i \\
0.1 & 0.075\cdot i & 0
\end{bmatrix}
$ and vary $i\in [15]$. As $i$ grows, $\lambda_K(\bB)$ becomes more negative. 
We set  $K=3$, $\rho=0.15$, $\balpha=(1/3,1/3,1/3)$.
In the plot of  relative error against $\lambda_K(\bB)$ (Fig~\ref{fig:expsimone}~(D)), 
we see that \OurAlgo is much better than others over the entire parameter range.

%% file: sim_figure.tex
\begin{figure}
	\begin{tabular}{@{\hspace{-2em}}c@{\hspace{-1em}}c@{\hspace{-1em}}c@{\hspace{-1em}}c}
		\includegraphics[width=.3\textwidth]{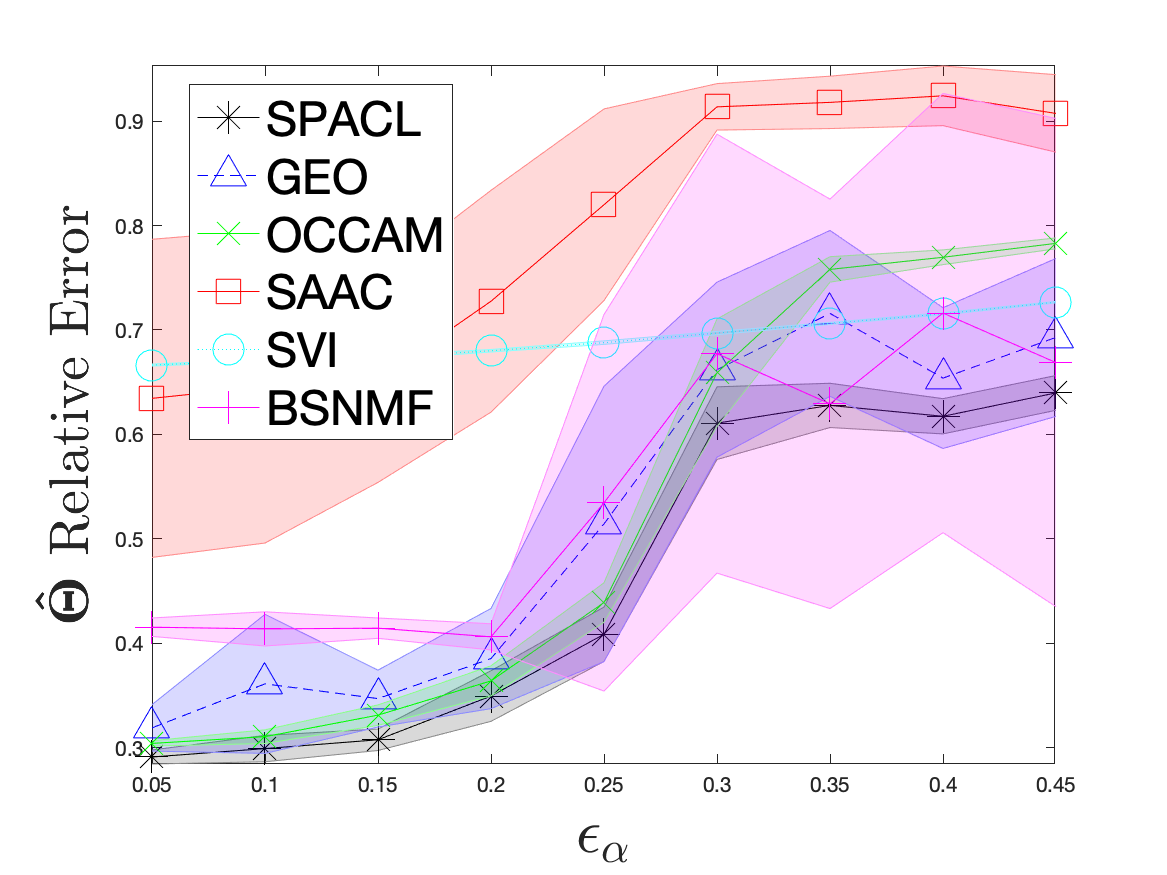}&
		\includegraphics[width=.3\textwidth]{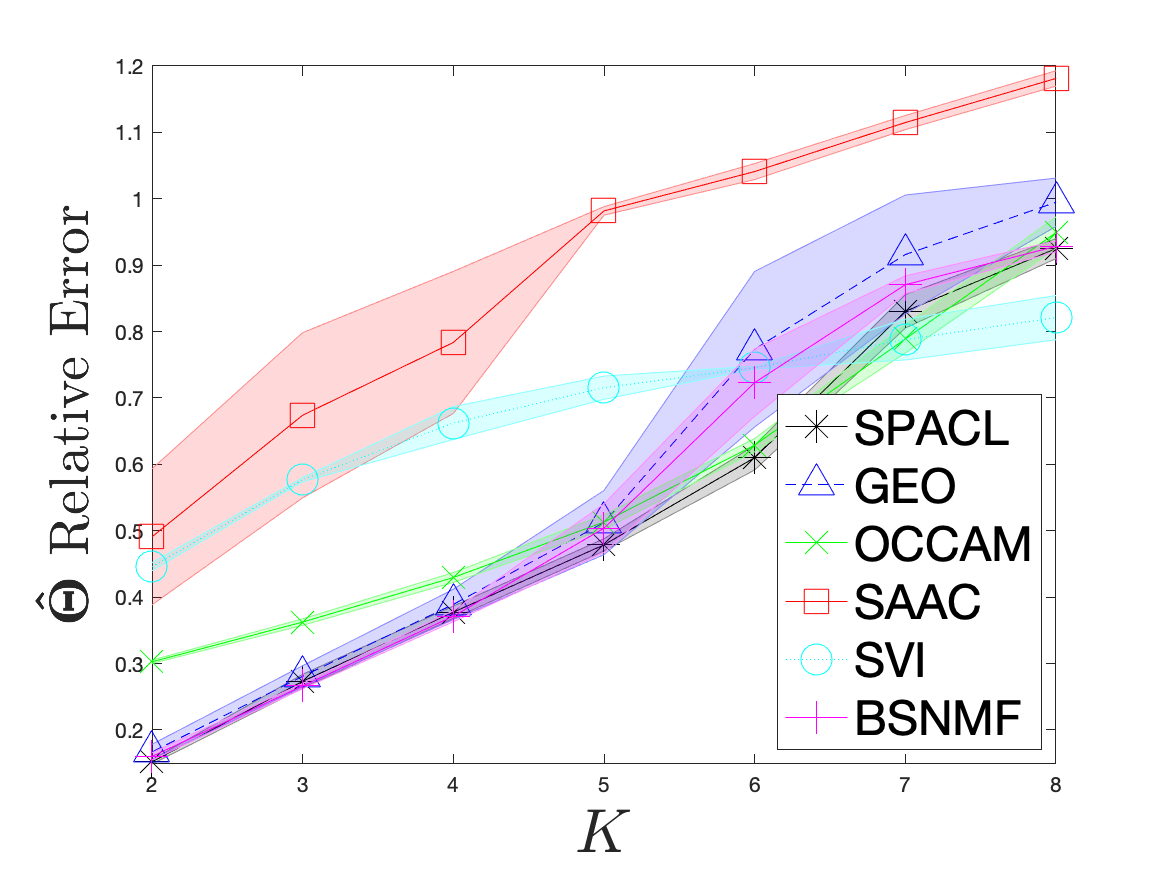}&
		\includegraphics[width=.3\textwidth]{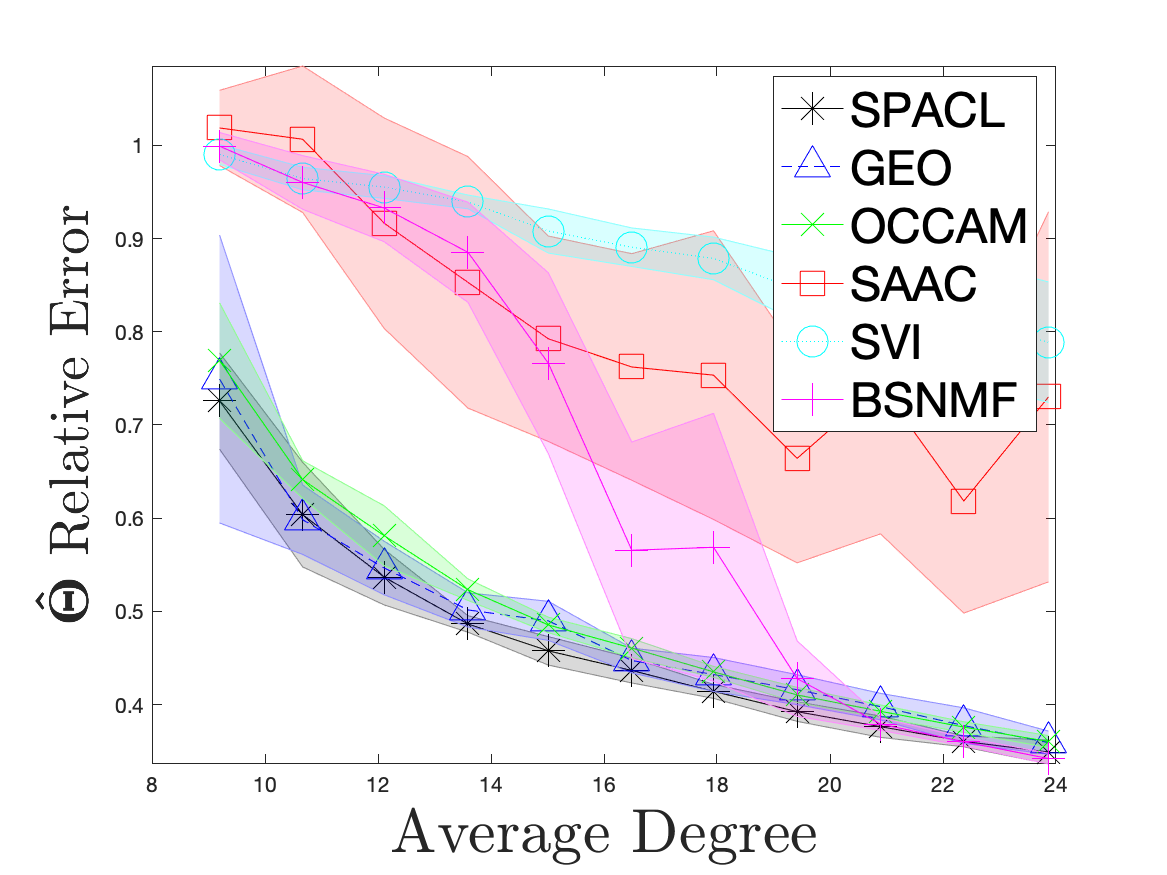}&
		\includegraphics[width=.3\textwidth]{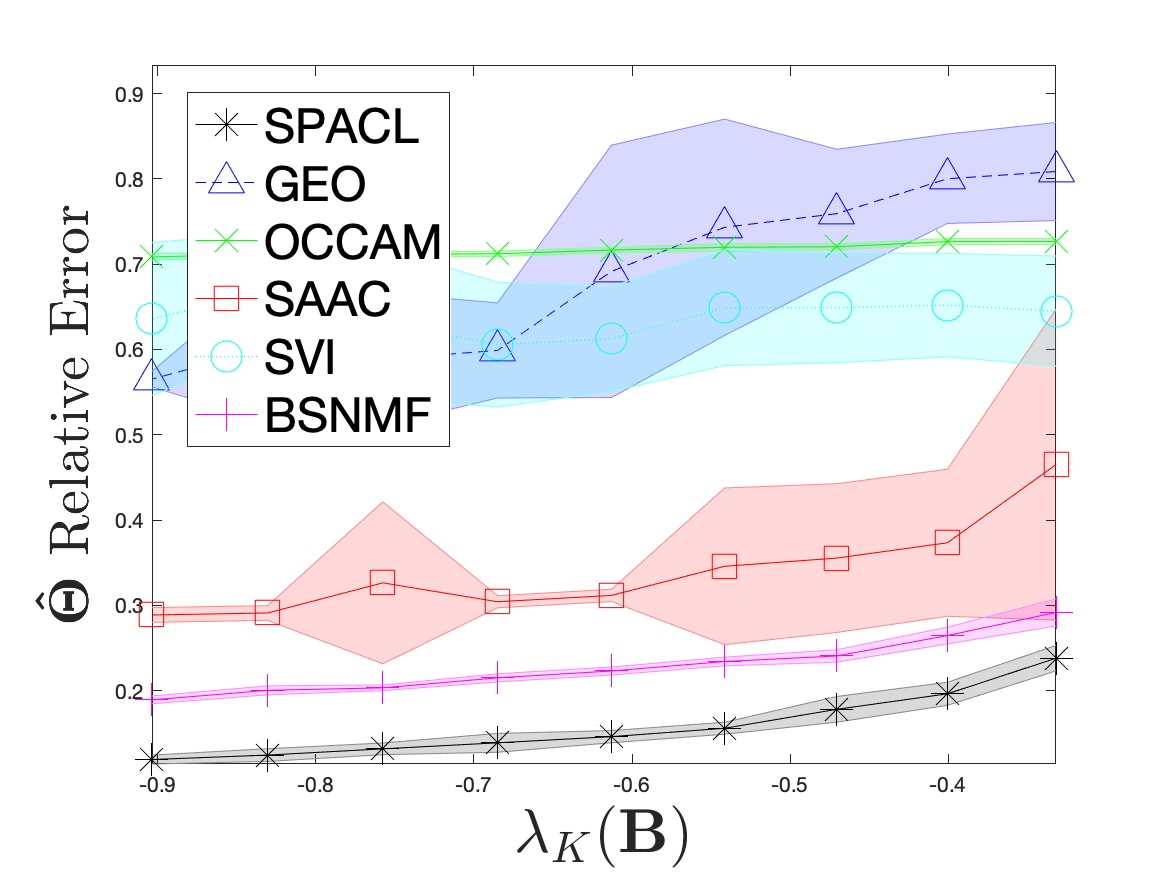}\\
		{\small(A)}&{\small(B)}&{\small(C)}&{\small(D)}
	\end{tabular}
	\caption{\label{fig:expsimone}(A) Error against $\epsilon_\alpha$: $\balpha = (0.5 - \epsilon_\alpha,0.5,0.5 + \epsilon_\alpha)$. (B) Error against increasing $K$. (C) Error against increasing $\rho$ (D) Error against $\lambda_K(\bB)$.}
\end{figure}

%% file: expreal.tex
\input{net_stats}

\subsection{Real Data}\label{exp:real}
\newcommand{\tablefontsize}{\scriptsize}
We use the two types of DBLP networks  obtained from the DBLP dataset\footnote{\url{http://dblp.uni-trier.de/xml/}}, where each ground truth community is a group of conferences on one topic. The author-author networks were used in~\citep{mao2017}; in this paper we also conduct experiments on the bipartite networks by using both papers and authors as nodes. Each community is split into two, the paper community and the author community. The papers are pure nodes since they belong to one conference and hence one community, whereas the authors may belong to more than one community, since they often publish in many conferences. The details of the subfields can be found in~\citep{mao2017}. 
 We have two simple preprocessing steps for the adjacency matrix: 1) delete nodes that do not belong to any community; 2) delete nodes with zero degree.
The statistics of the network are in Table \ref{table:net_stats}, which show that despite being sparse, the networks have large overlaps between communities. The amount of overlap  is measured by the number of overlapping nodes divided by $n$.

\smallskip\noindent
{\bf  Implementation details:} 
 For real world networks, 
 {specially the bipartite networks, when average degree of graphs with 100,000 nodes is smaller than four},  
 some nodes may have extremely small values of $\hat{\bTheta}$ and the corresponding rows may in fact become zero after thresholding. 
 For those we essentially cannot make any prediction. This is why for Step 7 of Algorithm~\ref{alg:nmf-mmsb-pure-res}, we threshold all values smaller than $10^{-12}$ to zero and we do not normalize rows which are all zeros. This does not make any difference for simulations, but for the real world networks, this stabilizes the results. 
\input{real_fig}

\smallskip\noindent
{\bf  Evaluation Metric:}
For author nodes, we construct the corresponding row of $\bTheta$  by normalizing the number of papers an author has in different ground truth communities. 
We present the averaged Spearman rank correlation coefficients (RC) between $\bTheta(:,a)$, $a\in[K]$ and $\hat{\bTheta}(:,\sigma(a))$, where $\sigma$ is a permutation of $[K]$. The formal definition is:
\bas{
	\text{RC}_{\text{avg}}(\hat{\bTheta},\bTheta)&=\frac{1}{K}\max_{\sigma}\sum_{i=1}^{K} \text{RC}(\hat{\bTheta}(:,i),\bTheta(:,\sigma(i))).
}

Note that
$\text{RC}_{\text{avg}}(\hat{\bTheta},\bTheta)\in [-1,1]$, and higher is better. Since SAAC returns binary assignment, we compute its $\text{RC}_\text{avg}$ against the binary ground truth.

\smallskip\noindent
{\bf Performance:} 
We report the $\text{RC}_{\text{avg}}$ score in Fig~\ref{real_data}. The superior performance of  
\OurAlgo on the paper-author networks over the author-author networks can be explained by the fact that the bipartite  network retains information that is lost when the author-author networks are constructed.
Also \OurAlgo outperforms all other methods on bipartite networks, since these are disassortative and the corresponding $\bB$ will have negative eigenvalues. On co-authorship graphs, \OurAlgo performs comparably to GeoNMF, while the other methods are worse. Both \OurAlgo and GeoNMF are much faster than the competing algorithms.

%% file: net_stats.tex
\begin{table}[!t]
	\caption{\label{table:net_stats} Statistics for author-author (Mono) and bipartite paper-author (Bi)  graphs.}
	\centering
	\scalebox{0.76}{
		\begin{tabular}{|c|cc|cc|cc|cc|cc|}
			\hline
			{ Dataset}  & \multicolumn{2}{c|}{ DBLP1}     & \multicolumn{2}{c|}{ DBLP2}     & \multicolumn{2}{c|}{ DBLP3}    & \multicolumn{2}{c|}{ DBLP4}    & \multicolumn{2}{c|}{ DBLP5}   \\
			\hline
			&Mono&Bi&Mono&Bi&Mono&Bi&Mono&Bi&Mono&Bi\\
			\hline\hline
			$\#$ nodes $n$        &  30,566&103,660    &  16,817&50,699    &  13,315&42,288    &  25,481&53,369    &   42,351&81,245       \\
			\hline
			$\#$ communities $K$     &  6 &12        &  3 & 6       &  3 & 6       &  3 & 6     &    4  & 8      \\
			\hline
			Average Degree                          &  8.9 & 3.4&   7.6 & 3.4 & 8.5 & 3.6 &   5.2 & 2.6  &     6.8   & 3.0    \\
			\hline
			Overlap $\%$       &  18.2&6.3 &   14.9&5.6 & 21.1&5.7 & 14.4&6.9 &   18.5 & 9.7   \\
			\hline
		\end{tabular} 
	}
\end{table}

%% file: real_fig.tex
\begin{figure}[!t]
	\centering
	\begin{tabular}{@{\hspace{0em}}c@{\hspace{0em}}c}
		\includegraphics[width=.4\textwidth]{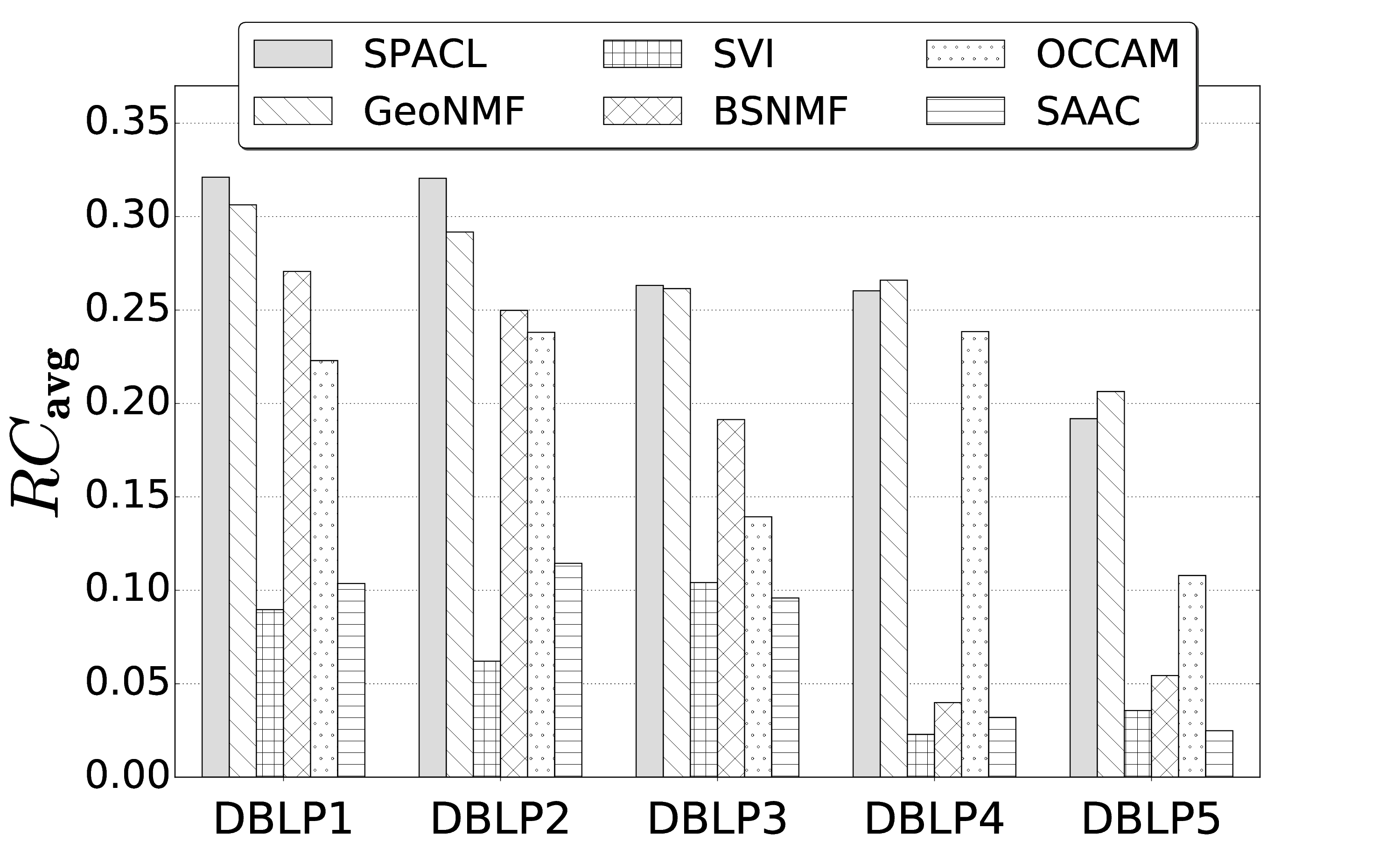}&
		\includegraphics[width=.4\textwidth]{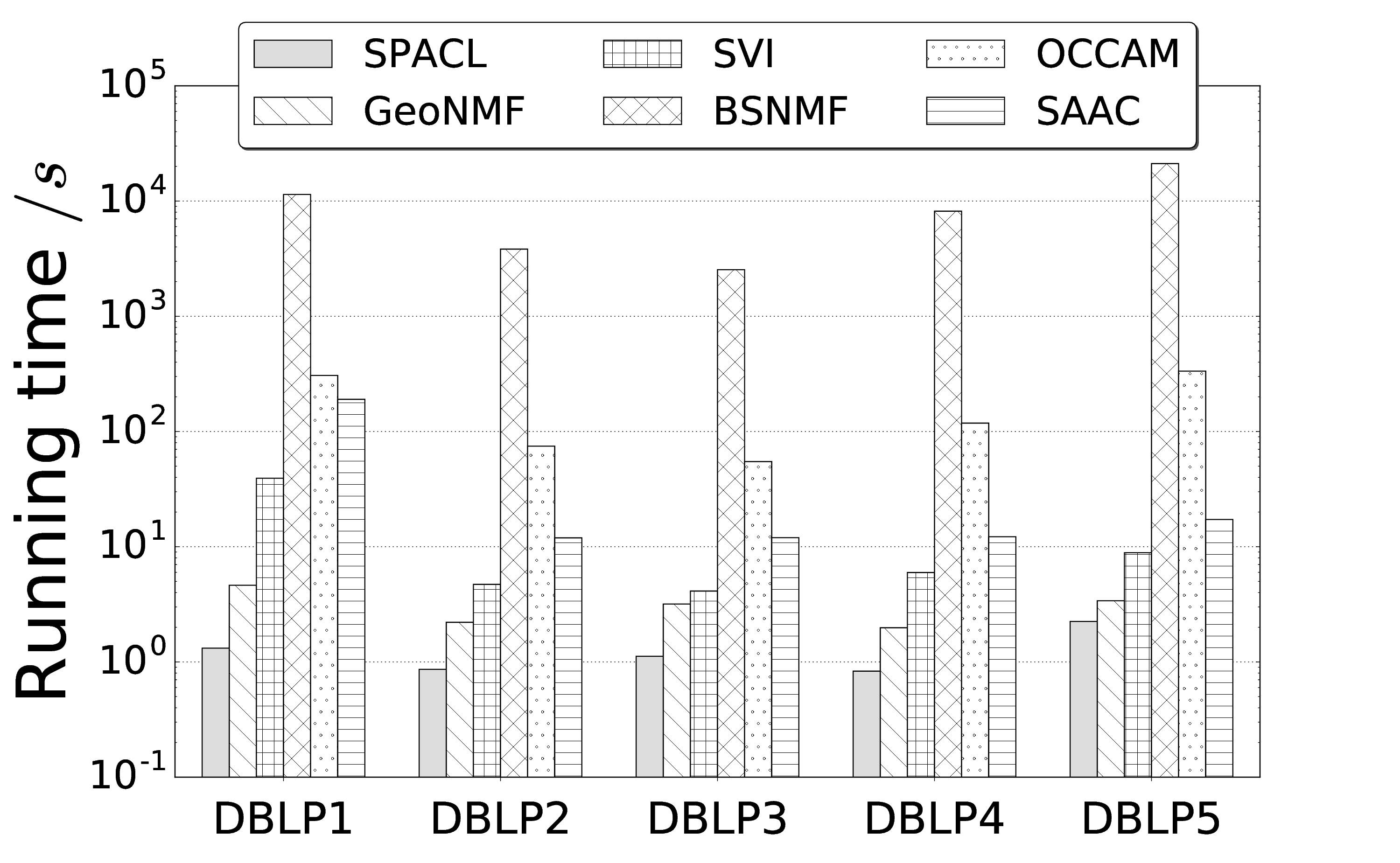}\\
		{\small (A) Rank correlation on DBLP.}&{\small(B) Running time (log scale) on DBLP.}\vspace{0.05in}\\
		\includegraphics[width=.4\textwidth]{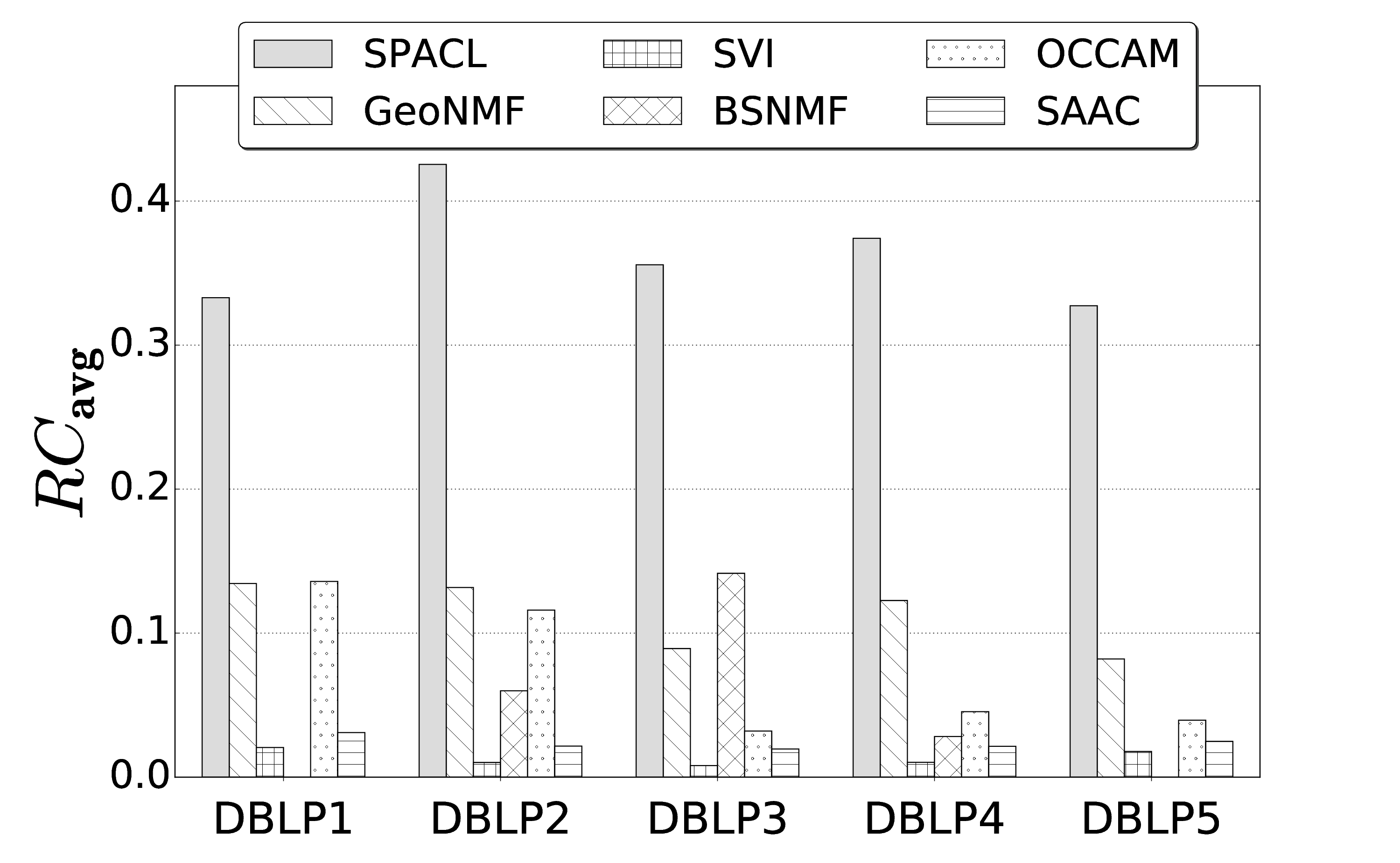}&
		\includegraphics[width=.4\textwidth]{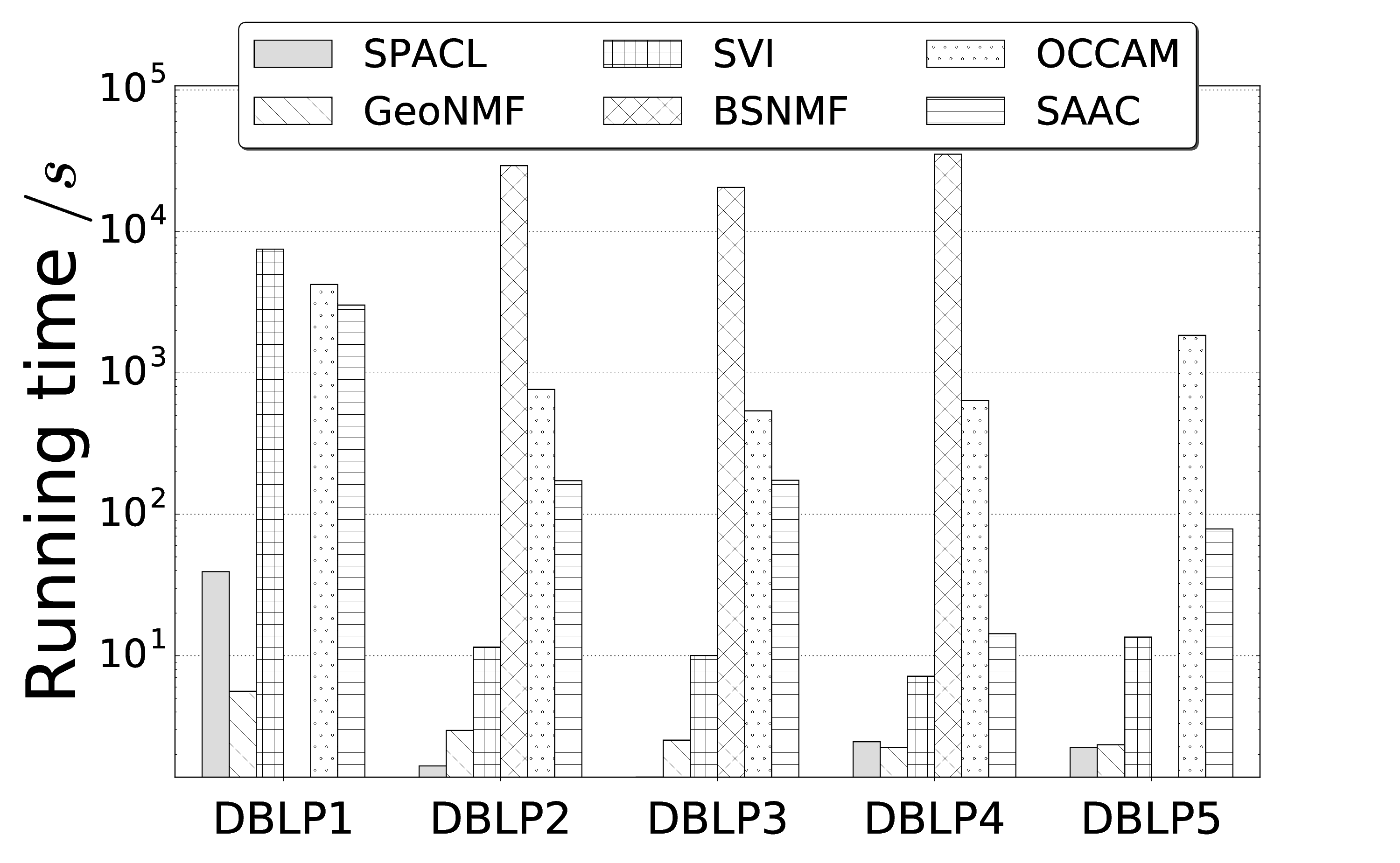}\\
		 {\small \tabular{@{}l@{}}{(C) }\\ {} \endtabular  \tabular{@{}l@{}}{\ Rank correlation on DBLP} \\ \ bipartite graphs.\endtabular}&{\small \tabular{@{}l@{}}{(D) }\\ {} \endtabular \tabular{@{}l@{}}{\ Running time (log scale) on DBLP} \\ \ bipartite graphs.\endtabular}
	\end{tabular}
	\caption{Results on DBLP networks. BSNMF was out of memory for bipartite versions of DBLP1 and DBLP5.}\label{real_data}
\end{figure}

%% file: analysis-detailed.tex
\section{Analysis}
\label{sec:analysis} 
Here we present the main proof idea of Theorem~\ref{thm:entrywise}.
We equate the difference in empirical and population eigenspaces with the Cauchy integral of a matrix resolvent.
To bound the row-wise difference in eigenspaces, we have to specify the contours for the complex integration and then bound a matrix series expansion.
Our contours are carefully chosen by a discretization of the eigenvalues of $\bP$.
This yields an error bound with the proper dependence on $\lambda^*(\bB)$ and $\kappa(\bP)$.
The matrix series expansion is controlled by upper-bounding the first $\log n$ terms and the rest separately, where the partial sum for the first $\log n$ terms is controlled by applying the union bound. 
This is a common technique in perturbation analysis~\citep{erdos2013}. 
A similar strategy is also used in concurrent work~\citep{eldridge2017unperturbed}.
We defer proofs of some of the technical lemmas to the Appendix (Sec~\ref{sec:proof_sec_5}). 

%% file: entry_eigen_1.tex
\subsection{Eigenspace Row-wise Concentration}
Before presenting the analysis of the row-wise error-bounds of empirical eigenvectors, we present a discretization scheme of the population eigenvalues, which later helps in getting a better dependence of the overall row-wise error on the smallest singular value of $\bP$, which can also be thought of as the separation between blocks.

\begin{defn}[A discretization of eigenvalues]
	\label{def:int}
	\input{example_1}

	Let us divide the eigenvalues of $\bP$ into the positive ones ($S^+$) and negative ones ($S^-$). 
	We start with the smallest eigenvalue in $S^+$.
	Denote this by $\lambda^*_+$.
	We set the gap $g_1=\lambda^*_+$ and keep moving through the eigenvalues in $S^+$ in increasing order until we find two consecutive eigenvalues which have gap $g_2>g_1$.
	We repeat this until all eigenvalues in $S^+$ are covered.
	Then every pair of consecutive eigenvalues in the $k^{th}$ interval is within gap $g_k$, and $g_k$ grows with $k$.
	We define $s_k$ and $e_k$ as the starting and ending index of eigenvalues of the $k^{th}$ interval.
	Formally, the $k^{th}$ interval of positive eigenvalues is the set
	\bas{
		S_k^+=\{\lambda_{s_k},\dots \lambda_{e_k}\in S^+: \lambda_{i}-\lambda_{i+1}\le g_k\mbox{ for } s_k\leq i\leq e_k \mbox{ , } \lambda_{e_{k+1}}-\lambda_{s_k} >  g_k \}.
	}
	Let $n_k := |S_k^+|$ be the number of eigenvalues in the $k^{th}$ interval.
	Fig~\ref{fig:Example_def_6.1} shows an example. 
	
	Let the number of intervals with positive eigenvalues be $I^+$.
	Note that $\lambda^*(\bP)\leq \lambda_+^*\leq g_1<g_2\dots<g_{I^{+}}$.
	By a similar splitting process for the negative eigenvalues in $S^-$, we can define $I^-$, $s_{-k}$, $e_{-k}$, and $g_{-k}$.
	{Let $\lambda_{s_0}=0$ and define 
		\ba{
			\label{eq:gammap}
			\gammaP:=\sum_{k=1}^{I^{+}}{\frac{\lambda_{s_k}(\lambda_{s_k}-\lambda_{s_{k-1}})}{g_k^2}}+\sum_{k=1}^{I^{-}}{\frac{\lambda_{s_{-k}}(\lambda_{s_{-k}}-\lambda_{s_{-k+1}})}{g_{-k}^2}}.
		}
		$\gammaP$ measures how tightly the eigenvalues of $\bP$ can be packed together.
	}
	
\end{defn}
The above discretization lets us control the ratio of the largest eigenvalue in each interval and the gap between an interval and the next.
This in turn helps bound $\gammaP$. 
\begin{lem}\label{lem:gammaP_bound}
	In general, $\gammaP\leq 2\min\{K,\kappa(\bP)\}^2$. If the eigenvalues of $\bP$ can be divided into a constant number of bins where eigenvalues in each bin are of the same order, 
	$\gammaP=O(1)$.
\end{lem}

For ease of exposition, we shall henceforth work with just the positive eigenvalues in our proofs, and use $I$ for the number of intervals.
The proofs go through for negative eigenvalues using a nearly identical argument.
We emphasize that the statement of Theorem~\ref{thm:entrywise} considers both positive and negative eigenvalues. 

In order to prove Theorem~\ref{thm:entrywise}, we will first introduce the notion of matrix resolvents and useful identities on resolvents.
\begin{defn}
	A resolvent of a matrix $\bM\in \R^{n\times n}$ is defined as $\res_\bM(z)=(\bM-z\bI)^{-1}$, where 
	$z\not\in \{\lambda_i(\bM)\}_{i=1}^n$. We can also write the resolvent as 
	$\sum_{i=1}^n\frac{\bv_i(\bM)\bv_i(\bM)^T}{\lambda_i(\bM)-z}$, {where $\bv_i(\bM)$ is the $i^{\mathrm{th}}$ eigenvector of $\bM$.}
\end{defn}
Let us define:
\ba{ \bE_z=\diag\bb{\ddd{\frac{\lambda_i}{z({\lambda_i-z})}}_{i=1}^K},\qquad\bM_z=
	\bV \bE_z\bV^T.
	\label{eq:mzdef}
} 
As we see below this matrix is an integral part of the resolvent of the expectation matrix $\bP$.
\ba{\label{eq:resp-decomp}
	\res_\bP(z)&=\sum_{i=1}^n\frac{\bv_i\bv_i^T}{\lambda_i-z}=\sum_{i=1}^K\bv_i\bv_i^T\left(\frac{1}{\lambda_i-z}+\frac{1}{z}\right)-\frac{\bI}{z}=\bM_z-\frac{\bI}{z}
}
We will use a standard technique to compute eigenspaces of matrices (also used in~\citep{oliveira2009concentration} Lemma A.2). Consider an interval $(a,b)$ such that no eigenvalue of a symmetric matrix $\bM$ equals $a$ or $b$. Now consider a {rectangular} contour $\C$ in the complex plane which passes through $a+\gamma\sqrt{-1},a-\gamma\sqrt{-1},b-\gamma\sqrt{-1},b+\gamma\sqrt{-1}$ in counter clockwise direction, where $\gamma>0$.
From the Cauchy integration formula, we know that 
\ba{\label{eq:cauchy}
	\frac{1}{2\pi \sqrt{-1}}\oint_{\C}\res_\bM(z)dz=-\sum_{i:\lambda_i(\bM)\in (a,b)}\bv_i(\bM)\bv_i(\bM)^T	 
}
\begin{defn}\label{def:contour}
	We consider a sequence of non-overlapping contours  $\C_k,k\in [I]$ ($I\leq K$) created using $a_k,b_k,\gamma_k$, where $\|\bA-\bP\|< a_k<b_k$, and none of the eigenvalues of $\bA$ or $\bP$ equal $a_k,b_k$ for $k\in [I]$. 
\end{defn}
Let $\bV_k$ denote the $n\times n_k$ matrix with the eigenvectors of $\bP$ corresponding to eigenvalues in $(a_k,b_k)$. Similarly let $\hV_k$ denote the eigenvectors of $\bA$ corresponding to eigenvalues in $(a_k,b_k)$. 
Hence, using the Cauchy integration formula~\eqref{eq:cauchy}, we have:
\ba{
	\label{eq:res-eigen}
	\bV_k\bV_k^T-\vh_k\vh_k^T&=\frac{1}{2\pi\sqrt{-1}}\oint_{\C_k}\left(\res_\bA(z)-\res_\bP(z)\right)dz
}

Furthermore, it is not hard to check that, $\forall x\in [n]$,
\ba{
	\label{eq:res-eigen-row}
	\be_{x}^T\left(\bV_k\bV_k^T-\vh_k\vh_k^T\right)&=\frac{1}{2\pi\sqrt{-1}}\oint_{\C_k}\be_x^T\left(\res_\bA(z)-\res_\bP(z)\right)dz
}
We bound the Frobenius norm of the above quantities using Lemma~\ref{lem:buildup} below.
\begin{lem}\label{lem:buildup}
	{For contours in Definition~\ref{def:contour}}, we have:
	\ba{
		&\left\|\be_x^T\sum_{k=1}^I(\bV_k\bV_k^T-\vh_k\vh_k^T)\right\|\leq 
		\sum_{k=1}^I \frac{b_k-a_k+2\gamma_k}{\pi}\max_{z\in \C_k}(P_{1}(z)+P_{2}(z)),	\label{eq:p1p2z}\\
		\mbox{where}\quad
		&P_{1}(z)=|z|\|\res_{\bA}(z)\|\|\bA-\bP\|\|\bE_z\|\|\be_x^T\res_{\bA-\bP}(z)\bV\|,\nonumber\\ 
		&P_{2}(z)= \|\be_x^T\res_{\bA-\bP}(z)(\bA-\bP)\bV\|_F \|\bE_z\|.\nonumber
	}
\end{lem}

Now we need to:
\begin{enumerate}
	\item Define contours and events so that the LHS of Eq~\eqref{eq:p1p2z}  covers the whole eigenspace.
	\item Bound  $P_1(z)$ and $P_2(z)$ over each contour, under these events. This requires bounds on $\|\be_x^T\res_{\bA-\bP}(z)\bv_i\|$, and $\|\be_x^T\res_{\bA-\bP}(z)(\bA-\bP)\bv_i\|$, where $\bv_i$ denotes the $i^{th}$ column of $\bV$. We also need $\|\bE_z\|$, $\|\res_{\bA}(z)\|$, $\|\res_{\bA-\bP}(z)\|$, etc. This requires us to bound $|\be_i^T\bH^t\bv_i|$ for $t\leq \log n$, where $\bH:=(\bA-\bP)/\sqrt{n\rho}$.
	For $t=1$, we prove the following lemma, which uses the fact that $\bV$ is delocalized with high probability (see Lemma~\ref{lem:v-row-norm}).
\end{enumerate}

\begin{lem}
	\label{lem:azuma-better}
	Let $\bv_k$ denote  the $k^{th}$ population eigenvector of $\bP$.  If Assumption~\ref{as:theta_P} is satisfied, for a fixed $i$, 
	$\uP\bbb{\exists k\in [K],|\be_i^T\bH\bv_k|\geq 4\log n\|\bv_k\|_\infty}=O(K/n^3).$
\end{lem}

For $1<t\leq \log n$, we adapt a crucial result from~\citep{erdos2013}.
\begin{lem}\label{cor:erdos}
	Let $\bH:=(\bA-\bP)/\sqrt{n\rho}$. As long as 
	Assumption~\ref{as:theta_P} is satisfied for some constant $\xi$, for any fixed vector $\bv$,
	{for a fixed $i$ and for $1<t\leq \log n$,}
	\bas{
		\uP\bb{|\be_i^T \bH^t \bv|\leq (\log n)^{t\xi}\|\bv\|_\infty}\geq 1-\exp(-(\log n)^\xi/3).
	}
\end{lem}
Proofs of Lemmas~\ref{lem:buildup},~\ref{lem:azuma-better}, and~\ref{cor:erdos} are in the Appendix (Secs~\ref{sec:pf_buildup},~\ref{sec:azuma-better}, and~\ref{sec:supp-proof-erdos} respectively). 

We will now define some events, which will be used extensively to show that the contours cover all population and empirical eigenvalues, and to bound $P_1(z)$ and $P_2(z)$ in Eq~\eqref{eq:p1p2z}. We will use $\E$ to denote an event and $\bar{\E}$ to denote its compliment.
Let $\bv_k$ be the $k^{th}$ population eigenvector. Under Assumption~\ref{as:theta_P}, for $t\leq \log n$,

\hspace*{-1.5em}\vbox{
	\ba{\label{eq:lambdastar}
		\evz&:=\{\|\bA-\bP\|\leq C\sqrt{n\rho}\}  & &\uP(\bar{\evz})
		\stackrel{(i)}{\leq} n^{-3}\nonumber\\
		\E_1&:=\left\{\left|\be_i^T \bH \bv_k\right|\leq 4\log n\|\bv_k\|_\infty, \forall k\in[K]\right\}  & &\uP(\bar{\E_1})
		\stackrel{(ii)}{\leq} O\bb{K/n^3}\\
		\evt&:=\left\{\left|\be_i^T \bH^t \bv_k\right|\leq (\log n)^{t\xi}\|\bv_k\|_\infty, \forall k\in[K]\right\}  & &\uP(\bar{\evt})
		\stackrel{(iii)}{\leq} K\exp(-(\log n)^\xi/3), 1<t\leq\log n\nonumber
	}
}

For any community membership matrix $\bTheta$, $\uP(\bar{\evz}|\bTheta)$ can be bounded  directly using Theorem 5.2 of \citep{lei2015consistency}, since Assumption~\ref{as:theta_P} requires that $n\rho=\Omega(\log n)$. Hence step $(i)$ follows.  
Steps $(ii)$ and $(iii)$ follow from Lemmas~\ref{lem:azuma-better} and~\ref{cor:erdos} respectively. 
To denote order notation conditioned on event $\evot$, we will use,
$X\evote O(.)$ to denote, $\uP(X=O(.))=\uP(\evot)$.

{\em Picking the contours $\C_k$: }
Consider the discretization in Definition~\ref{def:int}. For the $k^{th}$ interval, use $\gamma_k=g_k/4$,  $a_k=\max(\lambda_{e_k}-g_k/2,(1+c)\|\bA-\bP\|)$, for some $c>0$ and  $b_k=\lambda_{s_k}+g_k/2$. If $b_k\leq a_k$, we ignore the contour. If either $a_k$ or $b_k$ equal an eigenvalue of $\bA$ or $\bP$, for any $\epsilon>0$, they can be perturbed by at most $\epsilon$ to guarantee that they do not coincide with eigenvalues of $\bA$ or $\bP$. This is possible because for a given $n$, the set $\{\bA\mid\bA\in\{0,1\}^{n\times n}\}$ is finite.

Now we bound $\|\res_{\bA}(z)\|$, $\|\res_{\bP}(z)\|$, $\|\bE_z\|$ and $\|\res_{\bA-\bP}(z)\|$. Since the gap between the smallest eigenvalue (in magnitude) of the $k^{th}$ interval and the largest eigenvalue in the {$(k-1)^{th}$} interval is $g_k$, and by construction (Definition~\ref{def:int}) $\lambda^*(\bP)\leq g_1< g_2< \dots$, and $\lambda_{e_k}\geq g_k$,
we note that for each contour $\C_k$, {conditioned on $\evz$},  $|z|$ can be upper and lower bounded as follows. 
\ba{
	|z|&\leq \sqrt{b_k^2+\gamma_k^2}\leq b_k+\gamma_k=\lambda_{s_k}+3g_k/4\label{eq:zupper}\\	 
	|z|&\geq \max((1+c)\|\bA-\bP\|,|\lambda_{e_k}-g_k/2|)
	\geq |\lambda_{e_k}-g_k/2|\geq g_k/2
	\label{eq:zmod}\\
	|z-\lambda_i|&\geq g_k/2,  \qquad |z-\hat{\lambda}_i|\stackrel{(i)}{\geq} g_k/2-O(\sqrt{n\rho})\label{eq:mzop}\\
	\|\bM_z\|&=\|\bE_z\|\leq \max_i \left|\frac{1}{\lambda_i-z}+\frac{1}{z}\right|= O\left(\frac{1}{g_k}\right)\label{eq:mz-row}
}
For all $i\in[n]$ and for all $z\in \C_k$, Eq~\eqref{eq:mz-row} follows from Eqs~\eqref{eq:mzdef},~\eqref{eq:zmod} and~\eqref{eq:mzop} and Assumption~\ref{as:theta_P}.

Step $(i)$ in Eq~\eqref{eq:mzop}, uses  $|\hat{\lambda}_i-\lambda_i|\evze O(\sqrt{n\rho})$ via Weyl's inequality. Finally using Eqs~\eqref{eq:zmod},~\eqref{eq:mzop} and~\eqref{eq:mz-row} we also have for all $z\in \C_k$, conditioned on $\evz$,
\ba{\label{eq:ganorm}
	\|\res_{\bP}(z)\| = O\bbb{\frac{1}{g_k}} \quad
	\|\res_{\bA}(z)\|\leq \left\|\sum_i \frac{\hat{\bv}_i\hat{\bv}_i^T}{\lh_i-z}\right\|=
	O\bbb{\frac{1}{g_k-O(\sqrt{n\rho})}}
}
Now conditioned on $\evz$ , Eq~\eqref{eq:ganorm} gives:
\ba{\label{eq:gaminuspnorm}
	\|\res_\bA(z)-\res_\bP(z)\|\leq\|\res_\bP(z)\|\|\bP-\bA\|\|\res_\bA(z)\|	
	= O\bbb{\frac{\sqrt{n\rho}}{g_k}}O\bbb{\frac{1}{g_k-O(\sqrt{n\rho})}}
}
Now we will bound the RHS of Eq~\eqref{eq:p1p2z} in Lemma~\ref{lem:buildup}.

\begin{lem}\label{lem:res-proj}
	Let $\bv_i$ denote the $i^{th}$ column of $\bV$. Let Assumption~\ref{as:theta_P} be satisfied for some constant $\xi$.
	Consider the events defined in Eq~\eqref{eq:lambdastar}. Conditioned on $\bigcap_{t=1}^{\log n} \E_t\cap \evz$,  
	\bas{
		|\be_x^T\res_{\bA-\bP}(z)\bv_i|&=\frac{O\bbb{ \|\bv_i\|_\infty+n^{-2\xi}}}{\lambda^*(\bP)}.\\
		|\be_x^T\res_{\bA-\bP}(z)(\bA-\bP)\bv_i|&=O\bbb{\frac{\sqrt{n\rho}\left((\log n)^{\xi}\|\bv_i\|_\infty+n^{-2\xi}\right)}{\lambda^*(\bP)}}.
	}
\end{lem}
\begin{proof} 
	First note that by construction $\forall z\in \C_k, \forall k$, $|z|\geq a_k >\|\bA-\bP\|$, we have the following series expansion for $\res_{\bA-\bP}(z)$,
	\ba{\label{eq:res-series}
		\res_{\bA-\bP}(z)=-\frac{1}{z}\sum_{t\geq 0}\bb{\frac{\bA-\bP}{z}}^t.
	} 
	For $\bH$ defined in Lemma~\ref{cor:erdos}, for $1\leq t\leq \log n$, conditioned on $\E_t$, $t\geq 1$,
	\ba{\label{eq:res_term}
		\left|\frac{\be_x^T(\bA-\bP)^t\bv_i}{z^t}\right|=	
		\left|\be_x^T\bH^t\bv_i\frac{(\sqrt{n\rho})^t}{z^t}\right|\leq
		\begin{cases}
			\left(\frac{\sqrt{n\rho}(\log n)^{\xi}}{|z|}\right)^t\|\bv_i\|_\infty&\mbox{$t\leq \log n$}\\
			\left(\frac{\|\bA-\bP\|}{|z|}\right)^t&\mbox{$t> \log n$}
		\end{cases},
	}
	where we use Lemmas~\ref{lem:azuma-better} and \ref{cor:erdos}. It is easy to verify that the above holds for $t=0$.
	As Assumption~\ref{as:theta_P} gives:
	\ba{\label{eq:geometric}
		&\lambda^*(\bP)\stackrel{\evot}{\geq} 4\sqrt{n\rho}(\log n)^{\xi}\quad\Rightarrow\quad
		\max_{k,z\in \C_k}\frac{\sqrt{n\rho}(\log n)^{\xi}}{|z|}	\stackrel{\evot}{\leq} \frac{1}{2}
	}  
	Conditioned on $\bigcap_{t=1}^{\log n} \E_t\cap \evz$, Eqs~\eqref{eq:res-series}~and~\eqref{eq:res_term} give: 
	\bas{
		\max_{k,z\in \C_k}|\be_x^T\res_{\bA-\bP}(z)\bv_i|&\leq \max_{k,z\in \C_k}\frac{1}{|z|}\left|\sum_{t=0}^\infty \frac{\be_x^T(\bA-\bP)^t}{z^t}\bv_i\right|\\
		&\leq \max_{k,z\in \C_k}\frac{1}{|z|}\sum_{t=0}^{\log n} \left|\frac{\be_x^T(\bA-\bP)^t\bv_i}{z^t}\right|+\max_{k,z\in \C_k}\frac{1}{|z|}\sum_{t>\log n}\left|\frac{\be_x^T(\bA-\bP)^t\bv_i}{z^t}\right|\\
		\mbox{(Eqs~\eqref{eq:res_term} and~\eqref{eq:geometric})}\quad &\leq \max_{k,z\in \C_k}\frac{\|\bv_i\|_\infty}{|z|-\sqrt{n\rho}(\log n)^{\xi}}+\max_{k,z\in \C_k}\frac{(\|\bA-\bP\|/|z|)^{\log n+1}}{|z|-\|\bA-\bP\|}\\
		&=O\bbb{ \frac{\|\bv_i\|_\infty}{\lambda^*(\bP)/2-\sqrt{n\rho}(\log n)^{\xi}}+\frac{({2C}\sqrt{n\rho}/\lambda^*(\bP))^{\log n+1}}{\lambda^*(\bP)/2-C\sqrt{n\rho}}}\\
		\mbox{(Eq~\eqref{eq:geometric})}\quad 
		&=\frac{O\bbb{ \|\bv_i\|_\infty+n^{-2\xi}}}{\lambda^*(\bP)}
	}
	We also have, for large enough $n$, $\left({{ 2C}\sqrt{n\rho}}/{\lambda^*(\bP)}\right)^{\log n+1}\leq \left({C}/({2(\log n)^{\xi}})\right)^{\log n+1}\leq \exp(O(\log n)-\xi(\log n+1)\log\log n) =O\left(1/(n^{2\xi})\right)$. 
	Furthermore, using the same argument as before,
	\bas{
		\max_{k,z\in \C_k}|\be_x^T\res_{\bA-\bP}(z)(\bA-\bP)\bv_i|&=\max_{k,z\in \C_k}\left|\sum_{t=1}^\infty \frac{\be_x^T(\bA-\bP)^t}{z^t}\bv_i\right|
		=O\bbb{\frac{\sqrt{n\rho}\left((\log n)^{\xi}\|\bv_i\|_\infty+n^{-2\xi}\right)}{\lambda^*(\bP)}}
	}\par 
	\qedhere
\end{proof}
Now we are ready to finish the proof of Theorem~\ref{thm:entrywise}.
\begin{proof}[Proof of Theorem~\ref{thm:entrywise}]
	Our goal is to bound the row norm of $\bV\bV^T-\vh\vh^T$ using Lemma~\ref{lem:buildup}. The first step is to show:
	\ba{\label{eq:wholespace}
		\|\be_x^T(\bV\bV^T-\vh\vh^T)\|
		\stackrel{\E'}{{=\joinrel=}}
		\left\|\sum_{k=1}^I\be_x^T(\bV_k\bV_k^T-\vh_k\vh_k^T)\right\|.
	}
	Recall that $a_k=\max(\lambda_{e_k}-g_k/2, (1+c)\|\bA-\bP\|)$. Conditioned on $\evz$, and using Assumption~\ref{as:theta_P} and Lemma~\ref{lem:P_eigen}, we have $\lambda_{e_k}-g_k/2\geq \lambda^*(\bP)/2= \omega(\|\bA-\bP\|)$. This gives $a_k=\lambda_{e_k}-g_k/2$.  Hence the intervals are mutually exclusive and cover all the population eigenvalues, proving Eq~\eqref{eq:wholespace}.
	By triangle inequality, conditioned on $\bigcap_{t=1}^{\log n} \evt\cap \evz$, 
	from Lemma~\ref{lem:buildup},
	we have:
	\ba{\label{eq:finial_bounds}
		&\|\be_x^T(\bV\bV^T-\vh\vh^T)\|\leq \sum_{k=1}^I\|\be_x^T(\bV_k\bV_k^T-\vh_k\vh_k^T)\| \nonumber\\
		&\stackrel{(i)}{=}\sum_{k=1}^IO\bbb{\frac{\lambda_{s_k}-\lambda_{e_k}+2g_k}{g_k}}\max_{z\in \C_k}\left(O\bbb{\frac{\sqrt{n\rho}(b_k+\gamma_k)}{g_k}}\|\be_x^T\res_{\bA-\bP}(z)\bV\| 
		+\|\be_x^T\res_{\bA-\bP}(z)(\bA-\bP)\bV\|\right)\nonumber\\
		&\stackrel{(ii)}{=}O(\gammaP)\max_{k, z\in \C_k}\left(O\bbb{\sqrt{n\rho}}\|\be_x^T\res_{\bA-\bP}(z)\bV\| 
		+\|\be_x^T\res_{\bA-\bP}(z)(\bA-\bP)\bV\|\right)\nonumber\\
		&\stackrel{(iii)}{=}
		O\bbb{\frac{\gammaP\sqrt{Kn\rho}}{\lambda^*(\bP)}}\bbb{(1+(\log n)^{\xi})\max_i\|\bv_i\|_\infty+2n^{-2\xi}} \nonumber\\		
		&\stackrel{(iv)}{=}
		O\bbb{\frac{\gammaP\sqrt{Kn\rho}}{\rho \lambda^*(\bB)  \lambda_K( \bTheta^T \bTheta)}}
		\bbb{\frac{1+(\log n)^{\xi}}{\sqrt{\lambda_K(\bTheta
					^T\bTheta)}}+2n^{-2\xi}}\nonumber\\
		&\stackrel{(v)}{=}
		\eigenspacerowwise 
	}
	Step $(i)$ uses Eq~\eqref{eq:zupper}. Step $(ii)$ uses the fact that $\lambda_{e_k}-\lambda_{s_{k-1}}=g_k$ and $\lambda_{s_k}/g_k=\Omega(1)$. 
	Step $(iii)$ follows from Lemma~\ref{lem:res-proj}.
	Step $(iv)$ 
	uses $\lambda^*(\bP)\geq\rho \lambda^*(\bB)  \lambda_K( \bTheta^T \bTheta)$ (Lemma~\ref{lem:P_eigen} in the Appendix) and Lemma~\ref{lem:v-row-norm}. 
	Step $(v)$ uses $\lambda_{1}(\bTheta^T\bTheta)\leq n$ (Lemma~\ref{lem:lambda_1_Theta} in the Appendix) and   $1/\sqrt{\lambda_K(\bTheta^T\bTheta)}\geq 1/\sqrt{\lambda_1(\bTheta^T\bTheta)}=\Omega(1/\sqrt{n})=\Omega(n^{-2\xi})$.  
	To bound the failure probability, for some constant $\xi> 1$ and large enough $n$,  Eq~\eqref{eq:lambdastar} gives:
	\bas{
		\uP(\bigcap_{t=1}^{\log n}\evt \cap \evz)&\geq 1-\uP(\bar{\evz})-\sum_{t=1}^{\log n}\uP(\bar{\evt})
		\geq 1-O(Kn^{-3}).}
	Now the theorem statement follows by using a union bound.
	\par 
\vspace{-0.99\baselineskip}
	\qedhere
\end{proof}

%% file: example_1.tex
\begin{figure}[!htbp]
    \centering
    {
    \begin{tikzpicture}
        \begin{axis}[
        xmin=-6, xmax=380, 
        ymin=0,
        height = 3.9cm,
        width=\textwidth,
        ticks=none,
        axis x line*=middle,
        hide y axis,        
        xticklabels={,,}
        ]
            \addplot[mark=*] coordinates {(0,0)} node[pin=90:{\footnotesize $0$}]{} ;
            
            \addplot[mark=square*] coordinates {(40,0)} node[pin=90:{\footnotesize $\lambda_{10}$}]{} ;
            \addplot[mark=square*] coordinates {(40,0)} node[pin=270:{\footnotesize $\lambda_{e_1}$}]{} ;
            
            \addplot[mark=*] coordinates {(55,0)} node[pin=90:{\footnotesize $\lambda_{9}$}]{} ;
            
            \addplot[mark=triangle*] coordinates {(90,0)} node[pin=90:{\footnotesize $\lambda_{8}$}]{} ;
            \addplot[mark=triangle*] coordinates {(90,0)} node[pin=270:{\footnotesize $\lambda_{s_1}$}]{} ;
            
            \addplot[mark=square*] coordinates {(145,0)} node[pin=90:{\footnotesize $\lambda_{7}$}]{} ;
            \addplot[mark=square*] coordinates {(145,0)} node[pin=270:{\footnotesize $\lambda_{e_2}$}]{} ;
            
            \addplot[mark=*] coordinates {(180,0)} node[pin=90:{\footnotesize $\lambda_{6}$}]{} ;
            \addplot[mark=*] coordinates {(220,0)} node[pin=90:{\footnotesize $\lambda_{5}$}]{} ;
            \addplot[mark=triangle*] coordinates {(240,0)} node[pin=90:{\footnotesize $\lambda_{4}$}]{} ;
            \addplot[mark=triangle*] coordinates {(240,0)} node[pin=270:{\footnotesize $\lambda_{s_2}$}]{} ;
            
            \addplot[mark=square*] coordinates {(320,0)} node[pin=90:{\footnotesize $\lambda_{3}$}]{} ;
            \addplot[mark=square*] coordinates {(320,0)} node[pin=270:{\footnotesize $\lambda_{e_3}$}]{} ;
            
            \addplot[mark=*] coordinates {(355,0)} node[pin=90:{\footnotesize $\lambda_{2}$}]{} ;
            \addplot[mark=triangle*] coordinates {(370,0)} node[pin=90:{\footnotesize $\lambda_{1}$}]{} ;
            \addplot[mark=triangle*] coordinates {(370,0)} node[pin=270:{\footnotesize $\lambda_{s_3}$}]{} ;
    
            \draw [thick,decoration={brace,mirror},decorate] (7,-3) -- (46,-3) node[midway,below] {\footnotesize $g_1$};
            \draw [thick,decoration={brace,mirror},decorate] (46,-3) -- (96,-3) node[midway,below] {\footnotesize $S_1^+$};
            
            \draw [thick,decoration={brace,mirror},decorate] (96,-3) -- (151,-3) node[midway,below] {\footnotesize $g_2$};
            \draw [thick,decoration={brace,mirror},decorate] (151,-3) -- (246,-3) node[midway,below] {\footnotesize $S_2^+$};
            
            \draw [thick,decoration={brace,mirror},decorate] (246,-3) -- (326,-3) node[midway,below] {\footnotesize $g_3$};
            \draw [thick,decoration={brace,mirror},decorate] (326,-3) -- (376,-3) node[midway,below] {\footnotesize $S_3^+$};
        \end{axis}
    \end{tikzpicture}
}
    \caption{An illustration of Definition~\ref{def:int}.} 
    \label{fig:Example_def_6.1}
\end{figure}

%% file: conc.tex
\section{Conclusion}
\label{sec:conc}
In this paper, we propose a fast and provably consistent algorithm called \OurAlgo for inferring community memberships of nodes in a network generated by a Mixed Membership Stochastic Blockmodel (MMSB).
Our proof has several new aspects, including 
a sharp row-wise eigenvector bound using complex contour integration, a new grouping of the eigenvalues to yield better dependence on the smallest singular value of $\bB$. Our eigenvector deviation results can be easily generalized to low rank population matrices arising from models other than MMSB. It also helps us  establish the convergence of inferred soft community memberships of each node to its population counterpart, which is to our knowledge, the first such result for overlapping network models.
In contrast to prior work, we only assume that each community has at least one pure node, and we prove both necessary and sufficient conditions for identifiability under MMSB.
We demonstrate the empirical performance of \OurAlgo on simulated and real-world networks of up-to 100,000 nodes. Our experiments show that \OurAlgo has smaller error as well as lower variability than other competing methods. In terms of scalability, we can obtain overlapping cluster memberships of large 100,000 node networks in tens of seconds.

%% file: identifiable_1.tex
\begin{proof}[Proof of Theorem~\ref{mmsb_iden_nece_not}]
Without loss of generality, we absorb $\rho$ in $\bB$, and reorder nodes so that the first $K$ nodes contain one pure node from each community. Thus, $\bTheta({1:K,:})=\bI_K$. 

Let $\bP = \bV \bE \bV^T$ be the eigen-decomposition of $\bP$, with $\bV\in\mathbb{R}^{n\times\text{rank}(\bB)}$. Let $\bV_P=\bV\bbb{1:K,:}$.
	Lemma~\ref{lem:v-theta} shows that
	$\bV = \bTheta \bV_P$.
	Thus, for any node $i$, $\bV(i,:)$ lies in the convex hull of the $K$ rows of $\bV_P$, that is, $\bV(i,:)\in\conv(\bV_P)$. We will slightly abuse the classical notation to denote by $\conv(\bM)$ the convex hull of the rows of matrix $\bM$.
	
  Now, suppose $\bP$ can be generated by another set of parameters $(\bTheta', \bB')$, where $\bTheta'$ has a different set of pure nodes, with indices $\mathcal{I} \neq 1:K$. 
  By the previous argument, we must have ${\bV({\mathcal{I}},:)} \subseteq \conv\bbb{\bV_{P}}$.
  Since $(\bTheta', \bB')$ and $(\bTheta, \bB)$ have the same probability matrix $\bP$, they have the same eigen-decomposition up to a permutation of the communities.
  Thus, swapping the roles of $\bTheta$ and $\bTheta'$ and reapplying the above argument, we find that ${\bV_{P}}  \subseteq \conv\bbb{\bV({\mathcal{I}},:)}$.
  Then $\conv\bbb{\bV_{P}}  \subseteq \conv\bbb{\bV({\mathcal{I}},:)} \subseteq \conv\bbb{\bV_{P}}$, so we must have $\conv\bbb{\bV_{P}}= \conv\bbb{\bV({\mathcal{I}},:)}$. 
  This means the pure nodes in $\bTheta$ and $\bTheta'$ are aligned up to a permutation, that is, $\bV({\mathcal{I}},:)=\bM \bV_{P}$, where $\bM\in \R^{K\times K}$ is a permutation matrix. 

  Now,
  $\bV = \bTheta \bV_P=\bTheta' \bV({\mathcal{I}},:) = \bTheta' \bM \bV_P$, which implies 
  \begin{equation}\label{eq:iden_core}
  	(\bTheta - \bTheta' \bM) \bV_P = 0
  \end{equation}
  Since $\bV = \bTheta \bV_P$ and $\rank(\bTheta) = K$, we have $\rank(\bV_P) = \rank(\bV) = \rank(\bB)$.
  Hence, if $\rank(\bB) = K$, $\bV_P$ is full rank, so $\bTheta = \bTheta' \bM$.
  Thus, $\bTheta$ and $\bTheta'$ are identical up to a permutation.
  To have the same $\bP$, $\bB$ and $\bB'$ must also be identical up to the same permutation.
  Hence, the MMSB model is identifiable.
  This proves part~(a).
  
  Now, suppose $\rank(\bB) = K-\ell < K$.
  We first permute the columns of $\bTheta$, and the rows and columns of $\bB$, so that
  \begin{align}
  \sbox0{$\bC$}
    \bB=\left[
    \begin{array}{c|c}
    \usebox{0}&\makebox{ \small{$\bC\bW$}}\\
    \hline
    \makebox{ \small{$\bW^T\bC$}}&\makebox{ \small{$\bW^T\bC\bW$}}
    \end{array}
    \right],
  \label{eq:BusingW}
  \end{align}
	where $\bC \in \R^{(K-\ell)\times (K-\ell)}$ is full rank, and $\bW \in \R^{(K-\ell)\times \ell}$.
  We see that
	\begin{align*}
    \bC \left[ \begin{array}{c|c} \bI_{K-\ell} & \bW \end{array} \right] &= \bV\bbb{1:(K-\ell),:} \bE \bV_P^T, \\
      \bW^T \bC \left[ \begin{array}{c|c} \bI_{K-\ell} & \bW \end{array} \right] &= \bV\bbb{(K-\ell+1):K,:} \bE \bV_P^T.
	\end{align*}
	The first equation shows that $\rank(\bV\bbb{1:(K-\ell),:}) = \rank(\bC) = K-\ell$, so $\bV\bbb{1:(K-\ell), :}$ is full rank.
  Hence,
	\begin{align}
  &\bV\bbb{(K-\ell+1):K,:} =\bW^T \bV\bbb{1:(K-\ell),:}
  \Rightarrow \bV_P = 
    \left[
		\begin{array}{c}
			\makebox{ \small{$\bI_{K-\ell}$}}\\
			\hline
			\makebox{ \small{$\bW^T$}}
		\end{array}
		\right]\bV\bbb{1:(K-\ell),:}
    \label{eq:identTheta}
  \end{align}
\medskip
\noindent {\em Case 1: $\rank(\bB) = K-1$ (so $\bW$ is a vector) and $\bW^T \bone_{K-\ell} \neq 1$.}
  
  {Now using Eqs~\eqref{eq:iden_core} and \eqref{eq:identTheta}, we have
    \begin{align}
    & (\bTheta - \bTheta' \bM) \left[
    \begin{array}{c}
    	\makebox{ \small{$\bI_{K-\ell}$}}\\
    	\hline
    	\makebox{ \small{$\bW^T$}}
    \end{array}
    \right]\bV\bbb{1:(K-\ell),:} = \bzero
    \Rightarrow 
      \bTheta = \bTheta'\bM.
    \end{align}
  The above equation is derived using $\bTheta\bone_K=\bTheta'\bone_K=\bone_n$, and $\bW^T\bone_{K-\ell} \neq 1$.
  
  Clearly $\bB'=\bM\bB\bM'$ as well, so the MMSB model is identifiable.}
  From Eq~\eqref{eq:BusingW}, we have $\bB\bbb{(K-\ell+1):K, :} = \bW^T \bB\bbb{1:(K-\ell), :}$,
  so $\bW^T\bone_{K-\ell} \neq 1$ if and only if the last row of $\bB$ is not a affine combination of the remaining rows.
  It is easy to see that the same holds for any row of $\bB$. This proves part~(b).

\medskip
\noindent {\em Case 2: $\rank(\bB) = K-1$ and $\bW^T \bone_{K-\ell} = 1$, or $\rank(\bB) < K-1$.}

  We will construct a $\bTheta'\neq \bTheta$ that yields the same probability matrix $\bP$.
  Let the completely mixed node be $m$, so $\theta_{mj} > 0$ for all communities $j$. We use
  $$\btheta'_j = \left\{
    \begin{array}{ll}
    \btheta_j & \text{if } j\neq m\\
    \btheta_m + \epsilon \bbeta^T \left[-\bW^T \mid \bI_l\right] & \text{if } j = m,
    \end{array}\right.,$$
  where $\epsilon$ is small enough that $\theta'_{mj}\in(0, 1)$ for all communities $j$, and  $\bbeta\in\mathbb{R}^{\ell}\neq \bzero$ is such that {$\bbeta^T\left[-\bW^T\bone_{K-\ell} + \bone_\ell\right]=0$.
  Note that such a $\bbeta$ always exists when $\ell>1$ and can be arbitrary vector when $\bW^T\bone_{K-\ell} = \bone_\ell$.}
  Hence, each row of $\bTheta'$ sums to $1$, and $\bTheta'$ is a valid community-membership matrix.
  Additionally, $\bTheta' \bV_P = \bTheta \bV_P$.

  Finally, we will show that $(\bTheta', \bB)$ and $(\bTheta, \bB)$ generate the same probability matrix.
  Note that $\bB = \bP_{1:K, 1:K} = \bV_P \bE \bV_P^T$. Hence,
  \begin{align*}
  \bTheta \bB \bTheta^T= \bTheta \bV_P \bE \bV_P^T \bTheta^T = \bV \bE \bV^T =\bP = \bTheta' \bV_P \bE \bV_P^T \bTheta'^T = \bTheta' \bB \bTheta'^T.
  \end{align*}
  This proves part~(c).
\end{proof}

\begin{proof}[Proof of Theorem~\ref{mmsb_iden_nece}]
  Consider an MMSB model parameterized by $(\bTheta^{(1)}, \bB^{(1)})$, with $\bP=\bTheta^{(1)} \bB^{(1)} {\bTheta^{(1)}}^T$ (we absorb $\rho$ in $\bB$ without loss of generality). 
  We want to construct a $(\bTheta^{(2)}, \bB^{(2)})$ that gives the same probability matrix $\bP$.
  The idea is to construct a matrix $\bM$ such that $\bTheta^{(2)} = \bTheta^{(1)} \bM$ and $\bB^{(2)} = \bM^{-1}\bB^{(1)}(\bM^T)^{-1}$. 
  The difficulty is in ensuring that all constraints are satisfied: $\bTheta^{(2)} \bone_K = \bone_n$, $\bTheta^{(2)} \geq 0$, and $0\leq \bB^{(2)}_{ij} \leq 1$ for all $i,j$.

  Without loss of generality, suppose that the first community does not have any pure nodes. In other words, for all nodes $i\in[n]$, $\theta^{(1)}_{i1}\leq 1-\delta$ for some $\delta > 0$. 
  Consider the following $\bM$:
  $$\bM=\left[
    \begin{array}{c|c}
    \makebox{$1+(K-1)\epsilon^2$} & \makebox{  $-\epsilon^2 \bone_{K-1}^T$ }\\
    \hline
    \makebox{$\bzero$}&\makebox{ $\epsilon \bone_{K-1} \bone_{K-1}^T + (1-(K-1)\epsilon)\bI_{K-1}$ }
    \end{array}
    \right],$$
  where $\epsilon$ is a small positive number ($0<\epsilon<\delta$).
	It is easy to check that $\bM$ is full rank (for small enough $\epsilon$) and $\bM \cdot \bone_K = \bone_K$.
  Hence, $\bTheta^{(2)} \bone_K = \bTheta^{(1)} \bM \bone_K = \bone_n$

  For any node $i$ and for $j>1$,
  \bas{
  	\theta^{(2)}_{i1} &=\theta^{(1)}_{i1}(1+(K-1)\epsilon^2)\geq 0, \\
  	\theta^{(2)}_{ij} &=-\theta^{(1)}_{i1}\epsilon^2+\sum_{\ell=2}^{K}\theta^{(1)}_{i\ell}\bM_{\ell j}
  	\geq -\theta^{(1)}_{i1}\epsilon^2+\epsilon\sum_{\ell=2}^{K}\theta^{(1)}_{i\ell} 
  	\geq \epsilon\delta^2 > 0,
  }
  where we used $\epsilon<\delta$ and $\btheta_i^{(1)} \bone_K = 1$.
  Hence, $\bTheta^{(2)} \geq 0$. 

  Finally, we must show that $\bB^{(2)} = \bM^{-1}\bB^{(1)}(\bM^T)^{-1}$ has all elements between $0$ and $1$.
	Note that
	\bas{
		\bM -\bI_K = 
		\left[
		\begin{array}{c|c}
			\makebox{$(K-1)\epsilon^2$} & \makebox{ $-\epsilon^2 \bone_{K-1}^T$ }\\
			\hline
			\makebox{$\bzero$}&\makebox{ $\epsilon \bone_{K-1} \bone_{K-1}^T -(K-1)\epsilon\bI_{K-1}$ }
		\end{array}
		\right],
	}
	so $\left\|\bM-\bI_K\right\|_F \rightarrow 0$ as $\epsilon\rightarrow 0$. Since $\bM^{-1}$ is continuous in $\bM$, we have $\left\|\bM^{-1}-\bI_K\right\|_F \rightarrow 0$.
  Thus,
	\begin{align*}
		\left\| \bB^{(2)} - \bB^{(1)} \right\|_F
		&=\left\| \bM^{-1}\bB^{(1)}(\bM^T)^{-1} - \bB^{(1)} \right\|_F \\
		&\leq \left\| \bM^{-1} - \bI_K \right\|_F^2 \left\| \bB^{(1)} \right\|_F + 2 \left\| \bM^{-1} - \bI_K \right\|_F \left\| \bB^{(1)} \right\|_F \\
		&\rightarrow 0 \text{ as $\epsilon\rightarrow 0$}.
	\end{align*}
  Since $\bB^{(1)}_{ij}\in (0,1)$, we have $\bB^{(2)}_{ij}\in (0,1)$ for $\epsilon$ small enough, completing the proof.
\end{proof}

%% file: concentration2.tex
\begin{defn}(A construction of rotation matrix)
	\label{def:o}
	Consider the discretization defined in Definition~\ref{def:int}. The Davis-Kahan Theorem states that there exists a rotation matrix $\hat{\bO}$ such that $\|\hat{\bV}-\bV\hat{\bO}\|_F$ is small. In this definition we will carefully construct this matrix. 
	Consider the intervals resulting from the discretization of population eigenvalues in Definition~\ref{def:int}.
	Now, from Theorem 2 of~\citep{yu2015useful},  $\exists \hat{\bO}_k$ such that
	\ba{\label{eq:rkfrob}
		\|\bR_k\|_F=\|\hat{\bV}_{S_k}-\bV_{S_k}\hat{\bO}_k\|_F&\leq \frac{\sqrt{8n_k}\|\bA-\bP\|}{g_k}
	}
	Typically the denominator is  $f_k:=\min(\lambda_{s_k}-\lambda_{s_k-1},\min(\lambda_{e_k}-\lambda_{e_k+1},\lambda_{e_k}))$. 
	We now construct our $\hat{\bO}$ by stacking the $\hat{\bO}_k$ matrices on the diagonal of a $K\times K$ matrix. This is also a valid rotation matrix. Now, let  $\bE_k$ by the submatrix of $\bE$ corresponding to eigenvalues in $S_k$. Similarly define $\hat{\bE}_k$. Furthermore, let $\bR:=[\bR_1|\dots|\bR_I]$. 
\end{defn}

\begin{lem}\label{lem:E-lemma-op}
  If Assumption~\ref{as:theta_P} holds,
	then there exists an orthogonal matrix $\hat{\bO} \in \mathbb{R}^{K\times K}$ constructed using Definition~\ref{def:o} 
	that satisfies
	\ba{
			\|\bR\|_F&\leq \frac{\sqrt{8K}\|\bA-\bP\|}{\lambda^*(\bP)}\label{eq:r-frob}\\
		\left\|\hat{\bE}-\hat{\bO}^T\bE\hat{\bO}\right\|_F&=O_P\bbb{K^{2}\sqrt{n\rho}}\label{eq:e-frob}
	}
	with probability larger than $1-n^{-3}$.
\end{lem}
\begin{proof}
	Consider the rotation matrix $\oh$, the residual matrix $\bR$ constructed as in Definition~\ref{def:o}.
		 This gives us the Frobenius norm of $\bR$ as follows, since by construction $g_k\geq \lambda^*(\bP)$.
		\bas{
		\|\bR\|_F	\leq \sqrt{\sum_k \|\bR_k\|^2_F}\leq \frac{\sqrt{8K}\|\bA-\bP\|}{\lambda^*(\bP)} 
}
Finally note that, using Lemma~\ref{lem:e-discr}, since $\lambda_{s_k}\leq \sum_{i=1}^k n_i g_k\leq Kg_k$,
	\begin{equation}\label{eq:rkek}
\|\bR_k\oh_k^T\bE_k\|_F\leq \|\bR_k\|_F\|\bE_k\| \leq
\dfrac{\sqrt{8n_k}\|\bA-\bP\|\lambda_{s_k}}{g_k}=O_P(K\sqrt{n_k}\sqrt{n\rho})
\end{equation}
	
Now we use these intervals as follows.
\bas{
	\left\|\hat{\bE}-\hat{\bO}^T\bE\hat{\bO}\right\|_F
	&=\left\|\hat{\bV}\hat{\bE}\hat{\bV}^T-\hat{\bV}\hat{\bO}^T\bE\hat{\bO}\hat{\bV}^T\right\|_F\nonumber\\
	&=\|\bA_K-(\bV+\bR\oh^T)\bE(\bV+\bR\oh^T)^T\|_F\nonumber\\
	&\leq  \|\bA_K-\bP\|_F+2\underbrace{\|\bR\oh^T\bE\bV^T\|_F}_{P_1}+\underbrace{\|\bR\oh^T\bE\oh\bR^T\|_F}_{P_2}\nonumber\\
	&=O_P(\sqrt{Kn\rho})+P_1+P_2
}
	
The last step is true because $\|\bA_K-\bP\|_F\leq \sqrt{K}\|\bA_K-\bP\|\leq \sqrt{K}(\|\bA-\bA_K\|+\|\bA-\bP\|)\leq 2\sqrt{Kn\rho}$ with probability at least $1-n^{-r}$ using Weyl's inequality and Theorem~5.2 of \cite{lei2015consistency}. As for $P_1$, note that: $P_1\leq \|\bR\oh^T\bE\|_F\leq \sqrt{\sum_k\|\bR_k\oh_k^T\bE_k\|_F^2}=:O_P(K^{3/2}\sqrt{n\rho})$. 
	
As for $P_2$, we have:
\bas{
	P_2\leq \|\bR\oh^T\bE\|_F\|\bR\|_F=
	O_P\left(\frac{K^{2}n\rho}{\lambda^*(\bP)}\right)
}
Thus the final bound is $O_P(K^{2}\sqrt{n\rho}(\max(1/K^{3/2},1/\sqrt{K},\sqrt{n\rho}/\lambda^*(\bP)))=O_P(K^{2}\sqrt{n\rho})$ using Assumption~\ref{as:theta_P}. 
The failure probability comes from the failure of event $\|\bA-\bP\|=O_P(\sqrt{n\rho})$. Taking $r=3$ we get the required bound. 
\end{proof}

%% file: supp-mmsb.tex
\begin{lem}\label{lem:lambda_1_Theta}
	For $\bTheta\in \R^{n\times K}$, with $\|\btheta_i\|_1=1$ and $\theta_{ij}\geq 0$, $\forall i, j \in [n]$, $\lambda_{1}(\bTheta^T\bTheta)\leq \max_j \bone_n^T\bTheta\be_j\leq n$ and $\lambda_{K}(\bTheta^T\bTheta)\leq \min_j \bone_n^T\bTheta\be_j$.
	\begin{proof}
		By Proposition~2.4 of \citep{schaefer1974banach}, as $\bTheta^T\bTheta$ is a nonnegative matrix, $\lambda_{1}(\bTheta^T\bTheta)$ is upper bounded by its largest row sum and $\lambda_{K}(\bTheta^T\bTheta)$ is lower bounded by its smallest row sum. For the $i$-th row of $\bTheta^T\bTheta$, its row sum is:
		\bas{
			\be_i^T\bTheta^T\bTheta\bone_K
			=\be_i^T\bTheta^T\bone_n
			=\bone_n^T\bTheta\be_i
			\leq n.
		}
		Thus the result follows.
	\end{proof}
\end{lem}

\begin{lem} \label{lem:theta_condition_num}
	Under Assumption~\ref{as:ident}, we have $\bTheta^T \bTheta=\bbb{\bV_P \bV_P^T}^{-1}$, which implies $\lambda_1({\bV_P \bV_P^T}) = 1/\lambda_K({\bTheta^T \bTheta})$, $\lambda_K({\bV_P \bV_P^T}) = 1/\lambda_1({\bTheta^T \bTheta})$ and $\kappa(\bV_P \bV_P^T)=\kappa\bbb{\bTheta^T \bTheta}$.
	\begin{proof}
		From Lemma~\ref{lem:v-theta}, $\bV = \bTheta\bV_P$, so
		\bas{
			\bI = \bV^T\bV=\bV_P^T\bTheta^T \bTheta\bV_P.
		}
		As $\bV_P$ is full rank, we have $\bTheta^T \bTheta=\bbb{\bV_P \bV_P^T}^{-1}$, which gives
		\bas{
			\lambda_1({\bV_P \bV_P^T})=\frac{1}{ \lambda_K\bbb{\bTheta^T \bTheta}} \quad \text{ and } \quad \lambda_K({\bV_P \bV_P^T})=\frac{1}{ \lambda_1\bbb{\bTheta^T \bTheta}}, 
		}
		so  $\kappa(\bV_P \bV_P^T)=\kappa\bbb{\bTheta^T \bTheta}$.
	\end{proof}
\end{lem}
\begin{proof}[Proof of Lemma~\ref{lem:theta_property}]
	 If $\btheta_i\sim\mathrm{Dirichlet}(\balpha)$, let us consider $\btheta_i$ as a random variable. Denote 
	\bas{
		\hat{\bM}= \bTheta^T\bTheta
		=\sum_{i=1}^{n}  \btheta_i\btheta_i^T.
	}
	Note that  $\hat{\bM}-\uE[\hat{\bM}]=\sum_i \bX_i$ where $\bX_i$ are independent mean zero symmetric $K\times K$ random matrices.  We have 
	\bas{
		\uE[\btheta_i\btheta_i^T]=\frac{\diag(\balpha)+\balpha\balpha^T}{\alpha_0(1+\alpha_0)}\qquad \cov(\btheta_i)=\frac{\alpha_0\diag(\balpha)-\balpha\balpha^T}{\alpha_0^2(1+\alpha_0)}. 
	}
	
	Furthermore, since $\|\btheta_i\|_1=1$, and $\alpha_0=\sum_i\alpha_i$, we have $\|\bX_i\|\leq \btheta_i^T\btheta_i+\|\uE[\btheta_i\btheta_i^T]\|\leq 1+ \frac{\alpha_{\max}+\|\balpha\|^2}{\alpha_0(1+\alpha_0)}\leq 2$.
	Finally, since the operator norm is convex, Jensen's inequality gives: $\|\uE[\bX_i^2]\|\leq \uE[\|\bX_i^2\|]\leq \uE[\|\bX_i\|^2] \leq 4$. 
	Using standard Matrix Bernstein type concentration results (Theorem 1.4 of~\citep{tropp2012user}), for large $n$ we get:
	\bas{
		\uP(\|\hat{\bM}-\uE[\hat{\bM}]\|\geq t)&\leq K\exp\bbb{-\frac{t^2}{8n+4t/3}}=:\delta_t
	}
	Now Weyl's inequality gives, with probability at least $1-\delta_t$, 
	\bas{
		|\lambda_1(\hat{\bM})-\lambda_1(\uE[\hat{\bM}])|\leq t\qquad |\lambda_K(\hat{\bM})-\lambda_K(\uE[\hat{\bM}])|\leq t
	}
	
	For the population quantities, 
	\bas{
		{\lambda_1(\uE[\hat{\bM}])}\leq \frac{\alpha_{\max}+\|\balpha\|^2}{\alpha_0(1+\alpha_0)}n,\qquad
		{\lambda_K(\uE[\hat{\bM}])}\geq
		\frac{\alpha_{\min}}{\alpha_0(1+\alpha_0)}n
	}
	
	For $\lh_1(\hat{\bM})$ we take $t=\frac{n}{2}\frac{\alpha_{\max}+\|\balpha\|^2}{\alpha_0(1+\alpha_0)}\in[\frac{n}{2\nu(1+\alpha_0)},  \frac{n}{2}]$ and hence $\delta_t\leq K\exp\left(-\frac{n}{36\nu^2(1+\alpha_0)^2}\right)$.
	For $\lh_K(\hat{\bM})$, we take $t=\frac{n}{2}\frac{\alpha_{\min}}{\alpha_0(1+\alpha_0)}\in [\frac{n}{2\nu(1+\alpha_0)},\frac{n}{2}]$. Hence $\delta_t\leq K\exp(-\frac{n}{36\nu^2(1+\alpha_0)^2})$.
	
	Hence the condition number of $\hat{\bM}$ can also be bounded as:
	\begin{align*}
	\kappa(\bTheta^T\bTheta) &= \frac{\lambda_1(\bTheta^T\bTheta)}{\lambda_K(\bTheta^T\bTheta)}
	\leq
	\frac{{\frac{3n}{2}\frac{\alpha_{\max}+\|\balpha\|^2}{\alpha_0(\alpha_0+1)}} }{\frac{n}{2}\frac{\alpha_{\min}}{\alpha_0(\alpha_0+1)}}
	=3{ \frac{\alpha_{\max}+\|\balpha\|^2}{\alpha_{\min}}} 
	\end{align*}
\end{proof}

\begin{lem}
	\label{lem:P_eigen}
	Let $\lambda^*(\bP)$ denote the $K^{th}$ largest singular value of $\bP$. We have $\lambda^*(\bP)\geq\rho \lambda^*(\bB)  \lambda_K( \bTheta^T \bTheta)  $.
\end{lem}
\begin{proof}
First note that, by Theorem~1.3.22 of~\citep{horn2012matrix}, we have $(\bB\bTheta^T)\bTheta$ and $\bTheta(\bB\bTheta^T)$ have the same $K$ largest eigenvalues in magnitude, then
	\ba{
		\lambda^*(\bP)
		=\lambda^*(\rho \bTheta \bB \bTheta^T)
		=\lambda^*(\rho \bB \bTheta^T\bTheta )
		\geq  \rho \lambda^*(\bB)  \lambda_K( \bTheta^T \bTheta).\label{eq:lambda_P}
	}
	The inequality holds because for all full rank positive definite matrix $\bM_1$, $\bM_2\in\R^{K\times K}$, $\|{(\bM_1\bM_2)^{-1}}\|\leq\|{\bM_1^{-1}}\|\|{\bM_2^{-1}}\|$, and as $\sigma_K(\bM_1)=1/\|\bM_1^{-1}\|$ (same for $\bM_1$ and $\bM_1\bM_2$), where $\sigma_K(.)$ denotes the $K^{th}$ largest singular value of a matrix. Then  
	we have:
	\bas{
		\sigma_K{(\bM_1\bM_2)}\geq  \sigma_K\bbb{\bM_1} \sigma_K\bbb{\bM_2}.
	}
\end{proof}

\begin{proof}[Proof of Lemma~\ref{lem:v-row-norm}]
	Note that $\bTheta^T \bTheta=\bbb{\bV_P \bV_P^T}^{-1}$ by Lemma~\ref{lem:theta_condition_num}, thus for pure nodes,
	\bas{
		\max_i\left\| \be_i^T\bV_P \right\|^2
		=  \max_i\be_i^T\bV_P \bV_P^T \be_i 
		\leq \max_{\|\bx\|=1} \bx^T\bV_P \bV_P^T\bx 
		= \lambda_1\bbb{\bV_P \bV_P^T} 
		= \frac{1}{\lambda_K\bbb{\bTheta^T \bTheta}}
	} 
	As for other nodes, their rows are convex combinations of the rows of pure nodes and would be smaller than or equal to the norm of the pure nodes. Thus the result follows.
	
	Note that by Lemma \ref{lem:v-theta}, $\be_i^T\bV=\btheta_i^T\bV_P$, then
	\bas{
		\min_i\left\| \be_i^T\bV \right\|^2
		&=  \min_i\btheta_i^T\bV_P \bV_P^T \btheta_i^T 
		= \min_i \|\btheta_i\|^2 \frac{\btheta_i^T}{\|\btheta_i\|}\bV_P \bV_P^T\frac{\btheta_i}{\|\btheta_i\|}
		\geq \min_i \|\btheta_i\|^2 \min_{\|\bx\|=1} \bx^T\bV_P \bV_P^T\bx \\
		&= \min_i \|\btheta_i\|^2 \lambda_K\bbb{\bV_P \bV_P^T} 
		= \frac{\min_i \|\btheta_i\|^2}{\lambda_1\bbb{\bTheta^T \bTheta}}
		\geq \frac{1}{K\lambda_1\bbb{\bTheta^T \bTheta}}
	}
	where for the last inequality we use for any $i\in[n]$, $\|\btheta_i\| \geq \|\btheta_i\|_1/\sqrt{K}=1/\sqrt{K}$.
	Thus the result follows.
\end{proof}

%% file: concentration1.tex
\begin{proof}
For ease of notation we will {first} prove this for one population eigenvector $\bv$. Recall that $\bH:=(\bA-\bP)/\sqrt{n\rho}$.	Let  
	$X_j=(\bA_{ij}-\bP_{ij})v_j$, where $v_j$ is the $j^{th}$ component of $\bv$.
	We have $|X_j|\leq \|\bv\|_\infty=: M$.  
	Since $\bTheta$ is assumed to be fixed for this lemma, $\bP$ is fixed and 
	$X_j$ are mean zero independent random variables.
	Also note that, since $\|\bv\|=1$ and $\bP_{ij}\leq \rho$, 
		\bas{
		\sum_j\var(X_j)=
		\sum_j \uE\ccc{(\bA_{ij}-\bP_{ij})^2v_j^2}
		=
		\sum_j \bP_{ij}(1-\bP_{ij})v_j^2
		\leq \rho
	}
An application of Bernstein's inequality gives us:
    \bas{
    	&\uP\bbb{|\sum_j X_j|\geq t}
    	\leq 2 \exp\left(-\frac{t^2}{2(\sum_j\var(X_j)+tM/3)}\right) =: 2 \exp(-A)
    }
    First note that the RHS of the above equation is a decreasing function of $t$. We set $t=4\max(M,\sqrt{\rho})\log n$. Consider the following two cases:
    
    \medskip
    {\em Case 1: $M>\sqrt{\rho}$:}
    We have $t= 4M\log n$. Hence,
    \bas{
		\exp(-A)\leq \exp\bbb{-\frac{16M^2\log^2 n}{2\rho + 8/3M^2\log n}}    \leq \exp\bbb{-\frac{16M^2\log^2 n}{2M^2 + 8/3M^2\log n}} \leq \frac{1}{n^3}
}

 \medskip
{\em Case 2: $M\leq \sqrt{\rho}$:}
         We have $t= 4\sqrt{\rho}\log n$. Hence,
    \bas{
    	\exp(-A)\leq \exp\bbb{-\frac{16\rho\log^2 n}{2\rho + 8/3M\sqrt{\rho}\log n}}    \leq \exp\bbb{-\frac{16\rho\log^2 n}{2\rho + 8/3\rho\log n}} \leq \frac{1}{n^3}
    }
    
    Applying this to all $K$ population eigenvectors we have:
    \ba{
    	\label{eq:bernstein}
    	\uP\bbb{\exists k\in[K],|\be_i^T(\bA-\bP)\bv_k|\geq 4\max(\|\bv_k\|_\infty,\sqrt{\rho})\log n}\leq \frac{2K}{n^3}
}
Recall from Lemma~\ref{lem:v-row-norm} that, $\forall k\in[K],\|\bv_k\|_\infty\leq \max_i\|\be_i^T\bV\|\leq \sqrt{\frac{1}{\lambda_K(\bTheta^T\bTheta)}}$. 
{Then if Assumption~\ref{as:theta_P} is satisfied}, we have, $\|\bv_k\|_\infty\leq \sqrt{\rho}$, $\forall k\in [K]$. 
So from Eq.~\eqref{eq:bernstein},
\bas{
	\uP\bbb{\exists k\in[K],|\be_i^T(\bA-\bP)\bv_k|\geq 4\sqrt{\rho\log^2 n}}\leq \frac{2K}{n^3}.
}
Note that $\forall k\in [K]$, $\|\bv_k\|_\infty\geq 1/\sqrt{n}$, then,
\bas{
	\uP\bbb{\exists k\in [K],|\be_i^T\bH\bv_k|\geq 4\log n\|\bv_k\|_\infty}&\leq \uP\bbb{\exists k\in [K],|\be_i^T\bH\bv_k|\geq 4\sqrt{\frac{\log^2n}{n}}}
	=O\bb{\frac{K}{n^3}}.
}
\par \vspace{-0.5\baselineskip}
\qedhere
\end{proof}

%% file: erdos-proof.tex
\subsection{Proof of Lemma~\ref{cor:erdos}}
\label{sec:supp-proof-erdos}
\begin{proof}[Proof of Lemma~\ref{cor:erdos}]
	For $t\geq 2$, we claim that this result follows via straightforward modifications of the proof of Lemma 7.10 in~\citep{erdos2013}, where the
	main two elements are:
	\begin{enumerate}
		\item  $\uE[|\bH_{ij}|^{m}]\leq \frac{1}{n}$ for $m\geq 2$. Note that for our setting, Assumption~\ref{as:theta_P} implies that $n\rho\geq 1$. Hence $|\bH_{ij}|\leq 1$, and hence 
		\bas{
			\uE[|\bH_{ij}|^m]\leq\uE[|\bH_{ij}|^2]\leq\frac{\bP_{ij}}{n\rho}\leq \frac{1}{n}
		}
		\item The authors use a higher order Markov inequality. This inequality upper bounds the number of terms that are non-zero in the summand via a multigraph construction for path counting. Then these non-zero elements are bounded by their absolute value and hence, even though $\bv$ does not equal $\be$, just the fact that it is fixed and hence independent of $\bH_{ij}$, is enough to apply the proof directly to get the required result.
	\end{enumerate}
	
	Using an almost identical argument as~\citep{erdos2013}, we have:
	\bas{
		\uE[|\be_i^T \bH^t \bv|^p]\leq (tp)^{tp}\|\bv\|_\infty^p
	}
	
	Now a higher order Markov inequality, with $p=(\log n)^\xi/2t$ gives:
	\bas{
		\uP\bbb{|\be_i^T \bH^t \bv|\geq (\log n)^{t\xi}\|\bv\|_\infty}&\leq \dfrac{(tp)^{tp}\|\bv\|_\infty^p}{ (\log n)^{tp\xi}\|\bv\|_\infty^p}=\frac{1}{\sqrt{2}^{(\log n)^{\xi}}}\\
		&=\exp(-(\log n)^\xi\log \sqrt{2})\leq \exp(-(\log n)^\xi/3)
	}
\qedhere
\end{proof}

%% file: lemma_for_abbe.tex
\section{Comparison with~\citep{abbe2017entrywise} on row-wise deviation of eigenspace}
\label{sec:supp-abbe}
{
Here we give a row-wise error bound for eigenspace using \citet{abbe2017entrywise}'s result. 
\begin{lem} \label{lem:Abbe}
	For $\bP=\bV\bE\bV^T$ and $\bA=\hv\hat{\bE}\hv^T$ as $\bP$ and $\bA$'s top-$K$ eigen-decomposition respectively, we have
	\bas{
		\|\hat{\bV}sgn(\hv^T\bV)-\bA\bV\bE^{-1}\|_{2\rightarrow \infty}
		&
		=O_P\left( \frac{(\kappa(\bTheta^T\bTheta))^2{K}\sqrt{n}}{\sqrt{\rho}(\lambda^*(\bB))^3(\lambda_K(\bTheta^T\bTheta))^{1.5}} \right)\\
		\|\hat{\bV}sgn(\hv^T\bV)-\bV\|_{2\rightarrow \infty}
		&=O_P\left({\max}\left( \frac{(\kappa(\bTheta^T\bTheta))^2{K}\sqrt{n}}{\sqrt{\rho}(\lambda^*(\bB))^3(\lambda_K(\bTheta^T\bTheta))^{1.5}} ,{\frac{1}{\sqrt{\lambda_K(\bTheta^T\bTheta)}}}\right)\right),
	}
	where $\|\bU\|_{2\rightarrow \infty}=\max_i \|\be_i^T\bU\|$ denotes the maximum row norm of a matrix $\bU$, and $sgn(\hv^T\bV)$ is the matrix sign function
	\bas{
		sgn(\hv^T\bV)=\bU_1\bU_2^T, \quad \mbox{SVD of $\hv^T\bV$ is $\hv^T\bV=\bU_1\bf{\Sigma}\bU_2^T$}.
	}
\begin{proof}
	First from Assumption A3 of \citep{abbe2017entrywise} we have $c\sqrt{\rho n} \leq \gamma \Delta^*$ for some constant $c$. $\Delta^*$ is the eigen-gap, which is $\lambda^*(\bP)\geq \rho \lambda^*(\bB)\lambda_K(\bTheta^T\bTheta)$ from Lemma~\ref{lem:P_eigen}. This requires
	\begin{equation}
	\gamma\geq\frac{c\sqrt{n}}{\sqrt{\rho}\lambda^*(\bB)\lambda_K(\bTheta^T\bTheta)}.
	\end{equation}
	Using Eqs (13) of Corollary 2.1 in \citep{abbe2017entrywise}, we have,
	\bas{
		\|\hat{\bV}sgn(\hv^T\bV)-\bA\bV\bE^{-1}\|_{2\rightarrow \infty}&\leq \kappa(\kappa+\varphi(1))(\gamma+\varphi(\gamma))\|\bV\|_{2\rightarrow \infty}\\
		\mbox{(Lemma~\ref{lem:v-row-norm})}\quad&\leq \kappa(\kappa+\varphi(1))(\gamma+\varphi(\gamma)){\frac{1}{\sqrt{\lambda_K(\bTheta^T\bTheta)}}}.
	}
	$\kappa$ is the condition number of $\bP$ which is upper bounded by:
	\bas{
		\kappa\leq \kappa(\bTheta^T\bTheta) \kappa(\bB)\leq \kappa(\bTheta^T\bTheta) \frac{{\|\bB\|}}{\lambda^*(\bB)}\leq \kappa(\bTheta^T\bTheta) \frac{\sqrt{K}}{\lambda^*(\bB)}.
	}
	Since $\varphi(x)$ is typically bounded by a constant and $\varphi(x)/x$ non-increasing, we have,
	\ba{
		\|\hat{\bV}sgn(\hv^T\bV)-\bA\bV\bE^{-1}\|_{2\rightarrow \infty}
		&=O_P\left( \kappa^2\gamma {\frac{1}{\sqrt{\lambda_K(\bTheta^T\bTheta)}}}\right)\nonumber\\
		&=O_P\left(  \frac{(\kappa(\bP))^2\sqrt{n}}{\sqrt{\rho}\lambda^*(\bB)(\lambda_K(\bTheta^T\bTheta))^{1.5}} \right)\nonumber\\
		&=O_P\left(  \frac{(\kappa(\bTheta^T\bTheta))^2{K}}{(\lambda^*(\bB))^2} \frac{\sqrt{n}}{\sqrt{\rho}\lambda^*(\bB)(\lambda_K(\bTheta^T\bTheta))^{1.5}} \right)\nonumber\\
		&=O_P\left( \frac{(\kappa(\bTheta^T\bTheta))^2{K}\sqrt{n}}{\sqrt{\rho}(\lambda^*(\bB))^3(\lambda_K(\bTheta^T\bTheta))^{1.5}} \right).\label{eq:v_ae_abbe}
	}
	Furthermore, using Eqs (14) of Corollary 2.1 in \citep{abbe2017entrywise}, we have,
	\ba{
		\|\hat{\bV}sgn(\hv^T\bV)-\bV\|_{2\rightarrow \infty}
		&\leq \|\hat{\bV}sgn(\hv^T\bV)-\bA\bV\bE^{-1}\|_{2\rightarrow \infty} +\varphi(1)\|{\bV}\|_{2\rightarrow \infty}\nonumber\\
		&=O_P\left({\max}\left( \frac{(\kappa(\bP))^2\sqrt{n}}{\sqrt{\rho}\lambda^*(\bB)(\lambda_K(\bTheta^T\bTheta))^{1.5}} ,{\frac{1}{\sqrt{\lambda_K(\bTheta^T\bTheta)}}}\right)\right)\nonumber\\
		&=O_P\left({\max}\left( \frac{(\kappa(\bTheta^T\bTheta))^2{K}\sqrt{n}}{\sqrt{\rho}(\lambda^*(\bB))^3(\lambda_K(\bTheta^T\bTheta))^{1.5}} ,{\frac{1}{\sqrt{\lambda_K(\bTheta^T\bTheta)}}}\right)\right)
	\label{eq:vv_abbe}
	}
\end{proof}
\end{lem}	
}

\begin{rem}
	Our bound in Theorem~\ref{thm:entrywise} has better dependency on $\lambda^*(\bB)$ comparing to Eq~\eqref{eq:v_ae_abbe} when $\lambda^*(\bB)$ goes to 0 ($\kappa(\bP)$ goes to infinity). If $\kappa(\bP)=\Theta(1)$ or $\lambda^*(\bB)=\Theta(1)$ or $K=\Theta(1)$, the bound in Theorem~\ref{thm:entrywise} is comparable or better than that in Eq~\eqref{eq:v_ae_abbe}. However, in comparison to their eigenvector deviation bound from Eq~\eqref{eq:vv_abbe}, we have a tighter bound by an order of $1/\sqrt{n\rho}$ when $K=\Theta(1)$, $\lambda^*(\bB)=\Theta(1)$, $\kappa(\bP)=\Theta(1)$ and $\lambda_K(\bTheta^T\bTheta)=\Omega(n)$. As a matter of fact, when $\btheta_i\sim \mathrm{Dirichlet}(\balpha)$ and $\alpha_{\max}\leq C\alpha_{\min}$ for some constant $C\geq1$, we have $\nu=\alpha_0/\alpha_{\min}=\Theta(K)$ and by Lemma~\ref{lem:theta_property}, with high probability $\lambda_K(\bTheta^T\bTheta)=\Omega(n/\nu)=\Omega(n)$ when $K=\Theta(1)$.
\end{rem}

%% file: lemma_for_cape.tex
\section{Comparison with~\citep{cape2019two} on row-wise deviation of eigenspace}
\label{sec:supp-cape}
{
Here we give a row-wise error bound for eigenspace applying \citet{cape2019two}'s result.

\begin{lem} \label{lem:cape}
	For $\bP=\bV\bE\bV^T$ and $\bA=\hv\hat{\bE}\hv^T$ as $\bP$ and $\bA$'s top-$K$ eigen-decomposition respectively, if $\lambda^*(\bP)\geq \|\bA-\bP\|_\infty$, then there exists an orthognal matrix $\bW_\bV \in \R^{K \times K}$ such that
	\ba{\label{eq:v_ae_cape}
		\|\hV-\bV\bW_\bV\|_{2\rightarrow \infty} 
		=O_P\left(\frac{ n }{\lambda^*(\bB)(\lambda_K(\bTheta^T\bTheta))^{1.5}}\right).
	}
	\begin{proof}
		From Lemma~\ref{lem:v-row-norm}, $\|\bV\|_{2\rightarrow \infty} \leq {\frac{1}{\sqrt{\lambda_K(\bTheta^T\bTheta)}}}$. By applying Theorem~4.2 of~\citep{cape2019two}, if $\lambda^*(\bP)\geq \|\bA-\bP\|_\infty$,
		\bas{
			\|\hV-\bV\bW_\bV\|_{2\rightarrow \infty} \leq 14 \left(\frac{\|\bA-\bP\|_\infty}{\lambda^*(\bP)}\right)\|\bV\|_{2\rightarrow \infty},
		}
		where $\|\bA-\bP\|_{\infty}= \max_i \sum_j |\bA_{ij}-\bP_{ij}|$.
		Note that as $\uE[\sum_j |\bA_{ij}-\bP_{ij}|]=\sum_j \uE[|\bA_{ij}-\bP_{ij}|]=\sum_j [\bP_{ij}(1-\bP_{ij})+(1-\bP_{ij})\bP_{ij}]=O(\rho n)$. By applying Chernoff bound, it can be shown that for all $i\in[n]$, $\sum_j |\bA_{ij}-\bP_{ij}|=O_P(\rho n)$. Then $\|\bA-\bP\|_{\infty}= \max_i \sum_j |\bA_{ij}-\bP_{ij}|= O(\rho n)$ with high probability. 
		From Lemma~\ref{lem:P_eigen}, we have 
		$\lambda^*(\bP)\geq \rho \lambda^*(\bB)\lambda_K(\bTheta^T\bTheta)$. 
		Then we have,
		\bas{
			\|\hV-\bV\bW_\bV\|_{2\rightarrow \infty} 
			=O_P\left(\frac{ n }{\lambda^*(\bB)(\lambda_K(\bTheta^T\bTheta))^{1.5}}\right).
		}
	\end{proof}
\end{lem}

\begin{rem}
	Our bound in Theorem~\ref{thm:entrywise} is tighter by an order of $1/\sqrt{n\rho}$ comparing to Eq~\eqref{eq:v_ae_cape} when $K=\Theta(1)$.
\end{rem}

Note that Lemma~\ref{lem:cape} is a direct application of the perturbation bound in \citep{cape2019two} to the MMSB model. If we use the more careful analysis of the authors for the $\rho$-correlated SBM graphs, together with our Lemma~\ref{lem:azuma-better}, we can get a better bound as in the following Lemma.

\begin{lem} \label{lem:cape_tighter}
	For $\bP=\bV\bE\bV^T$ and $\bA=\hv\hat{\bE}\hv^T$ as $\bP$ and $\bA$'s top-$K$ eigen-decomposition respectively, and $\bW_\bV =sgn(\bV^T\hv) \in \R^{K \times K}$, we have
	\ba{\label{eq:v_ae_cape_tighter}
		\|\hV-\bV\bW_\bV\|_{2\rightarrow \infty} 
		=O_P\bbb{ \frac{{ n}}{{\rho} (\lambda^*(\bB))^2(\lambda_K(\bTheta^T\bTheta))^2} },
	}
when $\lambda_K(\bTheta^T\bTheta)=\Omega(K)$ and $\rho (\lambda^*(\bB))^2 \leq  \frac{n}{ K\lambda_K(\bTheta^T\bTheta)}$
	\begin{proof}
		Using Corallary~3.3 and Proposition 6.5 of \citep{cape2019two}, noting that $(\bI-\bV \bV^T)\bP=\bzero$, we have the following decomposition and bound:
		\bas{
			\|\hat{\bV}-\bV\bW_\bV\|_{2\rightarrow \infty}  
			\leq &\|(\bI-\bV \bV^T)(\bA-\bP) \bV \bW_\bV \hat{\bE}^{-1}\|_{2\rightarrow \infty} \\
			&+\|(\bI-\bV \bV^T)(\bA-\bP) (\hv-\bV \bW_\bV) \hat{\bE}^{-1}\|_{2\rightarrow \infty} \\
			&+\|\bV(\bV^T\hv-\bW_\bV)\|_{2\rightarrow \infty} \\
			\leq &\|(\bA-\bP) \bV\|_{2\rightarrow \infty}\| \hat{\bE}^{-1}\|
			+\|\bV\|_{2\rightarrow \infty} \|\bV^T\|\|(\bA-\bP) \bV\| \|\hat{\bE}^{-1}\|\ \\
			&+\|\bI-\bV \bV^T\|_{2\rightarrow \infty}\|\bA-\bP\| \|\hv-\bV \bW_\bV\|\|\hat{\bE}^{-1}\| \\
			&+\|\bV\|_{2\rightarrow \infty}\|\bV^T\hv-\bW_\bV\|.
		}
		By Lemma~\ref{lem:azuma-better}, we have $\|(\bA-\bP) \bV\|_{2\rightarrow \infty}=O_P(\sqrt{K n\rho}) \|\bV\|_{2\rightarrow \infty}$, so $$\|(\bA-\bP) \bV\|\leq \sqrt{K}\|(\bA-\bP) \bV\|_{2\rightarrow \infty}=O_P(K\sqrt{ n\rho}) \|\bV\|_{2\rightarrow \infty}.$$
		We also have $\|\bA-\bP\|=O_P(\sqrt{n\rho})$, $\| \hat{\bE}^{-1}\|=O_P\bbb{\frac{1}{\rho \lambda^*(\bB)\lambda_K(\bTheta^T\bTheta)}}$ from results in Sec~\ref{sec:auxiliary}. 
		By Lemmas~6.7 and 6.8 of \citep{cape2019two} and using $\lambda^*(\bP)
		\geq \rho \lambda^*(\bB)\lambda_K(\bTheta^T\bTheta)$ from Lemma~\ref{lem:P_eigen}, we have $\|\hv-\bV \bW_\bV\|=O_P\bbb{\frac{\sqrt{ n}}{\sqrt{\rho} \lambda^*(\bB)\lambda_K(\bTheta^T\bTheta)}}$ and $\|\bV^T\hv-\bW_\bV\|=O_P\bbb{\frac{{ n}}{{\rho} (\lambda^*(\bB))^2(\lambda_K(\bTheta^T\bTheta))^2}}$.
		From Lemma~\ref{lem:v-row-norm}, $\|\bV\|_{2\rightarrow \infty} \leq 
		{\frac{1}{\sqrt{\lambda_K(\bTheta^T\bTheta)}}}$. 
		It is also easy to get $\|\bI-\bV \bV^T\|_{2\rightarrow \infty}\leq 2$. Putting these bounds together, we have:
		\bas{
			\|\hat{\bV}-\bV\bW_\bV\|_{2\rightarrow \infty}
			&=O_P\bbb{\max\bbb{\frac{{ n}}{{\rho} (\lambda^*(\bB))^2(\lambda_K(\bTheta^T\bTheta))^2},
					\frac{\sqrt{K n}\bbb{1+\sqrt{K/\lambda_K(\bTheta^T\bTheta)}}}{\sqrt{\rho} \lambda^*(\bB)(\lambda_K(\bTheta^T\bTheta))^{1.5}}
			}}\\
			&=O_P\bbb{ \frac{{ n}}{{\rho} (\lambda^*(\bB))^2(\lambda_K(\bTheta^T\bTheta))^2} },
		}
		when $\lambda_K(\bTheta^T\bTheta)=\Omega(K)$ and $\rho (\lambda^*(\bB))^2 \leq  \frac{n}{ K\lambda_K(\bTheta^T\bTheta)}$.
	\end{proof}

\begin{rem}
	Our bound in Theorem~\ref{thm:entrywise} has better dependency on $\lambda^*(\bB)$ comparing to Eq~\eqref{eq:v_ae_cape_tighter} when $\lambda^*(\bB)$ goes to 0 ($\kappa(\bP)$ goes to infinity). When  $\lambda^*(\bB)=\Theta(1)$, $\kappa(\bP)=\Theta(1)$, and $\lambda_K(\bTheta^T\bTheta)=\Omega(n/K)$, 
	our bound in Theorem~\ref{thm:entrywise} is tighter by a factor of $\sqrt{\rho}$ comparing to Eq~\eqref{eq:v_ae_cape_tighter}. As discussed in Sec~\ref{sec:supp-abbe}, when $\btheta_i\sim \mathrm{Dirichlet}(\balpha)$ and $\alpha_{\max}\leq C\alpha_{\min}$ for some constant $C\geq1$, with high probability $\lambda_K(\bTheta^T\bTheta)=\Omega(n/K)\gg K$, and $\rho (\lambda^*(\bB))^2 \leq  \frac{n}{ K\lambda_K(\bTheta^T\bTheta)}=\Theta(1)$ is corresponding to the interesting regime when $\rho$ or $\lambda^*(\bB)$ is small.
\end{rem}
\end{lem}

}

%% file: low_rank_P.tex
\section{Row-wise eigenspace concentration for general low rank matrix}
\label{sec:generalP}
Note that although our focus of this paper is on MMSB, Theorem~\ref{thm:entrywise} can be easily extended to handle any low rank matrix. The proof is almost identical to that of Theorem~\ref{thm:entrywise}, 
{just instead of assuming Assumption~\ref{as:theta_P} is satisfied, we have some general conditions. The new events should be:}
\ba{
	\evot_{\mathrm{1}}
	&:=\left\{\max_i\|\bv_i\|_\infty\leq\sqrt{\rho}\right\} 
	& 
	&\uP(\bar{\evot}_{\mathrm{1}})\leq \delta_1 \nonumber\\
	\evot_{\mathrm{2}}
	&:=\left\{\lambda^*(\bP)\geq 4\sqrt{n\rho}(\log n)^\xi\right\} 
	& 
	&\uP(\bar{\evot}_{\mathrm{2}})\leq \delta_2 \nonumber\\
	\evz
	&:=\{\|\bA-\bP\|\leq C\sqrt{n\rho}\}  
	& 
	&\uP(\bar{\evz}){\leq} n^{-3} \label{eq:new_events}\\
	\E_1&:=\left\{\left|\be_i^T \bH \bv_k\right|\leq 4\log n\|\bv_k\|_\infty, \forall k\in[K]\right\}  & &\uP(\bar{\E_1}){\leq} O\bb{K/n^3}+\delta_1\nonumber\\
	\evt&:=\left\{\left|\be_i^T \bH^t \bv_k\right|\leq (\log n)^{t\xi}\|\bv_k\|_\infty, \forall k\in[K]\right\}  & &\uP(\bar{\evt}){\leq} K\exp(-(\log n)^\xi/3), 1<t\leq\log n\nonumber
}
If we use the new events in Eq~\eqref{eq:new_events} in the proof, we can get the following Theorem:
\begin{thm}\label{the:general_P}
	Suppose $\bP$ has rank $K$, $\max_{i,j}\bP_{ij}\leq\rho$. Let $\bA_{ij}=\bA_{ji}\sim \mathrm{Ber}(\bP_{ij})$, $\bV$ and $\hat{\bV}$ are $\bP$ and $\bA$'s $K$ leading eigenvectors respectively. If  $\uP(\max_i\|\bv_i\|_\infty>\sqrt{\rho})\leq\delta_1$, and for some constant $\xi>1$, $\rho n=\Omega((\log n)^{2\xi})$ and $\uP(\lambda^*(\bP)<4\sqrt{n\rho}(\log n)^\xi)<\delta_2$, 
	then 
	with probability at least $1-\delta_1-\delta_2-O(Kn^{-2})$, 
	\bas{
		\max_{i\in[n]} \| \be_i^T( \hat{\bV}\hat{\bV}^T-\bV\bV^T) \|
		=O\bbb{\frac{\gammaP\sqrt{Kn\rho}}{\lambda^*(\bP)}}\bbb{(1+(\log n)^{\xi})\max_i\|\bv_i\|_\infty + 2n^{-2\xi}}.
	}
\end{thm}

\begin{rem}
	For MMSB, it is easy to check that the condition $\lambda_K(\bTheta^T\bTheta)\geq 1/\rho$ in Assumption~\ref{as:theta_P} is only used in the proof of Lemma~\ref{lem:azuma-better} in Sec~\ref{sec:azuma-better} to show $\max_i\|\bv_i\|_\infty\leq\sqrt{\rho}$, so conditioned on $\evot_{\mathrm{1}}$ and $\evot_{\mathrm{2}}$, the proof goes through.
	If we plug in the upper bound of $\max_i\|\bv_i\|_\infty$ from Lemma~\ref{lem:v-row-norm} and lower bound of $\lambda^*(\bP)$ in Lemma~\ref{lem:P_eigen}, we can get the bound in Theorem~\ref{thm:entrywise} using Theorem~\ref{the:general_P}.
\end{rem}

%% file: err_bound_row.tex
\begin{proof}[Proof of Lemma~\ref{lem:vvt-v-equiv}]
	To see that the pruning algorithm returns identical nodes (up-to ties) is straightforward. This is because the pruning algorithm proceeds by calculating Euclidean distances between pairs of nodes for nearest neighbor computation. 
	We have 
	\bas{
		\|\hv\hv^T(\be_i-\be_j)\|^2&=(\be_i-\be_j)^T\hv\hv^T(\be_i-\be_j)=\|\hv^T(\be_i-\be_j)\|^2
	}
	Thus the pairwise distances between columns of $\hv\hv^T$ are the same as that between columns of $\hv^T$.	As for the SPA algorithm, we prove the claim by induction.
	
	{\it Base case}: For step $k=1$, as 
	\bas{
		\|\hv\hv^T\be_i\|^2=\be_i^T\hv\hv^T\hv\hv^T\be_i=\be_i^T\hv\hv^T\be_i=\|\hv^T\be_i\|^2,
	}
	picking max norm will give the same index, denoted as $k_1$.
	
	Now for $\hv^T$, the vector whose projection is removed is $\hv^T\be_{k_1}$, and the normalized vector is $\bu=\hv^T\be_{k_1}/\|\hv^T\be_{k_1}\|$, then for $\hv\hv^T$, the vector whose projection is removed is $\hv\hv^T\be_{k_1}$ and its normalized vector is $\bu_1=\hv\hv^T\be_{k_1}/\|\hv\hv^T\be_{k_1}\|=\hv\hv^T\be_{k_1}/\|\hv^T\be_{k_1}\|=\hV\bu$.

	Now 
	\bas{
		\|(\bI-\bu_1\bu_1^T)\hv\hv^T\be_i\|^2
		&=\|(\bI-\hV\bu\bu^T\bV^T)\hv\hv^T\be_i\|^2 \\
		&=\|\hv(\bI-\bu\bu^T)\hv^T\be_i\|^2 = \|(\bI-\bu\bu^T)\hv^T\be_i\|^2,
	}
	then for step $k=2$, picking max norm will also give the same index.
	
	{\it Induction}: Suppose for first $k-1\in[K-1]$ steps SPA on $\hv^T$ and on $\hv\hv^T$ will give the same indices as $S_{k-1}$, then for the $k$-th step, we are removing the projections of the $k-1$ columns in $S_{k-1}$ selected before, now denote the singular value decomposition  $(\hv_{S_{k-1}})^T=\bU\bS\bH^T$, then the projection matrix on columns of $(\hv_{S_{k-1}})^T$ is $\bU\bU^T$. Also note that $\hv(\hv_{S_{k-1}})^T=(\hv\bU)\bS\bH^T$, it is easy to check that this is singular value decomposition of $\hv(\hv_{S_{k-1}})^T$, and the projection matrix on columns of $\hv(\hv_{S_{k-1}})^T$ is $\hv\bU(\hv\bU)^T=\hv\bU\bU^T\hv^T$.
	Now the norm we need to pick from for SPA on $\hv\hv^T$ is 
	\bas{
		\|(\bI-\hv\bU\bU^T\hv^T)\hv\hv^T\be_i\|
		=\|\hv(\bI-\bU\bU^T)\hv^T\be_i\|
		=\|(\bI-\bU\bU^T)\hv^T\be_i\|,
	}
	so the norms to pick for SPA on $\hv^T$ and on $\hv\hv^T$ will still be same and picking max norm will also give the same index. 
\end{proof}

\begin{lem} \label{lem:Gillis}
	(Theorem 3 of {Gillis et al.}~\citep{gillis2014fast}). Let $\bM' = \bM + \bN = \bW \bH + \bN \in \R^{m\times n}$, where $\bM=\bW \bH =\bW[\bI_r | \bH']$, $\bW \in \R^{m\times r}$, $\bH\in \bR^{r\times n}_+$ and $\sum_{k=1}^{r}\bH'_{kj}\leq 1$, $\forall j$ and $r \geq 2$. Let $K(\bW)=\max_i \|\bW\bbb{:,i}\|_2$, and $\|\bN\bbb{:,i}\|_2\leq \epsilon$ for all $i$ with
	\bas{
		\epsilon < \sigma_r(\bW)\min\bbb{\frac{1}{2\sqrt{r-1}},\frac{1}{4}}\bbb{1+80\frac{K(\bW)^2}{\sigma_r^2(\bW)}}^{-1}
	}
	and $J$ be the index set of cardinality $r$ extracted by SPA, where $\sigma_r(\bW)$ is the $r$-th singular value of $\bW$. Then there exists a permutation $P$
	of $\{1,2,\cdots,r\}$ such that
	\bas{
		\max_{1\leq j \leq r}\|\bM'\bbb{:,J(j)}-\bW\bbb{:,P(j)}\| \leq \bar{\epsilon} = \epsilon\bbb{1+80\frac{K(\bW)^2}{\sigma_r^2(\bW)}}.
	}
\end{lem}	

\begin{thm} \label{thm:correct_indices_from_gillis}
	Let  $\cS_p$ be the indices set returned by SPA in Algorithm \ref{alg:nmf-mmsb-pure-res}, $\hvp=\hv(\cS_p,:)$. 
	If Assumptions~\ref{as:ident} and~\ref{as:theta_P} are satisfied, 
	then there exists 
	a permutation matrix $\bpi \in \R^{K\times K}$ such that
	\bas{
		\max_{1\leq j \leq K}\left\| \be_j^T\bbb{\hvp-\bpi^T{\bV_P}({\bV^T}\hv)}\right\|
		=O\bbb{\kappa({\bTheta^T\bTheta})\epsilon}
	}
	with probability larger than $1-O(Kn^{-2})$,
	where $\epsilon=\eigenspacerowwise$ is the row-wise error from Theorem~\ref{thm:entrywise}, and the rows of $\bV_P\in \R^{K\times K}$ correspond to pure nodes.
\end{thm}
	\begin{proof}[Proof of Theorem~\ref{thm:correct_indices_from_gillis}]
		
		Note that from Lemma~\ref{lem:v-theta}, $\bV=\bTheta \bV_P$. Let $\bM'=\hv\hv^T$, $\bW=\bV\bV_P^T$, $\bH=\bTheta^T$, $r=K$, then for $\bM'  = \bW \bH + \bN$, we have $\|\bN\bbb{:,i}\|_2 \leq \epsilon$ uniformly with probability larger than $1-O(Kn^{-2})$ by Theorem~\ref{thm:entrywise}. W.L.O.G., let the first $K$ rows of $\bTheta$ be $K$ different pure nodes.
		Now use Lemma~\ref{lem:Gillis}, there exists a permutation $\pi$
		of $\{1,2,\cdots,K\}$ such that
		\bas{
			\max_{1\leq j \leq K}\| \bM'(:,\cS_p(j))-\bW({,:\pi(j)})\| =\epsilon\bbb{1+80\frac{K(\bW)^2}{\sigma_K^2(\bW)}}= O\bbb{\kappa({\bTheta^T\bTheta})\epsilon},
		}
		since $K(\bW)=\max_i \|\bW\bbb{:,i}\|_2\leq \sigma_1(\bW)$ and $\frac{\sigma_1(\bW)}{\sigma_K(\bW)}=\kappa(\bW)\leq\kappa(\bV_P) =O\bbb{\sqrt{\kappa({\bTheta^T\bTheta})}}$ by Lemma~\ref{lem:theta_condition_num}.
		
		So $\exists$ a permutation matrix  $\bpi \in \R^{K\times K}$ such that
		\bas{
			\max_{1\leq j \leq K} \| \bbb{\hv\hvp^T-\bW\bpi}\be_j\| = O\bbb{\kappa({\bTheta^T\bTheta})\epsilon},
		}
		{taking transpose, it gives}
		\bas{
			\max_{1\leq j \leq K}\| \be_j^T\bbb{\hvp\hv^T-\bpi^T{\bV_P}{\bV^T}}\| = O\bbb{\kappa({\bTheta^T\bTheta})\epsilon},
		}
		and
		\bas{
			\max_{1\leq j \leq K}\left\| \be_j^T\bbb{\hvp-\bpi^T{\bV_P}({\bV^T}\hv)}\right\|
			&=\max_{1\leq j \leq K}\left\| \be_j^T\bbb{\hvp\hv^T-\bpi^T{\bV_P}{\bV^T}}\hv\right\|  \\
			&\leq \max_{1\leq j \leq K}\left\| \be_j^T\bbb{\hvp\hv^T-\bpi^T{\bV_P}{\bV^T}}\right\|\left\|\hv\right\| 
			=O\bbb{\kappa({\bTheta^T\bTheta})\epsilon}
		}
	with probability larger than $1-O(Kn^{-2})$. The inequality follows from Proposition~5.6 of \citep{cape2019two}.
	\end{proof}

\begin{lem}
	\label{lem:hxpinv}
	Let $\cSp$ be the set of of pure nodes extracted using \OurAlgo.
	Let $\hvp$ denote the rows of $\hv$ indexed by $\cSp$, and $\bV_P$ denote the pure nodes of $\bV$. Then, if Assumptions~\ref{as:ident},~\ref{as:theta_P}, and~\ref{as:spa} are satisfied, 
	\bas{
		\max_{i\in[n]} \left\|\be_i^T\bV(\bV^T\hv)\bbb{\hvp^{-1}-(\bpi^T\bV_P({\bV^T}\hv))^{-1}}\right\|
		=O\bbb{\sqrt{\lambda_1({\bTheta^T \bTheta})}\kappa({\bTheta^T\bTheta})\epsilon}
	}
	with probability larger than $1-O(Kn^{-2})$,  where $\epsilon=\eigenspacerowwise$ is the row-wise error from Theorem~\ref{thm:entrywise}, and rows of $\bV_P\in \R^{K\times K}$ are corresponding to pure nodes.
\end{lem}

\begin{proof}[Proof of Lemma~\ref{lem:hxpinv}]
	Define by $\bF:=\bV^T\hat{\bV}$, and $\tilde{\bV}_P:=\bpi^T\bV_P \bF$, then,
	\ba{\label{eq:xpnorm_new}
		&\left\|\be_i^T\bV(\bV^T\hv)\bbb{\hvp^{-1}-(\bpi^T\bV_P({\bV^T}\hv))^{-1}}\right\|\nonumber\\
		=&
		\left\|\be_i^T\bV\bF({\hvp}^{-1} -\tilde{\bV}_P^{-1} )\right\|
		= \left\| \be_i^T\bV\bF \tilde{\bV}_P^{-1} \bbb{\tilde{\bV}_P-\hvp} \hvp^{-1} \right\| \nonumber\\
		\leq& \left\| \be_i^T\bV\bF \bF^{-1}\bV_P^{-1}\bpi \bbb{\tilde{\bV}_P-\hvp}\right\| \left\|\hvp^{-1} \right\|
		= \left\| \be_i^T\bTheta\bpi \bbb{\tilde{\bV}_P-\hvp}\right\| \left\|\hvp^{-1} \right\| \nonumber\\
		\leq& \max_{1\leq i \leq K}\left\| \be_i^T\bbb{\hvp-\bpi^T\bV_P\bV^T\hat{\bV}}\right\| \left\|\hvp^{-1} \right\|
		=O\bbb{\kappa({\bTheta^T\bTheta})\epsilon} \left\|\hvp^{-1} \right\|
	}
	where the first inequality is true because rows of $\bTheta\bpi$ are still nonnegative and have unit $\ell_1$ norm, and the last step follows from Theorem~\ref{thm:correct_indices_from_gillis}.
	Now we will prove a bound on $\|\hvp^{-1}\|$. Let $\hat{\sigma}_i$ be the $i^{th}$ singular value of $\hvp$, then,
	\begin{align}
		\label{eq:invnorm}
		\|\hvp^{-1}\|= \frac{1}{\hat{\sigma}_K}.
	\end{align}
	From Lemma~\ref{lem:theta_condition_num}, $\sigma_K\bbb{\bV_P}=1/\sqrt{\lambda_1({\bTheta^T \bTheta})}$ and $\sigma_1\bbb{\bV_P}=1/\sqrt{\lambda_K({\bTheta^T \bTheta})}$.
	
	Now for using the orthogonal matrix $\hat{\bO} \in \mathbb{R}^{K\times K}$ constructed using Definition~\ref{def:o},
	\bas{
		\bbb{\hvp\hv^T-\bpi^T{\bV_P}{\bV^T}}\hv
		=\bbb{\hvp-\bpi^T\bV_P\hat{\bO}}+\bpi^T{\bV_P}\bbb{\hat{\bO}\hv^T-\bV^T}\hv,
	}
	then by Lemma~\ref{lem:theta_condition_num}, Theorem~\ref{thm:correct_indices_from_gillis} and Lemma~\ref{lem:P_eigen}, 
	we have,
	\ba{\label{eq:vp_with_o}
		\|{\hvp-\bpi^T\bV_P\hat{\bO}}\|_F
		&\leq \| \hvp\hv^T-\bpi^T{\bV_P}{\bV^T}\|_F\cdot\|\hv\|  + \|{\bV_P}\|\cdot\|{\hat{\bO}\hv^T-\bV^T}\|_F\cdot\|\hv\|  \nonumber \\
		&\leq \sqrt{K} \max_{1\leq j \leq K}\left\| \be_j^T\bbb{\hvp-\bpi^T{\bV_P}({\bV^T}\hv)}\right\| + \frac{1}{\sqrt{\lambda_K(\bTheta^T\bTheta)}}\|{\hat{\bO}\hv^T-\bV^T}\|_F \nonumber\\
		&\leq O\bbb{\kappa({\bTheta^T\bTheta})\sqrt{K}\epsilon} + \frac{1}{\sqrt{\lambda_K(\bTheta^T\bTheta)}} O\bbb{\frac{\sqrt{Kn}}{\sqrt{\rho}\lambda^*(\bB) \lambda_K(\bTheta^T\bTheta)}} \nonumber \\
		&=O\bbb{\kappa({\bTheta^T\bTheta})\sqrt{K}\epsilon}+O\bbb{\frac{\sqrt{Kn}}{\sqrt{\rho} \lambda^*(\bB) (\lambda_K(\bTheta^T\bTheta))^{1.5}}} \nonumber \\
		&=O\bbb{\kappa({\bTheta^T\bTheta})\sqrt{K}\epsilon} \quad \mbox{with probability larger than $1-O(Kn^{-2})$.}
	}

	Now, Weyl's inequality for singular values gives us:
	\ba{\label{eq:vp_lambda}
		&\left|\hat{\sigma}_i-\sigma_i(\bV_P)\right|
		\leq\|{\hvp-\bpi^T\bV_P\hat{\bO}}\|
		\leq \|{\hvp-\bpi^T\bV_P\hat{\bO}}\|_F 
		= O\bbb{\kappa({\bTheta^T\bTheta})\sqrt{K}\epsilon} \nonumber \\
		&\hat{\sigma}_K
		\geq {\frac{1}{\sqrt{\lambda_1({\bTheta^T \bTheta})}}}
		\bbb{1-O\bbb{ \kappa({\bTheta^T\bTheta})\sqrt{K\lambda_1({\bTheta^T \bTheta})}\epsilon }}
		\\
		&\hat{\sigma}_1
		\leq {{\frac{1}{\sqrt{\lambda_K({\bTheta^T\bTheta})}}}}
		\bbb{1+O\bbb{ \kappa({\bTheta^T\bTheta})\sqrt{K\lambda_K({\bTheta^T \bTheta})}\epsilon }}\nonumber.
	}

	Plugging this into Eq~\eqref{eq:invnorm} we get:
	\bas{ 
		\|\hvp^{-1}\|= {\sqrt{\lambda_1({\bTheta^T \bTheta})}}\bbb{1+O\bbb{ \kappa({\bTheta^T\bTheta})\sqrt{K\lambda_1({\bTheta^T \bTheta})}\epsilon }}=O\bbb{\sqrt{\lambda_1({\bTheta^T\bTheta})}}. 
	}
	The last step is true since Assumption~\ref{as:spa} implies $\kappa({\bTheta^T\bTheta})\sqrt{K\lambda_1({\bTheta^T \bTheta})}\epsilon=O(1)$.
	Note that we also have
	\bas{
		\|\bV_P^{-1}\|=\frac{1}{\sigma_K\bbb{\bV_P}} = O\bbb{\sqrt{\lambda_1({\bTheta^T \bTheta})}}.
	}	
	Finally putting everything together with Eq~\eqref{eq:xpnorm_new} we get, with probability larger than $1-O(Kn^{-2})$,
	\bas{
		\max_{i\in[n]} \left\|\be_i^T\bV(\bV^T\hv)\bbb{\hvp^{-1}-(\bpi^T\bV_P({\bV^T}\hv))^{-1}}\right\|
		&=O\bbb{\kappa({\bTheta^T\bTheta})\epsilon} \left\|\hvp^{-1} \right\|\\
		&=O\bbb{\sqrt{\lambda_1({\bTheta^T \bTheta})}\kappa({\bTheta^T\bTheta})\epsilon}.
	}
The failure probability comes from the event that 
Theorem~\ref{thm:entrywise} fails, giving $O(Kn^{-2})$.
\end{proof}

%% file: error_rest_row.tex
\begin{proof}[Proof of Theorem~\ref{thm:theta-B}] 
	We break this up into proofs of Eqs~\eqref{eq:theta-err} 
	and~\eqref{eq:b-err}.
  Recall that $\epsilon=\eigenspacerowwise$ is the row-wise error from Theorem~\ref{thm:entrywise}.
	\medskip
	
	\newcommand{\tvpi}{\tilde{\bV}_P^{-1}}
	\newcommand{\tvp}{\tilde{\bV}_P}
	
	{\em Proof of Eq~\eqref{eq:theta-err}.}
Recall that 
$\hat{\bTheta} = \hv \hvp^{-1}$.
We have uniformly $\forall i \in [n]$,
\bas{\left\|\be_i^T\bbb{\hat{\bTheta}-\bTheta\bpi}\right\| 
	&= \left\|\be_i^T\bbb{\hv\hvp^{-1}-\bV\bV_P^{-1}\bpi}\right\| \\
	&\leq
	\left\|\be_i^T(\hv-\bV({\bV^T}\hv))\hvp^{-1} \right\|
	+
	\left\|\be_i^T\bV(\bV^T\hv)\bbb{\hvp^{-1}-(\bpi^T\bV_P({\bV^T}\hv))^{-1}}\right\|\\
	&\stackrel{(i)}{\leq} \left\|\be_i^T\bbb{\hv-\bV({\bV^T}\hv)}\right\| \left\|\hvp^{-1} \right\|
	+ O\bbb{\sqrt{\lambda_1({\bTheta^T \bTheta})}\kappa({\bTheta^T\bTheta})\epsilon}\\
	&\stackrel{(ii)}{\leq}  \epsilon \cdot O(\sqrt{\lambda_1(\bTheta^T\bTheta)})+ O\bbb{\sqrt{\lambda_1({\bTheta^T \bTheta})}\kappa({\bTheta^T\bTheta})\epsilon}\\
	&= O\bbb{\sqrt{\lambda_1({\bTheta^T \bTheta})}\kappa({\bTheta^T\bTheta})\epsilon}\\
	&= O\bbb{\sqrt{\lambda_1({\bTheta^T \bTheta})}\kappa({\bTheta^T\bTheta})}\eigenspacerowwise\\
	&= \ThetaError
}
with probability larger than $1-O(Kn^{-2})$. 
Here (i) and (ii) follow from Lemma~\ref{lem:hxpinv} and its proof, {and the failure probability comes from the event that 
Theorem~\ref{thm:entrywise} does not hold.}

\medskip
{\em Proof of Eq~\eqref{eq:b-err}.}
Note
$\hat{\rho}\hat{\bB} = \hvp\hat{\bE}\hvp^T$ and $\rho\bB=\bV_P\bE\bV_P^T$. 
Note that $\|{\bE}\|\leq \max_i \|\be_i^T\bP\|_1=O(\rho n)$, and $\|\hat{\bE}\|\leq\|\bE\|+\|{\bA}-\bP\|=O(\rho n)$ using Weyl's inequality and Theorem~5.2 of \cite{lei2015consistency}.
Then we have the following decomposition
\bas{
	&\left\| \hat{\rho}\hat{\bB} - \rho\bpi^T\bB\bpi \right\|_F 
	=\left\| \hvp\hat{\bE}\hvp^T - \bpi^T\bV_P\bE\bV_P^T\bpi \right\|_F \\
	=& \left\| \bbb{\hvp-\bpi^T\bV_P\hat{\bO}}\hat{\bE}\hvp^T 
	+ \bpi^T\bV_P\bbb{\hat{\bO}\hat{\bE}-\bE\hat{\bO}}\hvp^T 
	+ \bpi^T\bV_P\bE\hat{\bO}\bbb{\hvp^T - \hat{\bO}^T\bV_P^T\bpi} \right\|_F \\
	\leq& \left\| {\hvp-\bpi^T\bV_P\hat{\bO}}\right\|_F\left\|\hat{\bE}\right\|\left\|\hvp \right\|
	+ \left\| \bV_P\right\|\left\|{\hat{\bO}\hat{\bE}-\bE\hat{\bO}}\right\|_F\left\|\hvp \right\|
	+ \left\| \bV_P\right\| \left\|\bE\right\|\left\|{\hvp^T - \hat{\bO}^T\bV_P^T \bpi} \right\|_F \\
	\leq& 2\cdot O\bbb{\kappa({\bTheta^T\bTheta})\sqrt{K}\epsilon} \cdot O(\rho n)\cdot {{\frac{1}{\sqrt{\lambda_K({\bTheta^T\bTheta})}}}} 
	+{{\frac{1}{\sqrt{\lambda_K({\bTheta^T\bTheta})}}}}\left\|{\hat{\bO}\hat{\bE}-\bE\hat{\bO}}\right\|_F {{\frac{1}{\sqrt{\lambda_K({\bTheta^T\bTheta})}}}} \\
	=& O\bbb{\frac{ \kappa({\bTheta^T\bTheta})\sqrt{K}\rho n\epsilon}{\sqrt{\lambda_K({\bTheta^T\bTheta})}}}
	+ O\bbb{{\frac{1}{\lambda_K({\bTheta^T\bTheta})}}}\left\|{\hat{\bO}\hat{\bE}-\bE\hat{\bO}}\right\|_F 
}
using Eqs~\eqref{eq:vp_with_o} and \eqref{eq:vp_lambda} and Lemma~\ref{lem:theta_condition_num}. 

Now by Lemma~\ref{lem:E-lemma-op},
\bas{
	&\frac{1}{\rho}{\|\hat{\rho}\hat{\bB}-\rho\bpi^T\bB\bpi\|_F}
	\leq O\bbb{\frac{ \kappa({\bTheta^T\bTheta})\sqrt{K} n\epsilon}{\sqrt{\lambda_K({\bTheta^T\bTheta})}}}
	+ O\bbb{{\frac{1}{\rho\lambda_K({\bTheta^T\bTheta})}}}\left\|{\hat{\bO}\hat{\bE}-\bE\hat{\bO}}\right\|_F  \\
	&= O\bbb{\frac{ \kappa({\bTheta^T\bTheta})\sqrt{K} n}{\sqrt{\lambda_K({\bTheta^T\bTheta})}}}\cdot\eigenspacerowwise+O\bbb{{\frac{1}{\rho\lambda_K({\bTheta^T\bTheta})}}}\cdot O\bbb{K^{2}\sqrt{n\rho}}\\
	& =\RelativeErrorB
}
with probability larger than $1-O(Kn^{-2})$. {The failure probability comes from the event that 
Theorem~\ref{thm:entrywise} does not hold.}

\end{proof}

\begin{proof}[Proof Corollary~\ref{cor:theta-B_dirichlet_balance}] 
	Define the event 
	\bas{
		\Omega:=\{\bTheta:\lambda_K(\bTheta^T\bTheta)\geq1/\rho, \lambda^*(\bP)\geq 4\sqrt{n\rho}(\log n)^{\xi} \text{ for some constant } \xi>1\}.
	}
	If $\btheta_i\sim\mathrm{Dirichlet}(\balpha)$ and Assumption~\ref{as:sep} is satisfied, we have  $\uP(\bTheta \in \Omega)\geq 1-Kn^{-3}$.
	If Assumption~\ref{as:ident} 
	holds, and $\lambda^*(\bB)=\tilde{\Omega}(\frac{\min\{K,\kappa(\bB)\}^2K^2}{\sqrt{n\rho}})$,  for $\bTheta\in \Omega$, by Theorem~\ref{thm:entrywise} and Lemma~\ref{lem:theta_property},
	\ba{\label{eq:theta_dirich_balanced}
		\max_{i\in[n]}\left\|\be_i^T\bbb{\hat{\bTheta}-\bTheta\bpi}\right\| 
		&= \ThetaError
		=\tilde{O}\bbb{\frac{\gammaP  \bbb{\frac{\alpha_{\max}+\|\balpha\|^2}{\alpha_{\min}}}^{1.5} \sqrt{Kn}  }{\sqrt{\rho}\lambda^*(\bB)\frac{n}{2\nu(1+\alpha_0)}}}\nonumber\\
		&= \ThetaErrorDirichletBalanced.
	}
	Since $\max_a \alpha_a\leq C \min_a \alpha_a$ for some constant $C\geq1$, and $\alpha_{0}=O(1)$, the last step uses that  
	\bas{
		\frac{\alpha_{\max}+\|\balpha\|^2}{\alpha_{\min}}
		\leq \frac{\alpha_{\max}+\alpha_{\max}}{\alpha_{\min}}
		=(1+\alpha_0)\frac{\alpha_{\max}}{\alpha_{\min}}
		=O(1), 
	}
	and by the worst case bound from Lemma~\ref{lem:gammaP_bound}, we have,  $\gammaP \leq \min\{K,\kappa(\bP)\}^2 \leq \min\{K,\kappa(\bTheta^T\bTheta)\kappa(\bB)\}^2=O(\min\{K,\kappa(\bB)\}^2)$. 
	
	Now we are ready to obtain the failure probability of Eq~\eqref{eq:theta_dirich_balanced}.
	Consider the event $\A$ that $\hat{\bTheta}$ does not satisfy Eq~\eqref{eq:theta_dirich_balanced}. Then, by Theorem~\ref{thm:entrywise}, 
	\ba{
		\uP(\A)	&=	\int_{\bTheta\in \Omega}\uP\bbb{\A|\bTheta}\uP(\bTheta)d\bTheta+\int_{\bTheta\not\in \Omega}\uP\bbb{\A|\bTheta}\uP(\bTheta)d\bTheta\nonumber\\
		&= O\bbb{\frac{K}{n^2}}+ 1-\uP(\bTheta\in\Omega) 
		=  O\bbb{\frac{K}{n^2}}.\label{eq:theta-failure-dirichlet}
	}

	Similarly, by Theorem~\ref{thm:entrywise} and Lemma~\ref{lem:theta_property},
	\ba{
		\frac{1}{\rho}{\|\hat{\rho}\hat{\bB}-\rho\bpi^T\bB\bpi\|_F}
		&= \RelativeErrorB
		=\tilde{O}\bbb{\frac{\gammaP \bbb{\frac{\alpha_{\max}+\|\balpha\|^2}{\alpha_{\min}}}Kn^{1.5} }{\sqrt{\rho}\lambda^*(\bB)\bbb{\frac{n}{2\nu(1+\alpha_0)}}^{2}}}\nonumber\\
		&= \ErrorBDirichletBalanced.\label{eq:B-dirichlet}
	}
	By an argument analogous to that in Eq~\eqref{eq:theta-failure-dirichlet}, we can show that the failure probability of Eq~\eqref{eq:B-dirichlet} is  $O(Kn^{-2})$.
\end{proof} 

%% file: lemma_for_jin.tex
\section{Comparison with~\citep{jin2017estimating}} \label{sec:lem_for_jin}
We first translate some key assumptions in~\citep{jin2017estimating} (Eqs~(2.14) and~(2.15)) with our notation.

\begin{assumption}\label{as:jin}
	Assume for some constants $C>0$ and $c_1>0$,
	\bas{
		\frac{n}{CK}
		&\leq\lambda_K(\bTheta^T\bTheta)
		\leq\lambda_1(\bTheta^T\bTheta)\leq \frac{Cn}{K}\\
		\frac{c_1n}{K}\lambda^*(\bB)
		&\leq|\lambda_K(\bB\bTheta^T\bTheta)|\leq |\lambda_2(\bB\bTheta^T\bTheta)|\leq \frac{n}{c_1K}\lambda^*(\bB)\\
		|\lambda_2(\bB\bTheta^T\bTheta)|
		&\leq (1-c_1)\lambda_1(\bB\bTheta^T\bTheta)
	}
\end{assumption}

\begin{lem}\label{lem:our_for_jin}
	If Assumption~\ref{as:jin} is satisfied, for $\hat{\bTheta}$ estimated by \OurAlgo, we have,
	\bas{
		\left\|\be_i^T\bbb{\hat{\bTheta}-\bTheta\bpi}\right\| 
		= \tilde{O}\bbb{\frac{ K^{1.5}  }{\sqrt{n\rho}\lambda^*(\bB)}}.
	}
	\begin{proof}
		By Theorem~1.3.22 of~\citep{horn2012matrix},  $\rho\bB\bTheta^T\bTheta$ and $\bP=\rho\bTheta\bB\bTheta^T$ have the same $K$ largest eigenvalues in magnitude. So
		Assumption~\ref{as:jin} implies that:
		\bas{
			\frac{c_1n\rho}{K}\lambda^*(\bB)&\leq|\lambda_K(\bP)|\leq |\lambda_2(\bP)|\leq \frac{n\rho}{c_1K}\lambda^*(\bB)\\
			|\lambda_2(\bP)|
			&\leq (1-c_1)\lambda_1(\bP).
		}
		Then the eigenvalues of $\bP$ can be divided into at most 2 groups where eigenvalues in each group are of the same order, by Lemma~\ref{lem:gammaP_bound}, we have $\gammaP =O(1)$.
		
		On the other hand, if Assumption~\ref{as:jin} is satisfied, we have $\kappa(\bTheta^T\bTheta)=O(1)$, and by Theorem~\ref{thm:theta-B},
		\bas{
			\left\|\be_i^T\bbb{\hat{\bTheta}-\bTheta\bpi}\right\| 
			&= \tilde{O}\bbb{\frac{ (\kappa({\bTheta^T\bTheta}))^{1.5} \sqrt{Kn}  }{\sqrt{\rho}\lambda^*(\bB)\lambda_K(\bTheta^T\bTheta)}}
			= \tilde{O}\bbb{\frac{ \sqrt{Kn}  }{\sqrt{\rho}\lambda^*(\bB)n/K}}
			= \tilde{O}\bbb{\frac{ K^{1.5}  }{\sqrt{n\rho}\lambda^*(\bB)}}
		}
	\end{proof}	
\end{lem}

\begin{rem}
	Since~\citep{jin2017estimating} shows $\ell_1$ norm error bound, our result in Lemma~\ref{lem:our_for_jin} matches theirs with an extra $\sqrt{K}$ factor up-to logarithm factor, if we convert the bound in Lemma~\ref{lem:our_for_jin} to $\ell_1$ norm by multiplying $\sqrt{K}$.
\end{rem}

%% file: lemma_for_tensor.tex
\section{Comparison with~\citep{MMSBAnandkumar2014}}\label{sec:tensor}
\begin{lem}\label{lem:l1}
	Let $\btheta_i\sim \mathrm{Dirichlet}(\alpha)$. If Assumptions~\ref{as:ident} and~\ref{as:sep} 
	hold, and $\lambda^*(\bB)=\tilde{\Omega}(\frac{\min\{K,(1+\alpha_{0})\kappa(\bB){\alpha_{\max}}/{\alpha_{\min}}\}^2K^2}{\sqrt{n\rho}})$, there exists a permutation matrix $\bpi$ such that with probability at least $1-O(K/n^2)$, {$\forall i\in [n]$,}
	\bas{
		\left\|{\hat{\bTheta}-\bTheta\bpi}\right\|_1
		=\tilde{O}\bbb{\bbb{\frac{\alpha_{\mathrm{max}}}{\alpha_{\mathrm{min}}}}^{1.5}\sqrt{\frac{n}{\rho}}\frac{ \min\{K,(1+\alpha_{0})\kappa(\bB){\alpha_{\max}}/{\alpha_{\min}}\}^2 K\bal\bbb{1+\alpha_0}^{2.5} }{\lambda^*(\bB)}},
	}
	where $\|\bM\|_1=\sum_{i,j}|\bM_{ij}|$ is the $\ell_1$ norm for a matrix $\bM$.
	\begin{proof}
		First note from the proof of  Corollary~\ref{cor:theta-B_dirichlet_balance}, we have $(\alpha_{\max}+\|\balpha\|^2)/\alpha_{\min}\leq (1+\alpha_0)\alpha_{\max}/\alpha_{\min}$, and by Lemma~\ref{lem:theta_property}, we have, with high probability $\gammaP \leq \min\{K,\kappa(\bP)\}^2 \leq \min\{K,\kappa(\bTheta^T\bTheta)\kappa(\bB)\}^2=O(\min\{K,(1+\alpha_{0})\kappa(\bB){\alpha_{\max}}/{\alpha_{\min}}\}^2)$. Now by Theorem~\ref{thm:theta-B}, 
		if we sum up the squared error bound for each row, we can get a Frobenius bound:
		\bas{
			\frac{1}{\sqrt{n}}\left\|{\hat{\bTheta}-\bTheta}\right\|_F
			=&\tilde{O}\bbb{\bbb{\frac{\alpha_{\mathrm{max}}}{\alpha_{\mathrm{min}}}}^{1.5}\frac{ \min\{K,(1+\alpha_{0})\kappa(\bB){\alpha_{\mathrm{max}}}/{\alpha_{\mathrm{min}}}\}^2 K^{0.5}\bal\bbb{1+\alpha_0}^{2.5} }{\sqrt{\rho n}\lambda^*(\bB)}}
		}
		and so
		\bas{
			\left\|{\hat{\bTheta}-\bTheta}\right\|_1
			&\leq \sqrt{Kn}\left\|{\hat{\bTheta}-\bTheta}\right\|_F\\
			&=\tilde{O}\bbb{\bbb{\frac{\alpha_{\mathrm{max}}}{\alpha_{\mathrm{min}}}}^{1.5}\sqrt{\frac{n}{\rho}}\frac{ \min\{K,(1+\alpha_{0})\kappa(\bB){\alpha_{\mathrm{max}}}/{\alpha_{\mathrm{min}}}\}^2 K\bal\bbb{1+\alpha_0}^{2.5} }{\lambda^*(\bB)}}
		}
	\end{proof}
\end{lem}

\begin{rem}
	By Theorem~9 of~\citep{MMSBAnandkumar2014}, we have:
	\ba{\label{eq:aamm}
	\left\|\hat{\bTheta}-\bTheta\right\|_1
	&= \tilde{O}\left(\frac{\alpha_{\mathrm{max}}}{\alpha_0}\bbb{\frac{\alpha_{\mathrm{max}}}{\alpha_{\mathrm{min}} }}^{0.5} \frac{\sqrt{n}K\nu^{1.5}\left(1+\alpha_{0}\right)^{1.5}\sqrt{\max_i (\rho \be_i^T\bB \balpha)}}{\rho\sqrt{\alpha_0} \lambda^{*}\left({\bB}\right)}\right)\nonumber\\
	& = \tilde{O}\left(  \frac{\alpha_{\mathrm{max}}}{\alpha_0}\bbb{\frac{\alpha_{\mathrm{max}}}{\alpha_{\mathrm{min}} }}^{0.5} \sqrt{\frac{n}{\rho}} \frac{K \nu^{1.5}\left(1+\alpha_{0}\right)^{1.5}}{\lambda^{*}(\bB)} \right)
	}
	When $\max_a \alpha_a\leq C \min_a \alpha_a$ for some constant $C\geq1$, $\alpha_{0}=O(1)$ and $\kappa(\bB)=\Theta(1)$, we have $\nu=O(K)$, $\alpha_{\mathrm{max}}/\alpha_{0}=O(1/K)$,  ${\alpha_{\mathrm{max}}}/{\alpha_{\mathrm{min}}}=O(1)$ and  $\min\{K,(1+\alpha_{0})\kappa(\bB){\alpha_{\mathrm{max}}}/{\alpha_{\mathrm{min}}}\}^2=O(\min\{K,\kappa(\bB)\}^2)=O(1)$, so our bound in Lemma~\ref{lem:l1} is worse by $\sqrt{K}$ than Eq.~\eqref{eq:aamm}.
	
	For worst case analysis, $\alpha_{\mathrm{max}}/\alpha_{0}=O(1)$, $\alpha_{\mathrm{max}}/\alpha_{0}=O(\nu)$, and $\min\{K,(1+\alpha_{0})\kappa(\bB){\alpha_{\mathrm{max}}}/{\alpha_{\mathrm{min}}}\}^2=K^2$, so our bound in Lemma~\ref{lem:l1} is worse by $K^2\sqrt{\nu}(1+\alpha_{0})$ than Eq.~\eqref{eq:aamm}.
	
	Note that the proposed algorithm in~\citep{MMSBAnandkumar2014} requires prior knowledge on $\alpha_0$ while our algorithm does not need $\alpha_{0}$ as input. 
\end{rem}

%% file: pruning.tex
\section{Why Pruning Works}\label{sec:prune_works}
Proving the pruning algorithm requires strong distributional conditions on the residuals of the rows of eigenvectors. Here we present a heuristic argument of why pruning works. Note that in the pruning algorithm, essentially we are estimating the density of points in an $\epsilon$-ball around every point $i$ which has sufficiently large norm. This should work only if the points outside the population simplex have lower density in their $\epsilon$-balls than the corners of the simplex. Otherwise, the pruning will remove the corners of the population simplex, diminishing the quality of the pure nodes. 
\begin{figure}[!b]
	\centering
	\includegraphics[width=0.5\textwidth]{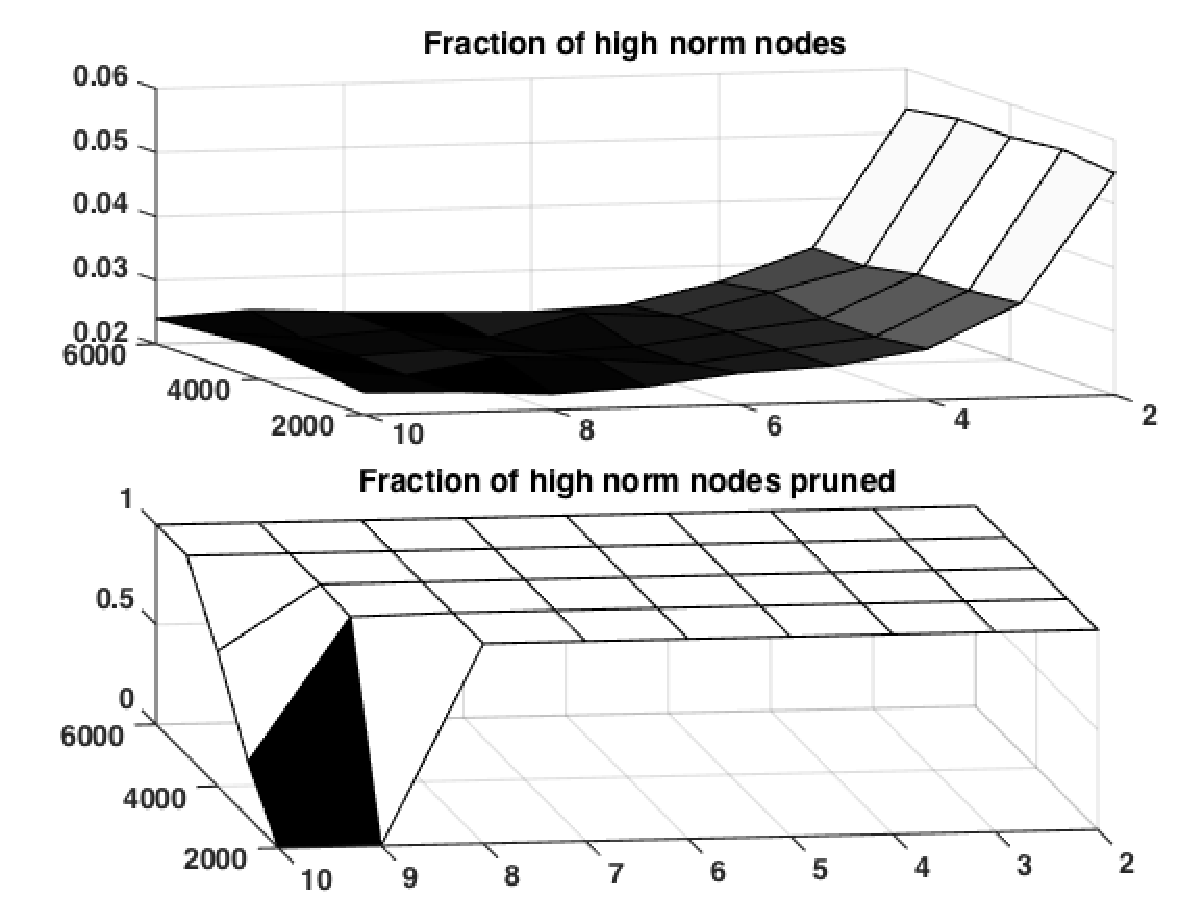}
	\caption{\label{fig:pruningmore}  Top panel: fraction of nodes with high norm. Bottom panel: fraction of nodes with high norm pruned. We vary $K\in \{2,\dots,10\}$ on the X axis  and vary $n\in\{2000,3000,\dots,6000\}$ on the $Y$ axis.  }
\end{figure}
We consider $K\in \{2,\dots, 10\}$ and $n\in\{2000,3000,\dots, 6000\}$, $\balpha=\bone_K/K$, $\bB_{ii}=1,\bB_{ij}=0.001$ and $\rho=\log n/n$. 
For each combination we use $\epsilon$ as the median of the row-wise difference of the empirical eigenvectors from their suitably rotated population counterpart. 
Let $y = \max_i \|\bV_i\|$ denote the largest row-wise norm of the population eigenvectors; recall that this occurs at one of the corners of the simplex.
Let $S_0$ denote the set of nodes with high empirical eigenvector row-norms (the ``high-norm'' nodes), defined as $S_0:=\{i: \|\hV_i\|\geq y+\epsilon\}$. 
SPA will choose at least one of these nodes (and possibly several of them) as its estimated corners.
Let $B(x,\epsilon)$ denote the $\ell_2$ ball of size $\epsilon$ centered at point $x$. 
For each of the $K$ corners $c_i$ of the population simplex ($c_i$ equals some row of $\bV_P$), we compute the number of neighborhood points $x_i:=|\{j|\hV_j\in B(c_i,\epsilon)\}|$; let
$\delta:=\min_i x_i$ be the minimum neighborhood size among these corners. 
Similarly, for each $i\in S_0$, we compute $z_i=|\{j|\hV_j\in B(\hV_i, \epsilon)\}|$.
Now we count the fraction of nodes in $S_0$ that could be pruned without pruning the corners $c_i$ of the population simplex.
This fraction is $m=\frac{\sum_{i\in S_0}\bone\{z_i<\delta\}}{|S_0|}$.
Fig~\ref{fig:pruningmore} shows that for almost all combinations of $K$ and $n$, we have $m=1$, i.e., all the nodes in $S_0$ do get pruned, except for $K=10,n=2000$. This is expected, since for large $K$ and small $n$ the pure node density around the corners of the population simplex will be small. Fig~\ref{fig:pruningmore} shows the fraction $|S_0|/n$ of high-norm nodes. For all $(K, n)$ combinations pruning removes about a 2\% to 6\% of the nodes.

%% file: supp-simu.tex
\section{Extra simulation results}
\label{sec:suppsimu}
\begin{figure}[!htb]
	\centering
	\begin{tabular}{@{\hspace{0em}}c@{\hspace{0em}}c@{\hspace{0em}}}
		\includegraphics[width=.4\textwidth]{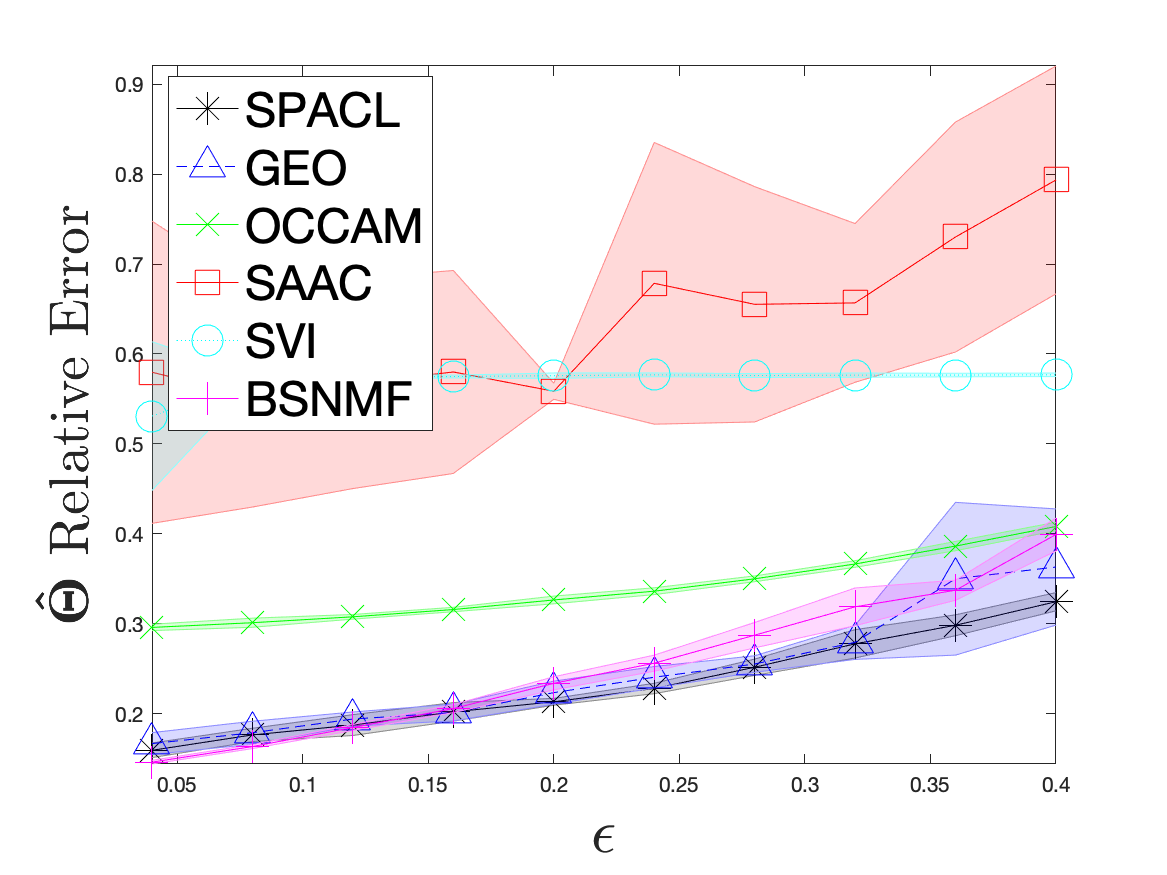}&
		\includegraphics[width=.4\textwidth]{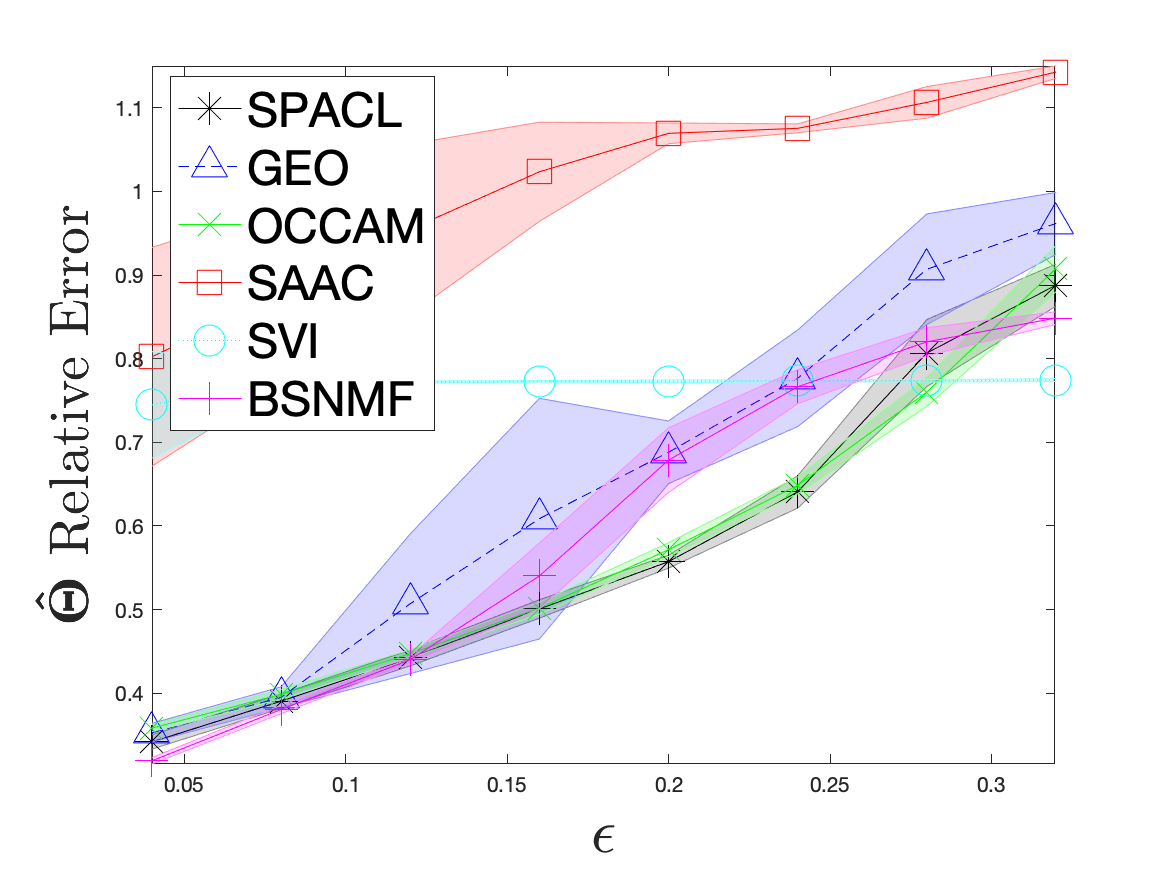}
		\\
		{\small(i)}&{\small(ii)}
	\end{tabular}
	\caption{\label{fig:expsimtwo}Error against $\epsilon$:  we use $\bB_{ii}=1$, $\bB_{ij}=\epsilon$ for $i\neq j$. (i) $K=3$. (ii) $K=7$.}	
\end{figure}

\noindent
{\bf Changing $\bB$:}
In Fig~\ref{fig:expsimtwo} (i), we plot the relative error in estimating $\bTheta$ against increasing off diagonal noise $\epsilon$ of $\bB$. We take $K=3$, $\rho=0.15$, $\alpha_i=3/K=1$, $\bB_{ii}=1$, $i\in[K]$. 
We see that \OurAlgo outperforms SAAC, SVI, and OCCAM over the entire parameter range. For large $\epsilon$, it is also better than GeoNMF and BSNMF.

We also include simulation results with $K=7$.
We take $\rho=0.15$, $\alpha_i=3/K=3/7$, $\bB_{ii}=1$, $i\in[K]$. 
We see in Fig~\ref{fig:expsimtwo} (ii) that SAAC performs poorly, and OCCAM performs similarly with \OurAlgo, which can also be implied from the simulation results on changing $K$. \OurAlgo is more stable and outperforms GeoNMF and BSNMF.

%% file: arxiv.bbl
\begin{thebibliography}{37}
\providecommand{\natexlab}[1]{#1}
\providecommand{\url}[1]{\texttt{#1}}
\expandafter\ifx\csname urlstyle\endcsname\relax
  \providecommand{\doi}[1]{doi: #1}\else
  \providecommand{\doi}{doi: \begingroup \urlstyle{rm}\Url}\fi

\bibitem[Abbe et~al.(2016)Abbe, Bandeira, and Hall]{abbe2015exact}
Emmanuel Abbe, Afonso~S Bandeira, and Georgina Hall.
\newblock Exact recovery in the stochastic block model.
\newblock \emph{IEEE Transactions on Information Theory}, 62\penalty0
  (1):\penalty0 471--487, 2016.

\bibitem[Abbe et~al.(2017)Abbe, Fan, Wang, and Zhong]{abbe2017entrywise}
Emmanuel Abbe, Jianqing Fan, Kaizheng Wang, and Yiqiao Zhong.
\newblock Entrywise eigenvector analysis of random matrices with low expected
  rank.
\newblock \emph{arXiv preprint arXiv:1709.09565}, 2017.

\bibitem[Airoldi et~al.(2008)Airoldi, Blei, Fienberg, and
  Xing]{airoldi2008mixed}
Edoardo~M Airoldi, David~M Blei, Stephen~E Fienberg, and Eric~P Xing.
\newblock Mixed membership stochastic blockmodels.
\newblock \emph{JMLR}, 9:\penalty0 1981--2014, 2008.

\bibitem[Anandkumar et~al.(2014)Anandkumar, Ge, Hsu, and
  Kakade]{MMSBAnandkumar2014}
Animashree Anandkumar, Rong Ge, Daniel Hsu, and Sham~M. Kakade.
\newblock A tensor approach to learning mixed membership community models.
\newblock \emph{JMLR}, 15\penalty0 (1):\penalty0 2239--2312, 2014.

\bibitem[Arora et~al.(2012)Arora, Ge, Kannan, and Moitra]{arora12computing}
Sanjeev Arora, Rong Ge, Ravindran Kannan, and Ankur Moitra.
\newblock Computing a nonnegative matrix factorization--provably.
\newblock In \emph{STOC}, pages 145--162. ACM, 2012.

\bibitem[Athreya et~al.(2016)Athreya, Priebe, Tang, Lyzinski, Marchette, and
  Sussman]{Athreya2016}
Avanti Athreya, Carey~E Priebe, Minh Tang, Vince Lyzinski, David~J Marchette,
  and Daniel~L Sussman.
\newblock A limit theorem for scaled eigenvectors of random dot product graphs.
\newblock \emph{Sankhya A}, 78\penalty0 (1):\penalty0 1--18, 2016.

\bibitem[Balakrishnan et~al.(2011)Balakrishnan, Xu, Krishnamurthy, and
  Singh]{BalaSpec2011}
Sivaraman Balakrishnan, Min Xu, Akshay Krishnamurthy, and Aarti Singh.
\newblock Noise thresholds for spectral clustering.
\newblock In \emph{NIPS}, pages 954--962. 2011.

\bibitem[Ball et~al.(2011)Ball, Karrer, and Newman]{Ball2011Overlapping}
Brian Ball, Brian Karrer, and Mark~EJ Newman.
\newblock Efficient and principled method for detecting communities in
  networks.
\newblock \emph{Physical Review E}, 84\penalty0 (3):\penalty0 036103, 2011.

\bibitem[Bentley(1975)]{Bentley:1975:kd}
Jon~Louis Bentley.
\newblock Multidimensional binary search trees used for associative searching.
\newblock \emph{Communications of the ACM}, 18\penalty0 (9):\penalty0 509--517,
  1975.

\bibitem[Beygelzimer et~al.(2006)Beygelzimer, Kakade, and
  Langford]{Beygelzimer:2006:covertree}
Alina Beygelzimer, Sham Kakade, and John Langford.
\newblock Cover trees for nearest neighbor.
\newblock In \emph{International Conference on Machine Learning}, pages
  97--104, 2006.

\bibitem[Cape et~al.(2018)Cape, Tang, and Priebe]{cape2018signal}
Joshua Cape, Minh Tang, and Carey~E Priebe.
\newblock Signal-plus-noise matrix models: eigenvector deviations and
  fluctuations.
\newblock \emph{arXiv preprint arXiv:1802.00381}, 2018.

\bibitem[Cape et~al.(2019)Cape, Tang, Priebe, et~al.]{cape2019two}
Joshua Cape, Minh Tang, Carey~E Priebe, et~al.
\newblock The two-to-infinity norm and singular subspace geometry with
  applications to high-dimensional statistics.
\newblock \emph{The Annals of Statistics}, 47\penalty0 (5):\penalty0
  2405--2439, 2019.

\bibitem[Chen et~al.(2014)Chen, Sanghavi, and Xu]{chen2014improved}
Yudong Chen, Sujay Sanghavi, and Huan Xu.
\newblock Improved graph clustering.
\newblock \emph{IEEE Transactions on Information Theory}, 60\penalty0
  (10):\penalty0 6440--6455, 2014.

\bibitem[Eldridge et~al.(2018)Eldridge, Belkin, and
  Wang]{eldridge2017unperturbed}
Justin Eldridge, Mikhail Belkin, and Yusu Wang.
\newblock Unperturbed: spectral analysis beyond davis-kahan.
\newblock In \emph{Algorithmic Learning Theory}, volume~83, pages 321--358.
  PMLR, 2018.

\bibitem[Erd{\H{o}}s et~al.(2013)Erd{\H{o}}s, Knowles, Yau, Yin,
  et~al.]{erdos2013}
L{\'a}szl{\'o} Erd{\H{o}}s, Antti Knowles, Horng-Tzer Yau, Jun Yin, et~al.
\newblock Spectral statistics of erd{\H{o}}s--r{\'e}nyi graphs i: local
  semicircle law.
\newblock \emph{The Annals of Probability}, 41\penalty0 (3B):\penalty0
  2279--2375, 2013.

\bibitem[Gillis and Vavasis(2014)]{gillis2014fast}
Nicolas Gillis and Stephen~A Vavasis.
\newblock Fast and robust recursive algorithmsfor separable nonnegative matrix
  factorization.
\newblock \emph{PAMI}, 36\penalty0 (4):\penalty0 698--714, 2014.

\bibitem[Gopalan and Blei(2013)]{gopalan2013efficient}
Prem~K Gopalan and David~M Blei.
\newblock Efficient discovery of overlapping communities in massive networks.
\newblock \emph{PNAS}, 110\penalty0 (36):\penalty0 14534--14539, 2013.

\bibitem[Holland et~al.(1983)Holland, Laskey, and
  Leinhardt]{holland_stochastic_1983}
Paul~W Holland, Kathryn~Blackmond Laskey, and Samuel Leinhardt.
\newblock Stochastic blockmodels: First steps.
\newblock \emph{Social networks}, 5\penalty0 (2):\penalty0 109--137, June 1983.
\newblock ISSN 0378-8733.

\bibitem[Hopkins and Steurer(2017)]{hopkins2017bayesian}
Samuel~B Hopkins and David Steurer.
\newblock Bayesian estimation from few samples: community detection and related
  problems.
\newblock In \emph{FOCS}, pages 379--390. IEEE, 2017.

\bibitem[Horn and Johnson(2012)]{horn2012matrix}
Roger~A Horn and Charles~R Johnson.
\newblock \emph{Matrix analysis}.
\newblock Cambridge university press, 2012.

\bibitem[Jin et~al.(2017)Jin, Ke, and Luo]{jin2017estimating}
Jiashun Jin, Zheng~Tracy Ke, and Shengming Luo.
\newblock Estimating network memberships by simplex vertex hunting.
\newblock \emph{arXiv preprint arXiv:1708.07852}, 2017.

\bibitem[Kaufmann et~al.(2016)Kaufmann, Bonald, and
  Lelarge]{kaufmann2016spectral}
Emilie Kaufmann, Thomas Bonald, and Marc Lelarge.
\newblock A spectral algorithm with additive clustering for the recovery of
  overlapping communities in networks.
\newblock In \emph{ALT}, pages 355--370, 2016.

\bibitem[Lei et~al.(2015)Lei, Rinaldo, et~al.]{lei2015consistency}
Jing Lei, Alessandro Rinaldo, et~al.
\newblock Consistency of spectral clustering in stochastic block models.
\newblock \emph{The Annals of Statistics}, 43\penalty0 (1):\penalty0 215--237,
  2015.

\bibitem[Mao et~al.(2017)Mao, Sarkar, and Chakrabarti]{mao2017}
Xueyu Mao, Purnamrita Sarkar, and Deepayan Chakrabarti.
\newblock On mixed memberships and symmetric nonnegative matrix factorizations.
\newblock In \emph{ICML}, pages 2324--2333, 2017.

\bibitem[McSherry(2001)]{mcsherry2001spectral}
Frank McSherry.
\newblock Spectral partitioning of random graphs.
\newblock In \emph{FOCS}, pages 529--537, 2001.

\bibitem[Oliveira(2009)]{oliveira2009concentration}
Roberto~Imbuzeiro Oliveira.
\newblock Concentration of the adjacency matrix and of the laplacian in random
  graphs with independent edges.
\newblock \emph{arXiv preprint arXiv:0911.0600}, 2009.

\bibitem[Panov et~al.(2017)Panov, Slavnov, and Ushakov]{panov2017consistent}
Maxim Panov, Konstantin Slavnov, and Roman Ushakov.
\newblock Consistent estimation of mixed memberships with successive
  projections.
\newblock In \emph{COMPLEX NETWORKS}, pages 53--64. Springer, 2017.

\bibitem[Press et~al.(2007)Press, Teukolsky, Vetterling, and
  Flannery]{press92numerical}
William~H Press, Saul~A Teukolsky, William~T Vetterling, and Brian~P Flannery.
\newblock \emph{Numerical recipes 3rd edition: The art of scientific
  computing}.
\newblock Cambridge university press, 2007.

\bibitem[Psorakis et~al.(2011)Psorakis, Roberts, Ebden, and Sheldon]{BNMF2011}
Ioannis Psorakis, Stephen Roberts, Mark Ebden, and Ben Sheldon.
\newblock Overlapping community detection using bayesian non-negative matrix
  factorization.
\newblock \emph{Phys.~Rev. E}, 83\penalty0 (6):\penalty0 066114, 2011.

\bibitem[Ray et~al.(2014)Ray, Ghaderi, Sanghavi, and
  Shakkottai]{ray2014overlap}
Avik Ray, Javad Ghaderi, Sujay Sanghavi, and Sanjay Shakkottai.
\newblock Overlap graph clustering via successive removal.
\newblock In \emph{52nd Annual Allerton Conference}, pages 278--285. IEEE,
  2014.

\bibitem[Rohe et~al.(2011)Rohe, Chatterjee, and Yu]{rohe2011spectral}
Karl Rohe, Sourav Chatterjee, and Bin Yu.
\newblock Spectral clustering and the high-dimensional stochastic blockmodel.
\newblock \emph{The Annals of Statistics}, pages 1878--1915, 2011.

\bibitem[Schaefer(1974)]{schaefer1974banach}
H.H. Schaefer.
\newblock \emph{Banach Lattices and Positive Operators}.
\newblock Grundlehren der mathematischen Wissenschaften. Springer Berlin
  Heidelberg, 1974.
\newblock ISBN 9783540069362.

\bibitem[Tropp(2012)]{tropp2012user}
Joel~A Tropp.
\newblock User-friendly tail bounds for sums of random matrices.
\newblock \emph{Foundations of computational mathematics}, 12\penalty0
  (4):\penalty0 389--434, 2012.

\bibitem[Wang et~al.(2011)Wang, Li, Wang, Zhu, and Ding]{wang2011community}
F.~Wang, T.~Li, X.~Wang, S.~Zhu, and C.~Ding.
\newblock Community discovery using nonnegative matrix factorization.
\newblock \emph{Data Mining and Knowl. Disc.}, 22\penalty0 (3):\penalty0
  493--521, 2011.

\bibitem[Wang et~al.(2016)Wang, Cao, Jin, Cao, and He]{wang2016supervised}
Xiao Wang, Xiaochun Cao, Di~Jin, Yixin Cao, and Dongxiao He.
\newblock The (un) supervised nmf methods for discovering overlapping
  communities as well as hubs and outliers in networks.
\newblock \emph{Physica A: Statistical Mechanics and its Applications},
  446:\penalty0 22--34, 2016.

\bibitem[Yu et~al.(2015)Yu, Wang, and Samworth]{yu2015useful}
Yi~Yu, Tengyao Wang, and Richard~J Samworth.
\newblock A useful variant of the davis--kahan theorem for statisticians.
\newblock \emph{Biometrika}, 102\penalty0 (2):\penalty0 315--323, 2015.

\bibitem[Zhang et~al.(2014)Zhang, Levina, and Zhu]{zhang2014detecting}
Yuan Zhang, Elizaveta Levina, and Ji~Zhu.
\newblock Detecting overlapping communities in networks using spectral methods.
\newblock \emph{arXiv preprint arXiv:1412.3432}, 2014.

\end{thebibliography}
